\documentclass[11pt,twoside,a4paper]{article}

\usepackage[utf8]{inputenc}     
\usepackage[T1]{fontenc}        
\usepackage[english]{babel}     
\usepackage{microtype}          

\usepackage{lmodern}            

\usepackage[leqno]{mathtools}   
\usepackage{amssymb}            
\usepackage{amsthm}             
\usepackage{nicefrac}           

\usepackage{bm}                 
\usepackage{dsfont}             
\usepackage{bbm}                
\usepackage{mathrsfs}           
\usepackage{euscript}           
\usepackage[bb=px]{mathalfa}    

\usepackage[dvipsnames]{xcolor} 
\usepackage{geometry}
\usepackage{graphicx}
\usepackage{booktabs}           
\usepackage{enumitem}
\setlist[enumerate]{itemsep=0pt}
\usepackage{fancyhdr}
\usepackage{titlesec}
\usepackage{authblk}            
\usepackage{float}

\usepackage[compress]{natbib}
\setcitestyle{authoryear, circle, citesep={,}, aysep={,}, yysep={,}}

\usepackage[ruled,linesnumbered]{algorithm2e}

\usepackage{tikz}
\usetikzlibrary{positioning, fit, tikzmark, calc}
\usetikzlibrary{decorations.pathreplacing, calligraphy}

\colorlet{inlinkcolor}{purple}
\colorlet{exlinkcolor}{purple}
\colorlet{citecolor}{teal!50!blue}

\usepackage{hyperref}
\hypersetup{
    colorlinks=true,
    linkcolor=inlinkcolor,
    citecolor=citecolor,
    urlcolor=exlinkcolor,
    linktoc=page,
    breaklinks=true,
    plainpages=false
}

\newtheorem{theorem}{Theorem}[section]
\newtheorem{proposition}[theorem]{Proposition}
\newtheorem{lemma}[theorem]{Lemma}
\newtheorem{corollary}[theorem]{Corollary}
\theoremstyle{definition}

\newtheorem{assumption}{Assumption}
\theoremstyle{remark}
\newtheorem{remark}[theorem]{Remark}
\newtheorem{example}[theorem]{Example}

\numberwithin{equation}{section} 

\newcommand{\setheader}[2]{%
    \pagestyle{fancy}
    \fancyhf{}
    \fancyhead[LO]{#1}
    \fancyhead[RO]{}
    \fancyhead[LE]{}
    \fancyhead[RE]{#2}
    \renewcommand{\headrulewidth}{0.4pt}
    \fancyfoot[C]{\thepage}
    \renewcommand{\headrule}{\color{gray!75}\hrule width\textwidth height 1.0pt}
}

\makeatletter
\renewcommand{\maketitle}{
    \begin{center}
        {\Large\bfseries\@title\par}
        \vskip 1.5em
        {\@author\par}
        \vskip 1em
        {\@date\par}
    \end{center}
    \vskip 1.5em
}
\makeatother

\usepackage{changepage}
\makeatletter
\renewenvironment{abstract}{
    \small
    \begin{adjustwidth}{0.3in}{0.3in} 
    \noindent{\bfseries Abstract.}
    \ignorespaces
}{
    \end{adjustwidth}
    \par\vspace{0.5em}
}
\makeatother

\newcommand{\keywords}[1]{
    \begin{adjustwidth}{0.3in}{0.3in} 
    \noindent{\small\bfseries Keywords:} {\small #1}
    \end{adjustwidth}
    \par\noindent\textcolor{gray}{\rule{\linewidth}{0.5pt}}
}

\titleformat{\section}
    {\normalfont\large\bfseries\centering}
    {\thesection}{1em}{}
\titleformat{\subsection}[runin]
    {\normalfont\normalsize\bfseries}
    {\thesubsection}{1em}{}[.]
\titleformat{\subsubsection}[runin]
    {\normalfont\normalsize\itshape}
    {\thesubsubsection}{1em}{}[.]
\titleformat{\paragraph}[runin]
    {\normalfont\normalsize\bfseries}
    {\theparagraph}{1em}{}[.]

\titlespacing{\section}
    {0pt}{2.0ex plus 0.5ex minus 0.2ex}{1.0ex plus 0.1ex}
\titlespacing{\subsection}
    {0pt}{1.0ex plus 0.1ex minus 0.1ex}{1.0ex plus 0.1ex}
\titlespacing{\subsubsection}
    {0pt}{1.0ex plus 0.1ex minus 0.1ex}{1.0ex plus 0.1ex}
\titlespacing{\paragraph}
    {0pt}{0.5ex plus 0.1ex minus 0.1ex}{0.5ex plus 0.1ex}

\def\d{\,\mathrm{d}}
\def\dt{\,\mathrm{d}t}
\def\ds{\,\mathrm{d}s}

\def\what{\widehat}

\usepackage{bm,bbm}
\def\bbone{{\mathbbm{1}}}

\def\bzero{{\mathbf{0}}}

\def\vxi{{\bm{\xi}}}

\def\vepsilon{{\bm{\varepsilon}}}

\def\va{\mathbf{a}}
\def\vb{\mathbf{b}}
\def\vc{\mathbf{c}}

\def\vh{\mathbf{h}}

\def\vm{\mathbf{m}}
\def\vn{\mathbf{n}}

\def\vs{\mathbf{s}}

\def\vu{\mathbf{u}}
\def\vv{\mathbf{v}}

\def\vx{\mathbf{x}}
\def\vy{\mathbf{y}}
\def\vz{\mathbf{z}}
\def\mA{\mathbf{A}}
\def\mB{\mathbf{B}}

\def\mD{\mathbf{D}}

\def\mI{\mathbf{I}}

\def\mX{\mathbf{X}}
\def\mY{\mathbf{Y}}
\def\mZ{\mathbf{Z}}

\def\calC{{\mathcal{C}}}

\def\calF{{\mathcal{F}}}
\def\calG{{\mathcal{G}}}

\def\calL{{\mathcal{L}}}
\def\calM{{\mathcal{M}}}

\def\calO{{\mathcal{O}}}

\def\calR{{\mathcal{R}}}

\def\calT{{\mathcal{T}}}

\def\bbE{{\mathbb{E}}}

\def\bbP{{\mathbb{P}}}

\def\bbR{{\mathbb{R}}}

\def\bbW{{\mathbb{W}}}

\usepackage{mathrsfs}

\usepackage{tikz}
\usetikzlibrary{positioning,fit,tikzmark,calc}
\usetikzlibrary{decorations.pathreplacing,calligraphy}
\usepackage{multirow}
\usepackage{pifont}
\usepackage{makecell}
\usepackage[normalem]{ulem}
\usepackage{subcaption}

\def\kl{\mathrm{KL}}
\def\iid{\mathrm{i.i.d.}}

\newcommand{\given}{\mkern 2mu | \mkern 2mu }

\DeclareMathOperator{\supp}{supp}
\DeclareMathOperator{\op}{op}
\DeclareMathOperator{\tv}{TV}

\DeclareMathOperator{\ent}{Ent}
\DeclareMathOperator{\prior}{prior}
\DeclareMathOperator{\post}{post}

\DeclareMathOperator{\LSI}{LSI}
\DeclareMathOperator{\SG}{SG}
\def\d{\,\mathrm{d}}

\def\dt{\,\mathrm{d}t}
\def\ds{\,\mathrm{d}s}
\def\du{\,\mathrm{d}u}

\def\trace{\mathrm{trace}}
\def\cov{\mathrm{Cov}}

\setheader{\textcolor{black}{Provable Diffusion Posterior Sampling}}{\textcolor{black}{Chang et al.}}

\title{Provable Diffusion Posterior Sampling for Bayesian Inversion}
\author[1]{Jinyuan Chang}
\author[2]{Chenguang Duan}
\author[3]{Yuling Jiao}
\author[4]{Ruoxuan Li}
\author[5]{Jerry Zhijian Yang}
\author[3]{Cheng Yuan}

\affil[1]{Joint Laboratory of Data Science and Business Intelligence, Southwestern University of Finance and Economics, Chengdu, Sichuan 611130, China. \protect\\ \texttt{changjinyuan@swufe.edu.cn}}
\affil[2]{Institut f{\"u}r Geometrie und Praktische Mathematik, RWTH Aachen University, Templergraben 55, Aachen 52056, Germany. \protect\\ \texttt{duan@igpm.rwth-aachen.de}}
\affil[3]{School of Artificial Intelligence, Wuhan University, Wuhan, Hubei 430072, China. \protect\\ \texttt{\{yulingjiaomath,yuancheng\}@whu.edu.cn}}
\affil[4]{Institute of Applied Mathematics, Academy of Mathematics and Systems Science, Chinese Academy of Sciences, Beijing 100190, China. \protect\\ \texttt{ruoxuanli.math@amss.ac.cn}}
\affil[5]{Institute for Math and AI, Wuhan University, Wuhan, Hubei 430072, China. \protect\\ \texttt{zjyang.math@whu.edu.cn}}
\date{\today}

\begin{document}

\thispagestyle{plain}
\maketitle

\begin{abstract}
We propose a novel diffusion-based posterior sampling method within a plug-and-play framework. Our approach constructs a probability transport from an easy-to-sample distribution to the target posterior via a diffusion process. To initialize the sampler efficiently, we introduce a warm-start strategy for the particles. The posterior score is then approximated using a Monte Carlo estimator in which samples are generated via Langevin dynamics, avoiding the heuristic approximations prevalent in prior work. The score function driving the Langevin dynamics is learned from data, enabling the model to capture rich structural features of the underlying prior. We also establish non-asymptotic error bounds in Wasserstein-2 distance guaranteeing convergence of the proposed method even for complex, multimodal posterior distributions. We corroborate our theoretical findings with numerical experiments demonstrating the effectiveness of the method across a variety of inverse problems.
\end{abstract}

\begin{adjustwidth}{0.3in}{0.3in}
\end{adjustwidth}

\section{Introduction}\label{section:introduction}

Inverse problems play a crucial role across diverse fields, encompassing applications such as imaging reconstruction~\citep{Kornak2026Bayesian}, data assimilation~\citep{Reich2015Probabilistic,Reich2019Data,ding2024nonlinear}, physics sciences~\citep{zheng2025inversebench}, digital twins~\citep{opper2025digital}, and generative artificial intelligence~\citep{chung2023diffusion,uehara2025inference}. The goal of a statistical inverse problem is to recover an unknown signal of interest $\mX_{0}\in\bbR^{d}$ from indirect and noisy measurements $\mY\in\bbR^{n}$, which are linked through
\begin{equation}\label{eq:forward}
\mY=\calF(\mX_{0})+\vn\,.
\end{equation}
Here $\calF:\bbR^{d}\rightarrow\bbR^{n}$ represents a known differentiable forward operator, and the random variable $\vn\in\mathbb{R}^{n}$ denotes a measurement noise with a known density $\rho$ but unknown realization. Consequently, the likelihood of observing $\mY=\vy$ given $\mX_{0}=\vx_{0}$ is given by $p_{\mY\given\mX_{0}}(\vy\given\vx_{0})\coloneqq \rho(\vy-\calF(\vx_{0}))$. Inverse problems are typically ill-posed due to measurement noise, limited data availability, and intrinsic information loss introduced by the forward operator. As a result, a unique solution may not exist, or, if it exists, it may be sensitive to small perturbations in the measurements $\vy$~\citep{hadamard1902problemes}. To ensure well-posedness, additional information about the unknown signal $\mX_{0}$ must be incorporated to regularize the problem.

Within the Bayesian framework, prior knowledge about the unknown signal is encoded through a probability measure $\pi_{0}$, referred to as the prior distribution. Applying Bayes' rule, the posterior distribution of $\mX_{0}$ given $\mY=\vy$ is expressed as
\begin{equation}\label{eq:posterior}
\underbrace{p_{\mX_{0}\given\mY}(\vx_{0}\given\vy)}_{\text{posterior}}\propto\underbrace{\exp\{-\ell_{\vy}(\vx_{0})\}}_{\text{likelihood}}\underbrace{\pi_{0}(\vx_{0})}_{\text{prior}}\,, \quad \vx_{0}\in\mathbb{R}^{d}\,,
\end{equation}
where $\ell_{\vy}(\cdot)\coloneq -\log p_{\mY\given\mX_{0}}(\vy\given\cdot)$ denotes the negative log-likelihood associated with the measurement $\mY=\vy$. The posterior distribution~\eqref{eq:posterior} enforces fidelity to the data through the likelihood, while incorporating prior knowledge through the prior distribution. Extensive work has established conditions for stability and well-posedness of Bayesian inverse problems~\citep{Stuart2010Inverse,Dashti2017Bayesian}.

To solve the Bayesian inverse problem, two primary methodological paradigms have been developed over recent decades~\citep{Stuart2010Inverse,Trillos2023from,Durmus2025Generative}: optimization-based methods and sampling-based methods. Optimization-based approaches, also known as variational methods, seek a reconstruction of the unknown signal by identifying the most probable $\vx_{0}$ under the posterior distribution~\eqref{eq:posterior}. This leads to the maximum a posteriori estimator
\begin{equation*}
\vx_{0}^{*}\in\arg\max_{\vx_{0}\in\bbR^{d}}p_{\mX_{0}\given\mY}(\vx_{0}\given\vy)=\arg\min_{\vx_{0}\in\bbR^{d}}\{\ell_{\vy}(\vx_{0})-\log\pi_{0}(\vx_{0})\}\,.
\end{equation*}
Here, the data-fidelity term $\ell_{\vy}(\cdot)$ enforces consistency with the observations, while the regularization term $-\log \pi_{0}(\cdot)$ encodes prior information about the unknown signal. This line of work is closely connected to Tikhonov regularization~\citep{Ito2014Inverse,Benning2018Modern,ji2025potential}. Although optimization-based methods yield a single point estimate, they do not capture the full structure of the posterior distribution, and thus cannot estimate other posterior statistics of interest, such as the mean, variance, and task-specific functional moments, which take the form
\begin{equation*}
\mathbb{E}\{\phi(\mX_{0})\given\mY=\vy\}=\int\phi(\vx_{0})p_{\mX_{0}\given\mY}(\vx_{0}\given\vy)\d\vx_{0}\,,
\end{equation*}
where $\phi(\cdot)$ is a test function of interest. Given the typically high dimensionality of $\mX_{0}$, a natural approach is to employ Monte Carlo methods to estimate this expectation: one generates a collection of approximate posterior samples $\mX_{0,1}^{\vy},\ldots,\mX_{0,n}^{\vy}\sim p_{\mX_{0}\given\mY}(\cdot\given\vy)$ and forms the Monte Carlo approximation
\begin{equation*}
\mathbb{E}\{\phi(\mX_{0})\given\mY=\vy\}\approx\frac{1}{n}\sum_{i=1}^{n}\phi(\mX_{0,i}^{\vy})\,.
\end{equation*}
This motivates the development of sampling-based methods~\citep{bach2025machine,li2025optimal} for Bayesian inverse problems. To effectively tackle Bayesian inverse problems via sampling, two central challenges must be addressed: 
\begin{itemize}
\item {\it How can we construct a prior distribution $\pi_{0}(\cdot)$ that accurately reflects features of real-world data?} 

\item {\it How can we efficiently and accurately sample from the resulting potentially complex posterior distribution $p_{\mX_{0}\given\mY}(\cdot\given\vy)$?}
\end{itemize}

Conventional priors are typically chosen as simple parametric distributions based on domain knowledge or intuition, such as Gaussian priors for smoothness and Laplacian priors for sparsity. Although computationally convenient, these priors present significant limitations. The choice of their parameters, such as the mean and covariance of a Gaussian prior, must be determined based on expert knowledge. Furthermore, these simple priors often fail to capture the rich underlying structure present in real-world data. In contrast, non-parametric data-driven approaches have been applied to a wide range of methods for Bayesian inverse problems, such as deep image prior~\citep{Ulyanov2018Deep}, GAN-based prior~\citep{Bohra2022Bayesian,Patel2022Solution,Dasgupta2024dimension}, normalizing flow-based prior~\citep{cai2024nf}, score-based prior~\citep{laumont2022bayesian,ding2024nonlinear}, and consistency model-based priors~\citep{purohit2024posterior}. Despite their strong ability to approximate complex prior distributions, there remains a need for methods that can effectively and reliably sample from multimodal posteriors.

Classical posterior sampling is dominated by Markov chain Monte Carlo (MCMC) methods~\citep{Gamerman2006Markov,Conrad2016Accelerating,Durmus2022Proximal}. However, the convergence guarantees for classical MCMC typically rely on the log-concavity of the target distribution~\citep{Bakry2014Analysis}, causing them to struggle in multimodal settings. Recent advances in diffusion models~\citep{ho2020denoising,song2021scorebased} and flow-based models~\citep{albergo2023building,lipman2023flow,albergo2025stochastic,Xie2026Flow} have led to substantial progress in sampling from multimodal distributions~\citep{Chen2024Opportunities,grenioux2024stochastic,He2024Zeroth,huang2024reverse,Huang2025Schrodinger,wu2025annealing}. These models are particularly well suited for multimodal sampling because they learn a smooth probability transport from a simple Gaussian base distribution to a complex target distribution. By constructing this continuous transformation path, they can naturally explore and recover separate modes of the distribution, provided that these modes are represented in the learned score (or velocity) network. These approaches avoid the limitations of local exploration inherent to classical MCMC, enabling efficient sampling in highly multimodal landscapes. A growing body of theoretical work further establishes their convergence guarantees in non-log-concave settings~\citep{Chen2023Improved,Lee2023Convergence,Oko2023Diffusion,ding2024characteristic,li2024faster,Tang2024Adaptivity,beyler2025convergence,Fan2025Optimal,kremling2025nonasymptotic,Huang2026Denoising}.

Bringing these developments together, integrating diffusion models with data-driven priors offers a promising avenue toward a unified framework that couples expressive, data-driven prior models with powerful multimodal posterior sampling capabilities~\citep{Xu2024Provably,Janati2025Bridging,jiao2025unified,Durmus2025Generative}. However, diffusion-based posterior sampling remains challenging in practice due to the difficulty of accurately estimating the posterior score. While the posterior score for Gaussian likelihoods, i.e., linear forward operators with Gaussian noise, can be obtained directly from the unconditional score~\citep{Guo2024Gradient}, the posterior score under general likelihoods is not directly accessible. As a result, most existing methods rely on heuristic approximations~\citep{chung2023diffusion,song2023pseudoinverse,Guo2024Gradient,li2025efficient}. These heuristic approximations introduce intrinsic bias into posterior estimation and lack theoretical guarantees. Further discussion appears in Appendix~\ref{section:method:previous:work}.

\subsection{Contributions}
To fully leverage the prior information contained in the data and avoid the heuristic 
approximations prevalent in existing literature, we propose a diffusion-based posterior sampling method with a data-driven prior that provides rigorous theoretical guarantees, even when the target posterior is multimodal. The posterior score is estimated via a Monte Carlo procedure driven by Langevin dynamics, and the initial particles for the posterior time-reversal process are generated through a Langevin-based warm-start 
strategy. Our main contributions are as follows:
\begin{enumerate}[label=(\roman*)]
\item We introduce a Monte Carlo-based score estimator for diffusion posterior sampling via Langevin dynamics, free of heuristic approximations, built upon a data-driven prior score that is independent of the measurement model. This plug-and-play design is compatible with general differentiable likelihood functions without requiring retraining.
\item We establish non-asymptotic convergence guarantees in the Wasserstein-2 distance, decomposing the total error into posterior score estimation error and the warm-start error, and derive rigorous bounds for each. The analysis quantifies the influence of the prior score accuracy and the condition number $\kappa_{\vy}$, which measures the potential mismatch between the prior and the likelihood (see Section~\ref{section:method:warm}), on the posterior sampling error, and provides practical guidance for hyperparameter selection. Crucially, these guarantees hold even when the target posterior is multimodal, a regime in which classical log-concave sampling methods fail to converge.
\item We demonstrate the numerical performance of our proposed method on a range of inverse problems with both linear and nonlinear measurement operators. Our method achieves competitive reconstruction quality and exhibits the ability to quantify uncertainty.
\end{enumerate}

\subsection{Notation and Organization}
We use $\bbR^{d}$ to denote the $d$-dimensional Euclidean space. Deterministic vectors in $\bbR^{d}$ are denoted by bold lowercase letters (e.g., $\vx, \vy$), and random vectors or stochastic processes are denoted by specific capitalized bold letters (e.g., $\mX_{t}, \mY, \mB_{t}$). The $d$-dimensional identity matrix and the all-zero vector are denoted by $\mI_{d}$ and $\bzero$, respectively. For a vector $\vx \in \bbR^{d}$, its Euclidean norm is denoted by $\|\vx\|_{2}$. We use $N(\boldsymbol{\mu}, \boldsymbol{\Sigma})$ to denote a multivariate Gaussian distribution with mean vector $\boldsymbol{\mu}$ and covariance matrix $\boldsymbol{\Sigma}$, and let $\gamma_{d, \sigma^{2}}$ denote the probability density function of $N(\bzero, \sigma^{2}\mI_{d})$. For a measurable map $\calT$ and a probability measure $\mu$, the push-forward measure $\calT \sharp \mu$ is defined by $(\calT \sharp \mu)(B) := \mu(\calT^{-1}(B))$ for any measurable set $B$. For two positive sequences $\{a_{n}\}_{n \geq 1}$ and $\{b_{n}\}_{n \geq 1}$, the asymptotic notation $a_{n} = \calO(b_{n})$ and $a_{n} = \Theta(b_{n})$ imply $a_{n} \leq C b_{n}$ and $c b_{n} \leq a_{n} \leq C b_{n}$ for some universal constants $c, C > 0$, respectively. Given two probability distributions $\mu$ and $\nu$, their total variation distance is defined as $\|\mu - \nu\|_{\tv} := \sup_{A} |\mu(A) - \nu(A)|$, where the supremum is taken over all measurable sets $A$. Assuming $\mu$ is absolutely continuous with respect to $\nu$ (denoted by $\mu \ll \nu$), the $\chi^{2}$-divergence is defined as $\chi^{2}(\mu \, \| \, \nu) := \int (\d\mu / \d\nu - 1)^{2} \d\nu$, where $\d\mu / \d\nu$ is the Radon-Nikodym derivative. Their Wasserstein-2 distance $\bbW_{2}(\mu, \nu)$ is defined via $\bbW_{2}^{2}(\mu, \nu) := \inf_{\pi \in \Pi(\mu, \nu)} \bbE_{(\mX, \mY) \sim \pi}(\|\mX - \mY\|_{2}^{2})$, where $\Pi(\mu, \nu)$ is the set of all couplings of $\mu$ and $\nu$.

The rest of this paper is organized as follows. Section~\ref{section:preliminary} introduces the diffusion-based framework for posterior sampling. Section~\ref{section:method} presents our posterior sampling method. Section~\ref{section:convergence} provides the non-asymptotic error analysis. Section~\ref{section:experiments} evaluates the performance of our proposed method through numerical experiments. Section~\ref{section:conclusion} provides some discussion. Related work, proofs of theoretical results, and additional experimental results are deferred to the appendices.

\section{Preliminaries}
\label{section:preliminary}
\subsection{Posterior diffusion models}
\label{section:method:diffusion}
Before proceeding, we introduce the forward and time-reversal processes for sampling from the prior distribution. Let $\pi_{0}$ be the prior density of the unknown signal $\mX_{0}$ in~\eqref{eq:forward}. Consider the Ornstein-Uhlenbeck process 
\begin{equation}\label{eq:method:forward}
\d\mX_{t}=-\mX_{t}\dt+\sqrt{2}\d\mB_{t} \, , \quad \mX_{0}\sim\pi_{0}\, , \quad t\in(0,T] \, ,
\end{equation}
where $T>0$ is the terminal time, and $(\mB_{t})_{t\geq 0}$ is a $d$-dimensional Brownian motion. The stochastic differential equation (SDE)~\eqref{eq:method:forward} can be interpreted as Langevin dynamics with a Gaussian stationary distribution $N(\bzero,\mI_{d})$. The transition distribution of this linear SDE admits an explicit form~\citep{sarkka2019applied}:
\begin{equation}\label{eq:method:forward:solution}
\mX_{t}\given\mX_{0}=\vx_{0}\sim N(\mu_{t}\vx_{0},\sigma_{t}^{2}\mI_{d}) \, , \quad t\in(0,T] \, ,
\end{equation}
where $\mu_{t}=e^{-t}$ and $\sigma_{t}^{2}= 1-e^{-2t}$. The marginal density $\pi_{t}$ of 
$\mX_{t}$ is given by 
\begin{equation}\label{eq:method:marginal}
\pi_{t}(\vx)=\int p_{\mX_{t}\given\mX_{0}}(\vx\given\vx_{0})\pi_{0}(\vx_{0})\d\vx_{0}=\int\gamma_{d,\sigma_{t}^{2}}(\vx-\mu_{t}\vx_{0})\pi_{0}(\vx_{0})\d\vx_{0} \, , \quad \vx\in\bbR^{d} \, ,
\end{equation}
where $p_{\mX_{t}\given\mX_{0}}(\cdot\given\vx_{0})$ represents the conditional density of $\mX_{t}$ given $\mX_{0}=\vx_{0}$. 
Since the Brownian motion of the forward process~\eqref{eq:method:forward} is independent of the measurement noise $\mathbf{n}$ in~\eqref{eq:forward}, the measurement $\mathbf{Y}$ is conditional independent of $\mathbf{X}_{t}$ given $\mathbf{X}_{0}$.

\subsubsection{Forward process for posterior sampling}
We now consider the posterior forward process, which is governed by the same SDE as~\eqref{eq:method:forward}:
\begin{equation}\label{eq:method:forward:y}
\d\mX_{t}^{\vy}=-\mX_{t}^{\vy}\dt+\sqrt{2}\d\mB_{t}\,, \quad \mX_{0}^{\vy}\sim q_{0}(\cdot\given\vy):= p_{\mX_{0}\given\mY}(\cdot\given\vy)\,, \quad t\in(0,T]\,,
\end{equation}
where $p_{\mX_{0}\given\mY}(\cdot\given\vy)$ denotes the conditional density of $\mX_{0}$ given $\mY=\vy$. Throughout, the measurement $\vy$ is a fixed vector for which this conditional density is well defined. Since this process shares the same transition distribution~\eqref{eq:method:forward:solution} as~\eqref{eq:method:forward}, the marginal density $q_{t}(\cdot\given\vy)$ of the posterior forward process $\mX_{t}^{\vy}$ follows analogously to~\eqref{eq:method:marginal}:
\begin{equation}\label{eq:method:marginal:y}
q_{t}(\vx\given\vy)=\int\gamma_{d,\sigma_{t}^{2}}(\vx-\mu_{t}\vx_{0})q_{0}(\vx_{0}\given\vy)\d\vx_{0}\,, \quad \vx\in\bbR^{d}\,.
\end{equation}
A key insight is that the marginal density $q_{t}(\cdot\given\vy)$ of the posterior forward process coincides with the conditional density $p_{\mX_{t}\given\mY}(\cdot\given\vy)$. To establish this, we apply Bayes' rule and the transition density~\eqref{eq:method:forward:solution}:
\begin{align}
q_{t}(\vx\given\vy)
&=\int\gamma_{d,\sigma_{t}^{2}}(\vx-\mu_{t}\vx_{0})\frac{p_{\mY\given\mX_{0}}(\vy\given\vx_{0})\pi_{0}(\vx_{0})}{p_{\mY}(\vy)}\d\vx_{0} \nonumber \\
&=\frac{1}{p_{\mY}(\vy)}\int p_{\mX_{t}\given\mX_{0}}(\vx\given\vx_{0})\,p_{\mY\given\mX_{0}}(\vy\given\vx_{0})\,\pi_{0}(\vx_{0})\d\vx_{0} \nonumber \\
&=\frac{1}{p_{\mY}(\vy)}\int p_{\mX_{t}\given\mX_{0}}(\vx\given\vx_{0})\,p_{\mY\given\mX_{t},\mX_{0}}(\vy\given\vx,\vx_{0})\,\pi_{0}(\vx_{0})\d\vx_{0} \label{eq:bayes:posterior:t} \\
&=\frac{p_{\mY,\mX_{t}}(\vy,\vx)}{p_{\mY}(\vy)}=p_{\mX_{t}\given\mY}(\vx\given\vy)\,, \nonumber
\end{align}
where the first equality uses~\eqref{eq:method:marginal:y} and Bayes' rule applied to $q_{0}(\vx_{0}\given\vy)$, the second recognizes the Gaussian kernel as the transition density~\eqref{eq:method:forward:solution}, and the third used the fact that $\mY$ is conditional independent of $\mX_{t}$ given $\mX_{0}$.

\subsubsection{Time-reversal process for posterior sampling}
The time-reversal process~\citep{anderson1982reverse} of the posterior forward process~\eqref{eq:method:forward:y} reads:
\begin{equation}\label{eq:reversal}
\d\bar{\mX}_{t}^{\vy}=\{\bar{\mX}_{t}^{\vy}+2\nabla_{\vx}\log q_{T-t}(\bar{\mX}_{t}^{\vy}\given\vy)\}\dt+\sqrt{2}\d\mB_{t}\,, \quad \bar{\mX}_{0}^{\vy}\sim q_{T}(\cdot\given\vy) \, , \quad t\in(0,T)\,.
\end{equation}
From~\citet{anderson1982reverse}, the marginal density of $\bar{\mX}_{t}^{\vy}$ is $q_{T-t}(\cdot\given\vy)$, which coincides with the marginal density of the forward process $\mX_{T-t}^{\vy}$. This fundamental result enables us to generate samples from the target posterior density $q_{0}(\cdot\given\vy)$ by simulating the time-reversal process~\eqref{eq:reversal}. However, a crucial question arises:
\begin{itemize}
\item {\it How can we estimate the posterior score $\nabla\log q_{t}(\cdot\given\vy)$ for $t\in(0,T)$?}
\end{itemize}
A natural starting point is the Bayes decomposition~\eqref{eq:bayes:posterior:t}, which expresses the posterior score as the sum of a time-dependent log-likelihood gradient and the time-dependent prior score:
\begin{align}
\nabla_{\vx}\log q_{t}(\vx\given\vy)
&=\nabla_{\vx}\log p_{\mY\given\mX_{t}}(\vy\given\vx)+\nabla_{\vx}\log\pi_{t}(\vx) \nonumber \\
&=\nabla_{\vx}\log\underbrace{\bbE\big[\exp\{-\ell_{\vy}(\mX_{0})\}\given\mX_{t}=\vx\big]}_{\text{time-dependent likelihood}}+\underbrace{\nabla_{\vx}\log\pi_{t}(\vx)}_{\text{time-dependent prior score}}, \label{eq:posterior:score:decomposition}
\end{align}
where the inner expectation is taken under the denoising density $p_{\mX_{0}\given\mX_{t}}(\cdot\given\vx)$, i.e., the conditional density of $\mX_{0}$ given $\mX_{t}=\vx$.
The time-dependent prior score is amenable to standard denoising score matching~\citep{vincent2011connection}, so the principal difficulty lies in the time-dependent likelihood, whose conditional expectation is not available in closed form. Existing diffusion-based posterior sampling methods typically rely on heuristic approximation to the denoising density $p_{\mX_{0}\given\mX_{t}}(\cdot\given\vx)$ to estimate the time-dependent likelihood. For example,~\citet{chung2023diffusion} and \citet{rout2023solving} approximate $p_{\mX_{0}\given\mX_{t}}(\cdot\given\vx)$ by a Dirac measure at its conditional mean, while~\citet{song2023pseudoinverse} and \citet{song2023loss} approximate it by a Gaussian density with the same conditional mean and an isotropic covariance. Both introduce structural bias in the likelihood approximation and depend on plug-in differentiation that carries no general accuracy guarantee; see Appendix~\ref{section:method:previous:work} for a detailed discussion.

In Section~\ref{section:method}, we depart from these heuristics and develop a Monte Carlo posterior score estimator via Langevin dynamics, yielding a theoretically principled scheme for diffusion-based posterior sampling.

\section{Posterior Score Estimation and Warm-Start}
\label{section:method}
We present a Monte Carlo technique for posterior score estimation in Section~\ref{section:method:score}, and then propose a warm-start strategy for sampling from the terminal posterior density in Section~\ref{section:method:warm}. Finally, we integrate these components into a complete posterior sampling algorithm in Section~\ref{section:method:posterior:sampling}.

\subsection{Posterior score estimation}
\label{section:method:score}
In this subsection, we propose a Monte Carlo-based approach for posterior score estimation. This method is inspired by the following proposition, which is a direct extension of Tweedie's formula~\citep{Efron2011Tweedie}. 

\begin{proposition}[Conditional Tweedie's formula]\label{proposition:method:score:denoiser}
For each $t\in(0,T]$, it holds that
\begin{equation}\label{eq:posterior:score}
\nabla_{\vx}\log q_{t}(\vx\given\vy)=-\frac{1}{\sigma_{t}^{2}}\vx+\frac{\mu_{t}}{\sigma_{t}^{2}}\mD(t,\vx,\vy)\,, \quad \vx\in\bbR^{d}\,,
\end{equation}
where $\mD(t,\vx,\vy)$ is the posterior denoiser, defined as the conditional expectation:
\begin{equation}\label{eq:posterior:denoiser}
\mD(t,\vx,\vy):=\bbE(\mX_{0}\given\mX_{t}=\vx,\mY=\vy)=\int\vx_{0}p_{t}(\vx_{0}\given\vx,\vy)\d\vx_{0}\,.
\end{equation} 
Here, the posterior denoising density $p_{t}(\vx_{0}\given\vx,\vy)$ is defined as:
\begin{equation}\label{eq:posterior:conditional:density}
p_{t}(\vx_{0}\given\vx,\vy):= p_{\mX_{0}\given\mX_{t},\mY}(\vx_{0}\given\vx,\vy)\propto\pi_{0}(\vx_{0})\exp\bigg\{-\frac{\|\vx-\mu_{t}\vx_{0}\|_{2}^{2}}{2\sigma_{t}^{2}}-\ell_{\vy}(\vx_{0})\bigg\}\,.
\end{equation}
\end{proposition}

The proof of Proposition~\ref{proposition:method:score:denoiser} is deferred to Appendix~\ref{appendix:proof:denoiser}. This proposition reveals the relationship between the posterior score $\nabla_{\vx}\log q_{t}(\cdot\given\vy)$ and the expectation of the posterior denoising distribution $p_{t}(\cdot\given\vx,\vy)$. 

\subsubsection{Sampling from the posterior denoising density}
Due to Proposition~\ref{proposition:method:score:denoiser}, to estimate the posterior score function $\nabla_{\vx}\log q_{t}(\cdot\given\vy)$, it suffices to generate samples that approximately follow the posterior denoising density $p_{t}(\cdot\given\vx,\vy)$. 

Sampling from the density with the form in~\eqref{eq:posterior:conditional:density} is known as the restricted Gaussian oracle (RGO)~\citep{Lee2021Structured}, whose potential function is a summation of a quadratic and the potential of the posterior density $q_{0}(\vx_{0}\given\vy)\propto\pi_{0}(\vx_{0})\exp\{-\ell_{\vy}(\vx_{0})\}$. Consequently, the RGO can be interpreted as a regularized version of sampling from the original posterior density $q_{0}(\cdot\given\vy)$ and exhibits enhanced log-concavity~\citep{Lee2021Structured}.

In our work, we implement the RGO~\eqref{eq:posterior:conditional:density} using the Langevin dynamics, following the approach adopted in~\citet{huang2024reverse,Huang2024Faster} and \citet{grenioux2024stochastic}. Specifically, we consider the Langevin dynamics with stationary density $p_{t}(\cdot\given\vx,\vy)$~\eqref{eq:posterior:conditional:density}:
\begin{align}\label{eq:RGO:Langevin}
\d\mX_{0,s}^{\vx,\vy,t}&=\bigg\{\mkern-1mu \nabla\log\pi_{0}(\mX_{0,s}^{\vx,\vy,t})+\frac{\mu_{t}}{\sigma_{t}^{2}}(\vx-\mu_{t}\mX_{0,s}^{\vx,\vy,t})\\
&~~~~~~~~~~~~~~~~~~~~~~~~~~-\nabla\ell_{\vy}(\mX_{0,s}^{\vx,\vy,t})\bigg\}\ds+\sqrt{2}\d\mB_{s}\,, \quad s\in(0,S)\,,\notag
\end{align}
where $S>0$ is the simulation horizon, and $(\mB_{s})_{s\geq 0}$ is a $d$-dimensional Brownian motion. Denote by $p_{t}^{s}(\cdot\given\vx,\vy)$ the marginal density of $\mX_{0,s}^{\vx,\vy,t}$. The Langevin dynamics with such score has also been investigated in the literature of annealed Langevin dynamics~\citep{guo2025provable}, and parallel tempering~\citep{Dong2022Spectral}.

In practical applications, the exact prior score $\nabla\log\pi_{0}$ is typically intractable. To address this, one can approximate the prior score using samples from the prior density $\pi_{0}$ with standard techniques like the implicit score matching~\citep{hyvarinen2005Estimation}, sliced score matching~\citep{song2020Sliced}, and denoising score matching~\citep{vincent2011connection}. Let $\hat{\vs}_{\prior}:\bbR^{d}\rightarrow\bbR^{d}$ denote an approximation to the prior score. We then consider the Langevin dynamics with this estimated prior score:
\begin{align}\label{eq:RGO:Langevin:score}
\d\what{\mX}_{0,s}^{\vx,\vy,t}&=\bigg\{\mkern-1mu \hat{\vs}_{\prior}(\what{\mX}_{0,s}^{\vx,\vy,t})+\frac{\mu_{t}}{\sigma_{t}^{2}}(\vx-\mu_{t}\what{\mX}_{0,s}^{\vx,\vy,t})\\
&~~~~~~~~~~~~~~~~~~~~~-\nabla\ell_{\vy}(\what{\mX}_{0,s}^{\vx,\vy,t})\bigg\}\ds+\sqrt{2}\d\mB_{s}\,, \quad s\in(0,S)\,.\notag
\end{align}
Denote by $\hat{p}_{t}^{s}(\cdot\given\vx,\vy)$ the marginal density of $\what{\mX}_{0,s}^{\vx,\vy,t}$, which serves as an approximation to $p_{t}(\cdot\given\vx,\vy)$ under certain conditions. A rigorous error analysis is provided in Proposition~\ref{proposition:posterior:score}.

\subsubsection{Monte Carlo approximation to the posterior score}
With the aid of particles sampled by the Langevin dynamics with the estimated score~\eqref{eq:RGO:Langevin:score}, the posterior denoiser $\mD(t,\vx,\vy)$ in~\eqref{eq:posterior:denoiser} can be estimated by Monte Carlo method:
\begin{equation}\label{eq:denoiser:MC}
\what{\mD}_{m}^{S}(t,\vx,\vy):=\frac{1}{m}\sum_{i=1}^{m}\what{\mX}_{0,S,i}^{\vx,\vy,t}\,, \quad \what{\mX}_{0,S,1}^{\vx,\vy,t},\ldots,\what{\mX}_{0,S,m}^{\vx,\vy,t}\overset{\iid}{\sim}\hat{p}_{t}^{S}(\cdot\given\vx,\vy)\,.
\end{equation}
Then using Proposition~\ref{proposition:method:score:denoiser} yields the following estimator of the posterior score function:
\begin{equation}\label{eq:posterior:score:estimate}
\hat{\vs}_{m}^{S}(t,\vx,\vy):=-\frac{1}{\sigma_{t}^{2}}\vx+\frac{\mu_{t}}{\sigma_{t}^{2}}\what{\mD}_{m}^{S}(t,\vx,\vy)\,.
\end{equation}

For each fixed query point $(t,\vx)$, the particles in~\eqref{eq:denoiser:MC} are i.i.d.\ with law $\hat{p}_{t}^{S}(\cdot\given\vx,\vy)$, while across query points the estimator reuses the same $m$ random seeds. Throughout the theoretical analysis, we use the admissible common-random-number realization specified in Appendix~\ref{appendix:random:field}. Under this realization, $\hat{\vs}_{m}^{S}(\cdot,\cdot,\vy)$ is a jointly measurable random field, serving below as the drift of the warm-start and time-reversal dynamics. Denote by $\calG$ the $\sigma$-field generated by these seeds. The complete procedure of the posterior score estimation is summarized as Algorithm~\ref{alg:unbiased:posterior:score}.

\begin{algorithm}[htbp]
\caption{Posterior score estimation via Monte Carlo}\label{alg:unbiased:posterior:score}
\KwIn{Measurement $\vy$, prior score estimator $\hat{\vs}_{\prior}$, time index $t$ of the diffusion model, evaluation point $\vx$, simulation horizon $S$, initial density $\hat{p}_{t}^{0}(\cdot\given\vx,\vy)$ for Langevin dynamics.}
\KwOut{Posterior score estimator $\hat{\vs}_{m}^{S}(t,\vx,\vy)$.}
\For{$i=1,\ldots,m$}{
Initialize particle by $\what{\mX}_{0,0,i}^{\vx,\vy,t}\sim\hat{p}_{t}^{0}(\cdot\given\vx,\vy)$. \\
Evolve particle to time $S$ by simulating~\eqref{eq:RGO:Langevin:score} to obtain $\what{\mX}_{0,S,i}^{\vx,\vy,t}$.
}
Compute the posterior denoiser estimate $\what{\mD}_{m}^{S}(t,\vx,\vy)$ using $\{\what{\mX}_{0,S,i}^{\vx,\vy,t}\}_{i=1}^{m}$ via~\eqref{eq:denoiser:MC}. \\
Compute the posterior score estimate $\hat{\vs}_{m}^{S}(t,\vx,\vy)$ using $\what{\mD}_{m}^{S}(t,\vx,\vy)$ via~\eqref{eq:posterior:score:estimate}. \\
\Return{$\hat{\vs}_{m}^{S}(t,\vx,\vy)$}
\end{algorithm}

\subsubsection{Log-concavity of the posterior denoising density}
Recalling $\sigma_{t}^{2}=1-\exp(-2t)$, we observe that $\sigma_{t}^{2}$ is sufficiently close to zero for small $t>0$. Under this regime, the logarithmic posterior denoising density $\log p_{t}(\cdot\given\vx,\vy)$ in~\eqref{eq:posterior:conditional:density} is dominated by its quadratic component, resulting in behavior that closely resembles a unimodal Gaussian distribution. Therefore, the posterior denoising density $p_{t}(\cdot\given\vx,\vy)$ is easy to sample when $t>0$ is small. Before formalizing this observation rigorously, we introduce the following assumption.

\begin{assumption}[Semi-log-concavity of the posterior]
\label{assumption:semi:log:concave}
For a fixed measurement $\vy\in\bbR^{n}$, $\log q_{0}(\cdot\given\vy)\in C^{2}(\bbR^{d})$, and there exist a constant $\alpha>0$, such that $-\nabla^{2}\log q_{0}(\vx_{0}\given\vy)+\alpha\mI_{d}\succeq\bzero$ for each $\vx_{0}\in\bbR^{d}$.
\end{assumption}

The semi-log-concavity condition in Assumption~\ref{assumption:semi:log:concave}, while imposing some regularity conditions on the log-density, still allows complex multimodal structures in the posterior distribution. 

According to Bayes' rule, Assumption~\ref{assumption:semi:log:concave} is equivalent to:
\begin{equation*}
-\nabla^{2}\log\pi_{0}(\vx_{0})+\nabla^{2}\ell_{\vy}(\vx_{0})+\alpha\mI_{d}\succeq\bzero\,, \quad \vx_{0}\in\bbR^{d}\,.
\end{equation*}
This condition holds, for example, if the Hessian of the negative log-likelihood $\nabla^{2}\ell_{\vy}$ is uniformly bounded below, and the Hessian of the negative log-prior $-\nabla^{2}\log\pi_{0}$ is uniformly bounded below. The examples of likelihood and prior distributions satisfying Assumption~\ref{assumption:semi:log:concave} are provided in Appendix~\ref{section:concrete:examples}. 

The following proposition establishes that the posterior denoising density $p_{t}(\cdot\given\vx,\vy)$~\eqref{eq:posterior:conditional:density} exhibits strong log-concavity when $t>0$ is sufficiently small. This strong log-concavity property is crucial because it ensures the convergence of Langevin dynamics~\eqref{eq:RGO:Langevin}, as shown by~\citet{Bakry2014Analysis} and \citet{Vempala2019Rapid}.

\begin{proposition}[Log-concavity of RGO]\label{proposition:appendix:RGO:Hessian}
Suppose Assumption~\ref{assumption:semi:log:concave} holds. For each $t\in(0,T]$,
\begin{equation*}
\bigg(\frac{\mu_{t}^{2}}{\sigma_{t}^{2}}-\alpha\bigg)\mI_{d}\preceq-\nabla_{\vx_{0}}^{2}\log p_{t}(\vx_{0}\given\vx,\vy)\,, \quad (\vx_{0},\vx)\in\bbR^{d}\times\bbR^{d}\,.
\end{equation*}
Further, the posterior denoising density is strongly log-concave provided that $t<\bar{t}:=\log(1+\alpha^{-1}) / 2$.
\end{proposition}

The proof of Proposition~\ref{proposition:appendix:RGO:Hessian} is deferred to Appendix~\ref{appendix:proof:log:concave:RGO}. Proposition~\ref{proposition:appendix:RGO:Hessian} demonstrates that posterior denoising density $p_{t}(\cdot\given\vx,\vy)$ serves as a regularized counterpart to the original posterior density $q_{0}(\cdot\given\vy)$, whereby the former can achieve log-concavity even when the latter is multimodal. 

\begin{remark}[Choice of the terminal time of the diffusion model, I]
\label{remark:terminal:1}
Proposition~\ref{proposition:appendix:RGO:Hessian} demonstrates that the posterior denoising density $p_{t}(\cdot\given\vx,\vy)$ is strongly log-concave for $t\in(0,\bar{t})$. This strong log-concavity, in turn, guarantees the convergence of the Langevin dynamics~\eqref{eq:RGO:Langevin} for the same time interval. Consequently, to ensure the validity of our posterior score estimation throughout the time-reversal process, the terminal time $T$ must satisfy $T<\bar{t}$.  The theoretical guarantees for the posterior score estimation~\eqref{eq:posterior:score:estimate} are provided in Proposition~\ref{proposition:posterior:score}. 
\end{remark}

\subsection{Warm-start strategy}
\label{section:method:warm}
This subsection presents a method for sampling from the terminal posterior density $q_{T}(\cdot\given\vy)$ of the posterior time-reversal process~\eqref{eq:reversal}, which serves as a warm-start procedure of the diffusion-based posterior sampling.

Observe that the marginal density $q_{T}(\cdot\given\vy)$~\eqref{eq:method:marginal:y} approaches the standard Gaussian distribution in the limit of large $T$. Exploiting this asymptotic behavior, the prevailing methods in diffusion-based generative modeling~\citep{ho2020denoising,song2021scorebased} involve choosing sufficiently large terminal time $T$ and directly substituting the standard Gaussian density as a surrogate for the terminal posterior density $q_{T}(\cdot\given\vy)$. This Gaussian approximation strategy, while computationally convenient, becomes problematic in our work since the terminal time $T$ should not be too large, as highlighted in Remark~\ref{remark:terminal:1}.

To address this limitation, we propose a Langevin dynamics to sample directly from the terminal posterior distribution without relying on Gaussian approximations. Specifically, we consider the Langevin dynamics with invariant density $q_{T}(\cdot\given\vy)$:
\begin{equation}\label{section:method:warm:Langevin} 
\d\mX_{T,u}^{\vy}=\nabla_{\vx}\log q_{T}(\mX_{T,u}^{\vy}\given\vy)\du+\sqrt{2}\d\mB_{u}\,, \quad u\in(0,U)\,,
\end{equation} 
where $U>0$ is the simulation horizon, and $(\mB_{u})_{u\geq 0}$ is a $d$-dimensional Brownian motion. Denote the marginal density of $\mX_{T,u}^{\vy}$ by $q_{T}^{u}(\cdot\given\vy)$. Since the terminal posterior score function $\nabla\log q_{T}(\cdot\given\vy)$ cannot be computed analytically, we approximate it using the estimator $\hat{\vs}_{m}^{S}(T,\cdot,\vy)$ developed in~\eqref{eq:posterior:score:estimate} with an additional clipping operation:
\begin{equation}\label{eq:warm:score:clipped}
\tilde{\vs}_{m}^{S}(T,\vx,\vy):=\Pi_{a_{T}(\vx,\vy)}\big\{\hat{\vs}_{m}^{S}(T,\vx,\vy)\big\}\,,
\end{equation}
where $\Pi_{a}$ denotes the Euclidean projection onto the ball $\{\vu\in\bbR^{d}:\|\vu\|_{2}\leq a\}$. The truncation radius $a_{T}(\vx,\vy)$ will be specified in Appendix~\ref{section:proof:warm:start}. Substituting this approximation into equation~\eqref{section:method:warm:Langevin} results in the following approximate Langevin dynamics:
\begin{equation}\label{section:method:warm:Langevin:score} 
\d\what{\mX}_{T,u}^{\vy}=\tilde{\vs}_{m}^{S}(T,\what{\mX}_{T,u}^{\vy},\vy)\du+\sqrt{2}\d\mB_{u}\,, \quad u\in(0,U)\,.
\end{equation} 
Since the drift of~\eqref{section:method:warm:Langevin:score} is the realized score field, we denote by $\hat{q}_{T}^{u}(\cdot\given\vy)$ the conditional marginal density of $\what{\mX}_{T,u}^{\vy}$ given $\calG$, which is a random probability measure. Algorithm~\ref{alg:warm:start} provides a description of the sampling procedure for the terminal posterior distribution.

\begin{algorithm}[htbp]
\caption{Warm-start for diffusion posterior sampling}\label{alg:warm:start}
\KwIn{Measurement $\vy$, simulation horizon $U$, initial density $\hat{q}_{T}^{0}(\cdot\given\vy)$ for Langevin dynamics.}
\KwOut{Sample $\what{\mX}_{T,U}^{\vy}$ approximately distributed according to $q_{T}(\cdot\given\vy)$.}
Initialize particle by $\what{\mX}_{T,0}^{\vy}\sim\hat{q}_{T}^{0}(\cdot\given\vy)$. \\
Simulate the approximate Langevin dynamics from time $0$ to $U$ using the clipped score function $\tilde{\vs}_{m}^{S}(T,\vx,\vy)$ computed via Algorithm~\ref{alg:unbiased:posterior:score} to obtain $\what{\mX}_{T,U}^{\vy}$. \\
\Return{$\what{\mX}_{T,U}^{\vy}$}
\end{algorithm}

\subsubsection{Log-Sobolev inequality of the terminal distribution}
The preceding warm-start strategy is based on the Langevin dynamics~\eqref{section:method:warm:Langevin}. According to~\citet{Bakry2014Analysis}, the convergence of~\eqref{section:method:warm:Langevin} can be guaranteed, provided that $q_{T}(\cdot\given\vy)$ satisfies the log-Sobolev inequality~\eqref{eq:LSI}.

As established in~\citet{Bakry2014Analysis}, the marginal density $q_{T}(\cdot\given\vy)$ converges exponentially to the standard Gaussian distribution $N(\bzero,\mI_{d})$, which satisfies the log-Sobolev inequality. This convergence behavior suggests that there exists a threshold $\underline{t}>0$ such that $q_{T}(\cdot\given\vy)$ satisfies the log-Sobolev inequality for all $T>\underline{t}$.

To establish this result rigorously, we first introduce the following assumption regarding the concentration of the prior density $\pi_{0}$.

\begin{assumption}[Sub-Gaussian tails of the prior]
\label{assumption:prior:subGaussian}
The log-prior density $\log\pi_{0}\in C^{2}(\bbR^{d})$, and there exist constants $V_{\SG}>0$ and $C_{\SG}>0$, such that 
\begin{equation*}
\int\exp\bigg(\frac{\|\vx_{0}\|_{2}^{2}}{V_{\SG}^{2}}\bigg)\pi_{0}(\vx_{0})\d\vx_{0}\leq C_{\SG}\,.
\end{equation*}
\end{assumption}

Assumption~\ref{assumption:prior:subGaussian} proposes mild conditions on the prior distribution. These conditions do not necessitate log-concavity and are weaker than the log-Sobolev inequality~\citep[Theorem 9.9]{Villani2003Topics}. Gaussian distributions,  Gaussian mixture (Example~\ref{example:gaussian:mixture}), and the Gaussian convolution (Example~\ref{example:gaussian:convolution}) satisfy this assumption; see Appendix~\ref{appendix:examples} for details.

Assumption~\ref{assumption:prior:subGaussian} implies the sub-Gaussian tails of the posterior, as stated by the following proposition.

\begin{proposition}[Sub-Gaussian tails of the posterior]
\label{proposition:sub:Gaussian:posterior}
Suppose Assumption~\ref{assumption:prior:subGaussian} holds. Then the posterior density $q_{0}(\cdot\given\vy)$ also has sub-Gaussian tails, that is,
\begin{equation*}
\int\exp\bigg(\frac{\|\vx_{0}\|_{2}^{2}}{V_{\SG}^{2}}\bigg)q_{0}(\vx_{0}\given\vy)\d\vx_{0}\leq\kappa_{\vy}C_{\SG}\,,
\end{equation*}
where $\kappa_{\vy}$ is the condition number defined as:
\begin{equation}\label{eq:posterior:score:condition} 
\kappa_{\vy}:=\frac{\sup_{\vx_{0}}\exp\{-\ell_{\vy}(\vx_{0})\}}{\int\exp\{-\ell_{\vy}(\vx_{0})\}\pi_{0}(\vx_{0})\d\vx_{0}}<\infty\,.
\end{equation}
\end{proposition}

The proof of Proposition~\ref{proposition:sub:Gaussian:posterior} is deferred to Appendix~\ref{appendix:proof:sub:Gaussian:posterior}. 

\begin{remark}[Condition number]
\label{remark:condition:number}
The condition number $\kappa_{\vy}$ in~\eqref{eq:posterior:score:condition} measures the difficulty of the posterior sampling problem in our analysis~\citep[Theorem 4.1]{purohit2024posterior}: it is moderate when the likelihood concentrates where the prior density is large, and large when they conflict. A detailed interpretation is provided in Appendix~\ref{appendix:discussions:condition:number}, and its rigorous influence is established in Proposition~\ref{proposition:posterior:score}. Note that $\kappa_{\vy}\geq1$, since the evidence $\int\exp\{-\ell_{\vy}(\vx_{0})\}\pi_{0}(\vx_{0})\d\vx_{0}$ is bounded above by $\sup_{\vx_{0}}\exp\{-\ell_{\vy}(\vx_{0})\}$.
\end{remark}

The following proposition shows that when $T$ is sufficiently large, the terminal posterior density $q_{T}(\cdot\given\vy)$ satisfies a log-Sobolev inequality. Further, we can explicitly estimate a log-Sobolev inequality constant using~\citet[Theorem 2]{Chen2021Dimension}.

\begin{proposition}[Log-Sobolev inequality]
\label{proposition:log:sobolev}
Suppose Assumption~\ref{assumption:prior:subGaussian} holds and $\kappa_{\vy}^{2}C_{\SG}^{2}\geq e$. For each $T>\underline{t}:=\log(1+2V_{\SG}^{2}) / 2$,
\begin{equation*}
C_{\LSI}\{q_{T}(\cdot\given\vy)\}\leq 12\sigma_{T}^{2}\exp\bigg\{2\times \frac{\sigma_{T}^{2}+2\mu_{T}^{2}V_{\SG}^{2}}{\sigma_{T}^{2}-2\mu_{T}^{2}V_{\SG}^{2}}\log(\kappa_{\vy}^{2}C_{\SG}^{2})\bigg\}\,.
\end{equation*}
\end{proposition}

The proof of Proposition~\ref{proposition:log:sobolev} is deferred to Appendix~\ref{appendix:proof:log:sobolev}. This proposition shows that sufficient Gaussian smoothing upgrades a sub-Gaussian-tailed distribution to one satisfying the log-Sobolev inequality, enabling efficient Langevin sampling. This transformation is crucial because the original posterior $q_{0}(\cdot\given\vy)$ cannot in general be efficiently sampled by Langevin dynamics whereas the smoothed terminal $q_{T}(\cdot\given\vy)$ can.

\begin{remark}[Choice of the terminal time of the diffusion model, II]
\label{remark:terminal:2}
The condition $T>\underline{t}$ guarantees the log-Sobolev inequality of $q_{T}(\cdot\given\vy)$ and hence the convergence of the warm-start Langevin dynamics~\eqref{section:method:warm:Langevin}. Similar findings appear in~\citet{grenioux2024stochastic} and \citet{beyler2025convergence}: \citet[Theorem~1]{grenioux2024stochastic} establish strong log-concavity under a Gaussian convolution assumption on the target, and~\citet[Lemma~3]{beyler2025convergence} prove a comparable result under compact support. Our Proposition~\ref{proposition:log:sobolev}, in contrast, applies to posterior distributions under the more transparent and directly verifiable Assumption~\ref{assumption:prior:subGaussian}.
\end{remark}

\subsection{Posterior sampling and the duality of convergence}
\label{section:method:posterior:sampling}
Utilizing the posterior score estimation approach and warm-start strategy introduced in Sections~\ref{section:method:score} and~\ref{section:method:warm}, respectively, we obtain an approximation to the posterior time-reversal process~\eqref{eq:reversal} given by:
\begin{align}\label{eq:reversal:score}
&\d\what{\mX}_{t}^{\vy}=\big\{\what{\mX}_{t}^{\vy}+2\hat{\vs}_{m}^{S}(T-t,\what{\mX}_{t}^{\vy},\vy)\big\}\dt+\sqrt{2}\d\mB_{t}\,, \\
&~~~~~~~~~~~~~~~~~~~~~~~~~~~~~~~~~~~~~~~~~~~~~~~~~~~~~~~\what{\mX}_{0}^{\vy}{:=\what{\mX}_{T,U}^{\vy}}\sim\hat{q}_{T}^{U}(\cdot\given\vy)\,, ~ t\in(0,T-T_{0})\,, \notag
\end{align}
where the posterior estimator $\hat{\vs}_{m}^{S}(t,\cdot,\vy)$ is defined as~\eqref{eq:posterior:score:estimate}, and the approximate terminal posterior density $\hat{q}_{T}^{U}(\cdot\given\vy)$ is characterized by~\eqref{section:method:warm:Langevin:score}. The process starts from the actual warm-start output $\what{\mX}_{T,U}^{\vy}$ and is driven by a fresh Brownian motion, with the precise independence structure specified in Appendix~\ref{appendix:random:field}. Denote by $\hat{q}_{t}(\cdot\given\vy)$ the conditional marginal density of $\what{\mX}_{t}^{\vy}$ given $\calG$, as for the warm-start stage. Here $T_{0}\in(0,T)$ is the early-stopping time, introduced to avoid the score estimation singularity near $t=0$; see Appendix~\ref{appendix:discussions:early:stopping} for a detailed discussion.

\par From Assumption~\ref{assumption:prior:subGaussian} and Proposition~\ref{proposition:sub:Gaussian:posterior}, the target posterior distribution admits a light tail. Nevertheless, the particles $\what{\mX}_{T-T_{0}}^{\vy}$ generated by the approximate time-reversal process~\eqref{eq:reversal:score} may deviate substantially from the origin. To address this, we apply a truncation operator $\calT_{R}(\vx):=\vx\bbone\{\|\vx\|_{2}\leq R\}$, which clips heavy-tailed outliers to a ball of radius $R$. The resulting truncated distribution is denoted by $\widehat{q}_{T-T_{0}}^{R}(\cdot\given\vy)\coloneqq\calT_{R}\sharp\hat{q}_{T-T_{0}}(\cdot\given\vy)$, and theoretical guidance for selecting the truncation radius $R$ is provided in Proposition~\ref{proposition:error:decomposition}. Furthermore, to compensate for the bias introduced by early stopping, we apply a rescaling map $\calM(\mu_{T_{0}}^{-1}):\vx\mapsto\mu_{T_{0}}^{-1}\vx$, motivated by the observation $\bbE(\mX_{T_{0}}^{\vy})=\mu_{T_{0}}\bbE(\mX_{0}^{\vy})$. The post-processed distribution $\calM(\mu_{T_{0}}^{-1})\sharp\widehat{q}_{T-T_{0}}^{R}(\cdot\given\vy)$ thus serves as our final approximation to a draw from the posterior $q_{0}(\cdot\given\vy)$. The complete procedure is summarized in Algorithm~\ref{alg:posterior:sampling}.

\begin{algorithm}[htbp]
\caption{Diffusion-based posterior sampling}\label{alg:posterior:sampling}
\KwIn{Measurement $\vy$, diffusion terminal time $T$, early-stopping time $T_{0}$, truncation radius $R$.}
\KwOut{Post-processed sample $\calM(\mu_{T_{0}}^{-1})\circ\calT_{R}(\what{\mX}_{T-T_{0}}^{\vy})$ approximately from $q_{0}(\cdot\given\vy)$.}
\textbf{Warm-start:} Set the initial particle $\what{\mX}_{0}^{\vy} = \what{\mX}_{T,U}^{\vy}$, where the warm-start sample $\what{\mX}_{T,U}^{\vy}$ is obtained using Algorithm~\ref{alg:warm:start}. \\
\textbf{Reverse diffusion:} Simulate the approximate time-reversal stochastic process defined in~\eqref{eq:reversal:score} over the time interval $[0, T-T_{0}]$, employing the estimated score function $\hat{\vs}_{m}^{S}(T-t,\vx,\vy)$ computed via Algorithm~\ref{alg:unbiased:posterior:score}, to generate $\what{\mX}_{T-T_{0}}^{\vy}$. \\
\textbf{Post-processing:} Truncate via $\calT_{R}$ and rescale by $\mu_{T_{0}}^{-1}$ to obtain $\calM(\mu_{T_{0}}^{-1})\circ\calT_{R}(\what{\mX}_{T-T_{0}}^{\vy})$. \\
\Return{$\calM(\mu_{T_{0}}^{-1})\circ\calT_{R}(\what{\mX}_{T-T_{0}}^{\vy})$}
\end{algorithm}

As discussed in Remark~\ref{remark:terminal:1}, the terminal time $T$ must satisfy $T<\bar{t}$ to ensure convergence of the posterior score estimator $\hat{\vs}_{m}^{S}(t,\cdot,\vy)$ for all $t\in(0,T]$. Conversely, as established in Remark~\ref{remark:terminal:2}, the terminal time $T$ must also satisfy $T>\underline{t}$ to guarantee reliable convergence when sampling from the terminal posterior density $q_{T}(\cdot\given\vy)$.

This dual constraint raises a natural question: \emph{does there exist a terminal time $T$ ensuring convergence of both procedures?} The following theorem answers affirmatively.

\begin{theorem}
\label{theorem:duality}
Suppose Assumptions~\ref{assumption:semi:log:concave} and~\ref{assumption:prior:subGaussian} hold, and assume $2\alpha V_{\SG}^{2}<1$. Let $\underline{t}:=\log(1+2V_{\SG}^{2}) / 2$ and $\bar{t}:=\log(1+\alpha^{-1}) / 2$. Then $\underline{t}<\bar{t}$, and for every terminal time $T\in(\underline{t},\bar{t})$:
\begin{enumerate}[label=(\roman*)]
\item the posterior denoising density $p_{t}(\cdot\given\vx,\vy)$ is log-concave for all $t\in(0,T]$; and
\item the terminal posterior density $q_{T}(\cdot\given\vy)$ satisfies a log-Sobolev inequality.
\end{enumerate}
\end{theorem}

The proof of Theorem~\ref{theorem:duality} is deferred to Appendix~\ref{appendix:proof:theorem:duality}. Under mild assumptions on the target posterior distribution, this theorem establishes the existence of a terminal time $T$ {that simultaneously guarantees convergence of both the posterior score estimation and sampling from the target posterior distribution}, as illustrated in Figure~\ref{figure:diffusion}.

\begin{figure}[htbp]
\center
\begin{tikzpicture}
\draw[densely dotted,thick] (0.0,4.0)--(1.0,4.0);
\draw[-,thick] (1.0,4.0)--(4.0,4.0);
\draw[-latex,ultra thick,purple] (4.0,4.0)--(11.0,4.0);
\draw[-,thick] (11.0,4.0)--(12.0,4.0);
\draw[densely dotted,thick] (12.0,4.0)--(13.0,4.0);

\draw[pen colour={violet!75!black},decorate,decoration={calligraphic brace,amplitude=6.0pt,mirror},thick] (2.0,2.5) -- (13.0,2.5);
\node at (7.5,2.0) {\textcolor{violet!75!black}{log-concave RGO regime}};
\draw[densely dotted,violet!75!black] (2.0,4.0)--(2.0,2.5);
\draw[densely dotted,violet!75!black] (13.0,4.0)--(13.0,2.5);

\draw[pen colour={teal!75!black},decorate,decoration={calligraphic brace,amplitude=6.0pt,mirror},thick] (0.0,3.3)--(6.0,3.3);
\node at (3.0,2.8) {\textcolor{teal!75!black}{log-Sobolev inequality regime}};
\draw[densely dotted,teal!75!black] (0.0,4.0)--(0.0,3.3);
\draw[densely dotted,teal!75!black] (6.0,4.0)--(6.0,3.3);

\draw[pen colour={purple},decorate,decoration={calligraphic brace,amplitude=6.0pt},very thick] (4.0,4.7)--(11.0,4.7);
\node at (7.5,5.2) {\textcolor{purple}{diffusion-based posterior sampling}};

\filldraw[black] (0.0,4.0) circle (1.5pt); \node at (0.0,4.3) {$\gamma_{d}$};
\filldraw[black] (2.0,4.0) circle (1.5pt); \node at (2.0,4.3) {$q_{\bar{t}}(\cdot\given\vy)$}; 
\filldraw[black] (4.0,4.0) circle (1.5pt); \node at (4.0,4.3) {\textcolor{purple}{$q_{T}(\cdot\given\vy)$}}; \node at (4.0,3.7) {\textcolor{purple}{warm-start}}; 
\filldraw[black] (6.0,4.0) circle (1.5pt); \node at (6.0,4.3) {$q_{\underline{t}}(\cdot\given\vy)$};
\filldraw[black] (11.0,4.0) circle (1.5pt); \node at (11.0,4.3) {\textcolor{purple}{$q_{T_{0}}(\cdot\given\vy)$}}; \node at (11.0,3.7) {\textcolor{purple}{early-stopping}}; 
\filldraw[black] (13.0,4.0) circle (1.5pt); \node at (13.0,4.3) {\textcolor{blue!75!black}{$q_{0}(\cdot\given\vy)$}}; \node at (13.0,5.1) {\textcolor{blue!75!black}{semi-log-concave}}; \node at (13.0,4.8) {\textcolor{blue!75!black}{sub-Gaussian tails}}; 
\end{tikzpicture}
\vspace*{-10.0pt}
\caption{\footnotesize Schematic of the duality of convergence. The log-Sobolev inequality regime ($t>\underline{t}$, Proposition~\ref{proposition:log:sobolev}) and the log-concave RGO regime ($t<\bar{t}$, Proposition~\ref{proposition:appendix:RGO:Hessian}) overlap under Theorem~\ref{theorem:duality}, enabling a terminal time $T$ valid for both the warm-start and score estimation.}
\label{figure:diffusion}
\end{figure}

While prior studies have applied Monte Carlo score estimation to diffusion-based unconditional sampling~\citep{grenioux2024stochastic,He2024Zeroth,huang2024reverse,Huang2024Faster}, to our knowledge, this work is the first to extend this approach to posterior sampling. Furthermore, from a theoretical perspective, our method differs fundamentally from the existing results. Specifically, the duality of log-concavity in stochastic localization has been established by~\citet[Theorem~3]{grenioux2024stochastic} under a Gaussian convolution assumption~\citep[Assumption~A0]{grenioux2024stochastic}. While this assumption permits multimodal targets, it is not directly verifiable for posterior distributions: it is unclear what conditions on the prior and likelihood guarantee that the posterior satisfies the Gaussian convolution property. In contrast, Theorem~\ref{theorem:duality} is established under semi-log-concavity (Assumption~\ref{assumption:semi:log:concave}) and sub-Gaussian tails (Assumption~\ref{assumption:prior:subGaussian}), conditions that are more readily verifiable in terms of the prior and likelihood.

\section{Convergence Guarantee}
\label{section:convergence}
In this section, we establish a rigorous convergence analysis for the diffusion-based posterior sampling method~\eqref{eq:reversal:score}. Recall that our ultimate objective is to simulate the time-reversal process~\eqref{eq:reversal}:
\begin{equation*}
\d\bar{\mX}_{t}^{\vy}=\big\{\bar{\mX}_{t}^{\vy}+2\nabla_{\vx}\log q_{T-t}(\bar{\mX}_{t}^{\vy}\given\vy)\big\}\dt+\sqrt{2}\d\mB_{t}\,, \quad \bar{\mX}_{0}^{\vy}\sim q_{T}(\cdot\given\vy)\,, ~ t\in(0,T)\,.
\end{equation*}
In practical implementations, we introduce three modifications: early-stopping at time $T-T_0$, replacing the true posterior score with $\hat{\vs}_{m}^{S}(t,\cdot,\vy)$~\eqref{eq:posterior:score:estimate}, and using the warm-start surrogate $\hat{q}_{T}^{U}(\cdot\given\vy)$~\eqref{section:method:warm:Langevin:score}. This yields the approximate process~\eqref{eq:reversal:score}:
\begin{equation*}
\d\what{\mX}_{t}^{\vy}=\big\{\what{\mX}_{t}^{\vy}+2\hat{\vs}_{m}^{S}(T-t,\what{\mX}_{t}^{\vy},\vy)\big\}\dt+\sqrt{2}\d\mB_{t}\,, \quad \what{\mX}_{0}^{\vy}:=\what{\mX}_{T,U}^{\vy}\sim\hat{q}_{T}^{U}(\cdot\given\vy)\,, ~ t\in(0,T-T_{0})\,.
\end{equation*}
Let $\hat{q}_{T-T_{0}}(\cdot\given\vy)$ denote the conditional marginal density of $\what{\mX}_{T-T_0}^{\vy}$ given $\calG$. We apply two post-processing operators in Section~\ref{section:method:posterior:sampling}, truncation $\calT_{R}$ and rescaling $\calM(\mu_{T_{0}}^{-1})$, to $\what{\mX}_{T-T_{0}}^{\vy}$. The resulting distribution $\calM(\mu_{T_{0}}^{-1})\sharp\hat{q}_{T-T_{0}}^{R}(\cdot\given\vy)$ is the final approximation to $q_{0}(\cdot\given\vy)$ that we analyze in this section. Since the hatted laws are random probability measures, distances like $\text{\rm TV}$ and $\bbW_2$ below involving them are $\calG$-measurable random variables.

In addition to Assumptions~\ref{assumption:semi:log:concave} and~\ref{assumption:prior:subGaussian}, our analysis requires further assumptions. All these assumptions are standard and mild in the context of sampling and generative models, and detailed interpretations are provided in Section~\ref{appendix:assumptions}. We show several concrete examples in Section~\ref{appendix:examples} that satisfy these assumptions.

\begin{assumption}\label{assumption:posterior:bound:zero}
For a fixed measurement $\vy\in\bbR^{n}$, there exists a constant $H_{\vy}>1$ such that $\|\nabla_{\vx}\log q_{0}(\bzero\given\vy)\|_{2}\leq H_{\vy}$.
\end{assumption}

\begin{assumption}[Prior score matching error]\label{assumption:prior:score:error}
Let $\hat{\vs}_{\prior}$ be an estimator of the prior score function $\nabla\log\pi_{0}$. There exists $\varepsilon_{\mathrm{prior}}\in(0,1)$, such that
\begin{equation*}
\int\|\nabla\log\pi_{0}(\vx_{0})-\hat{\vs}_{\prior}(\vx_{0})\|_{2}^{2}\pi_{0}(\vx_{0})\d\vx_{0}\leq\varepsilon_{\mathrm{prior}}^{2}\,.
\end{equation*}
\end{assumption}

\begin{assumption}[Polynomial growth bounds of prior score]
\label{assumption:prior:bound}
There exist constants $r\geq 1$ and $B>1$ such that for each $\vx_{0}\in\bbR^{d}$, $$\max\{\|\nabla\log\pi_{0}(\vx_{0})\|_{2},\|\hat{\vs}_{\prior}(\vx_{0})\|_{2}\}\leq B(1+\|\vx_{0}\|_{2}^{r})\,.$$
\end{assumption}

\begin{assumption}
\label{assumption:Lipschitz:prior:likelihood}
For a fixed measurement $\vy\in\bbR^{n}$, there exists a constant $G>1$ such that $\hat{\vs}_{\prior}-\nabla\ell_{\vy}$ is $G$-Lipschitz.
\end{assumption}

\subsection{Error decomposition}
\label{section:convergence:decomposition}
The following proposition decomposes the Wasserstein-2 error of the estimated posterior density into the early-stopping error, the posterior score estimation error, and the warm-start error.

\begin{proposition}[Error decomposition]\label{proposition:error:decomposition}
Suppose Assumptions~\ref{assumption:semi:log:concave}--\ref{assumption:posterior:bound:zero} and~\ref{assumption:Lipschitz:prior:likelihood} hold.
\begin{enumerate}[label=(\roman*)]
\item Let $T_{0}\in(0,\min\{T,1/2\})$, and let $\delta\in(0,1)$ be small enough that $R\geq1$, where $R^{2}=(4\mu_{T_{0}}^{2}V_{\SG}^{2}+16\sigma_{T_{0}}^{2})\log(\kappa_{\vy}\delta^{-1})$. If $\bbE\{\|q_{T_{0}}(\cdot\given\vy)-\hat{q}_{T-T_{0}}(\cdot\given\vy)\|_{\tv}^{2}\}\leq\delta^{2}$, then 
\begin{equation*}
\bbE\big[\bbW_{2}^{2}\big\{q_{0}(\cdot\given\vy),\calM(\mu_{T_{0}}^{-1})\sharp\hat{q}_{T-T_{0}}^{R}(\cdot\given\vy)\big\}\big]\leq C\bigg[\frac{\sigma_{T_{0}}^{2}}{\mu_{T_{0}}^{2}}+\delta\bigg\{1+\log\bigg(\frac{\kappa_{\vy}}{\delta}\bigg)\bigg\}\bigg]\,.
\end{equation*}
Here $C$ is a constant only depending on $d$, $V_{\SG}$, and $C_{\SG}$.
\item For $0<T_{0}<T<\log(1+\alpha^{-1})/2$, it holds that
\begin{align*}
&\bbE\big\{\|q_{T_{0}}(\cdot\given\vy)-\hat{q}_{T-T_{0}}(\cdot\given\vy)\|_{\tv}^{2}\big\} \\
&~~~~~~\leq\underbrace{\int_{T_{0}}^{T}\bbE\big\{\|\nabla_{\vx}\log q_{t}(\mX_{t}^{\vy}\given\vy)-\hat{\vs}_{m}^{S}(t,\mX_{t}^{\vy},\vy)\|_{2}^{2}\big\}\dt}_{\text{posterior score estimation}}+\underbrace{2\bbE\big\{\|q_{T}(\cdot\given\vy)-\hat{q}_{T}^{U}(\cdot\given\vy)\|_{\tv}^{2}\big\}}_{\text{warm-start}}\,,
\end{align*}
where the expectation is taken with respect to $\mX_{t}^{\vy}\sim q_{t}(\cdot\given\vy)$ and over $\calG$.
\end{enumerate}
\end{proposition}

The proof of Proposition~\ref{proposition:error:decomposition} is deferred to Appendix~\ref{section:proof:error:decomposition}.
Proposition~\ref{proposition:error:decomposition} decomposes the posterior sampling error into three distinct components:
\begin{enumerate}[label=(\roman*)]
\item The \textbf{posterior score estimation error} quantifies the discrepancy between the true posterior score $\nabla_{\vx}\log q_{t}(\cdot\given\vy)$ and its Monte Carlo estimator $\hat{\vs}_{m}^{S}(t,\cdot,\vy)$ defined in~\eqref{eq:posterior:score:estimate}. Section~\ref{section:convergence:score} provides a detailed convergence analysis, demonstrating that this error increases as the early-stopping time $T_{0}$ approaches zero.
\item The \textbf{warm-start error} measures the difference between the true terminal posterior density $q_{T}(\cdot\given\vy)$ and its estimator $\hat{q}_{T}^{U}(\cdot\given\vy)$ obtained through the warm-start procedure~\eqref{section:method:warm:Langevin:score}. This error is analyzed in Section~\ref{section:convergence:warm:start}.
\item The \textbf{early-stopping error} $\calO(\sigma_{T_{0}}^{2})$ captures the error introduced by terminating the reverse process at time $T-T_{0}$ rather than $T$. This error decreases as $T_{0}$ approaches zero, creating a fundamental trade-off with the posterior score estimation error. The optimal selection of $T_{0}$ to balance these competing effects is addressed in Section~\ref{section:convergence:main}.
\end{enumerate}

\subsection{Error of posterior score estimation}
\label{section:convergence:score}
We now provide an error bound for the posterior score estimation~\eqref{eq:posterior:score:estimate}.

\begin{proposition}[Error bound of the posterior score estimation]\label{proposition:posterior:score}
Suppose Assumptions~\ref{assumption:semi:log:concave}, \ref{assumption:prior:subGaussian} and~\ref{assumption:prior:score:error}--\ref{assumption:Lipschitz:prior:likelihood} hold. Let $0<T_{0}<T<\log(1+\alpha^{-1}) / 2$. For each time $t\in(T_{0},T]$,
\begin{align*}
&\bbE\big\{\|\nabla_{\vx}\log q_{t}(\mX_{t}^{\vy}\given\vy)-\hat{\vs}_{m}^{S}(t,\mX_{t}^{\vy},\vy)\|_{2}^{2}\big\} \\
&~~~~~~~~~~~~\leq\underbrace{C\frac{\mu_{T_{0}}^{2}}{\sigma_{T_{0}}^{4}}\frac{\kappa_{\vy}}{m}}_{\text{Monte Carlo}}+\underbrace{C\frac{\mu_{T_{0}}^{2}}{\sigma_{T_{0}}^{4}}\exp\bigg\{-\frac{2(\mu_{T}^{2}-\alpha\sigma_{T}^{2})}{\sigma_{T}^{2}}S\bigg\}\eta_{\vy}^{2}}_{\text{convergence of Langevin dynamics}} \\
&~~~~~~~~~~~~\quad\quad+\underbrace{C\frac{\mu_{T_{0}}^{2}}{\sigma_{T_{0}}^{4}}S^{3/2}\bigg(\frac{\sigma_{T}^{2}\eta_{\vy}^{2}}{\mu_{T}^{2}-\alpha\sigma_{T}^{2}}+S\bigg)^{1/2}\exp\bigg\{2\bigg(G+\frac{\mu_{T_{0}}^{2}}{\sigma_{T_{0}}^{2}}\bigg)S\bigg\}\kappa_{\vy}^{1/2}\varepsilon_{\prior}^{1/2}}_{\text{prior score estimation error}}\,,
\end{align*}
where $S$ is the simulation horizon of~\eqref{eq:RGO:Langevin:score}, and $C$ is a constant only depending on $d$, $B$, $r$, $V_{\SG}$, and $C_{\SG}$. The expectation is taken with respect to $\mX_{t}^{\vy}\sim q_{t}(\cdot\given\vy)$ and over $\calG$. The condition number $\kappa_{\vy}$ is defined as~\eqref{eq:posterior:score:condition}, and the initial discrepancy $\eta_{\vy}$ is defined as
\begin{align*}
&\eta_{\vy}^{2}:= \sup_{t\in(T_{0},T]}\max\Big(\bbE\big[\chi^{2}\{\hat{p}_{t}^{0}(\cdot\given\mX_{t}^{\vy},\vy)\|p_{t}(\cdot\given\mX_{t}^{\vy},\vy)\}\big],\bbE\big[\bbW_{2}^{2}\{\hat{p}_{t}^{0}(\cdot\given\mX_{t}^{\vy},\vy),p_{t}(\cdot\given\mX_{t}^{\vy},\vy)\}\big]\Big)\,.
\end{align*}
\end{proposition}

The proof of Proposition~\ref{proposition:posterior:score} is deferred to Appendix~\ref{section:proof:posterior:score}. This result demonstrates that the posterior score estimation error decomposes into three fundamental components: the Monte Carlo approximation error~\eqref{eq:denoiser:MC}, the Langevin dynamics convergence error~\eqref{eq:RGO:Langevin}, and the prior score estimation error (Assumption~\ref{assumption:prior:score:error}). Specifically, the Monte Carlo error decreases at a rate of $\calO(m^{-1})$, where $m$ denotes the number of particles. The Langevin dynamics error decreases exponentially with respect to the simulation horizon $S$. Conversely, the prior score error component grows with $S$. This creates an inherent trade-off in selecting $S$: a larger simulation horizon improves convergence but amplifies the impact of the prior score error. Consequently, Proposition~\ref{proposition:posterior:score} provides actionable insights for hyperparameter selection in posterior score estimation, and these theoretical guidelines are formalized in the following result.

\begin{corollary}
\label{corollary:posterior:score}
Under the same conditions as Proposition~\ref{proposition:posterior:score}, for each $\varepsilon\in(0,1)$  satisfying $T\eta_{\vy}^{2}\varepsilon^{-2}>1$, the following inequality holds:
\begin{equation*}
\int_{T_{0}}^{T}\bbE\big\{\|\nabla_{\vx}\log q_{t}(\mX_{t}^{\vy}\given\vy)-\hat{\vs}_{m}^{S}(t,\mX_{t}^{\vy},\vy)\|_{2}^{2}\big\}\dt\leq C\frac{\mu_{T_{0}}^{2}}{\sigma_{T_{0}}^{4}}\varepsilon^{2}\,,
\end{equation*}
provided that the terminal time $S$ of Langevin dynamics~\eqref{eq:RGO:Langevin:score}, the number of Monte Carlo particles $m$, and the prior score matching error $\varepsilon_{\prior}$ satisfy
\begin{equation*}
\begin{aligned}
S\geq\frac{\sigma_{T}^{2}\big\{1+\log(T\eta_{\vy}^{2}\varepsilon^{-2})\big\}}{2(\mu_{T}^{2}-\alpha\sigma_{T}^{2})}\,, \,\,\,\,
m\geq\frac{T\kappa_{\vy}}{\varepsilon^{2}}\,, \,\,\,\, 
\varepsilon_{\prior}
\leq\frac{\varepsilon^{4}\exp\{-4(G+\mu_{T_{0}}^{2}\sigma_{T_{0}}^{-2})S\}}{T^{2}\kappa_{\vy}S^{3}\{\sigma_{T}^{2}\eta_{\vy}^{2}(\mu_{T}^{2}-\alpha\sigma_{T}^{2})^{-1}+S\}}\,.
\end{aligned}
\end{equation*}
where $C$ has the same dependence as in Proposition~\ref{proposition:posterior:score}.
\end{corollary}

The proof of Corollary~\ref{corollary:posterior:score} is in Appendix~\ref{section:proof:posterior:score}. Both Proposition~\ref{proposition:posterior:score} and  Corollary~\ref{corollary:posterior:score} demonstrate that the condition number $\kappa_{\vy}$ acts as an amplification factor, directly controlling how prior score estimation errors propagate to the posterior score estimation. This provides theoretical justification for the empirical observation that posterior sampling becomes increasingly challenging as the prior-likelihood mismatch grows, as interpretations in Remark~\ref{remark:condition:number}. Prior research has examined error analysis for Monte Carlo-based score estimation; see, for example,~\citet[Proposition 3.1]{He2024Zeroth}. A detailed comparison is provided in Appendix~\ref{appendix:discussions:mc:score}.

\subsection{Error of warm-start}
\label{section:convergence:warm:start}
This subsection establishes an error bound for the warm-start procedure~\eqref{section:method:warm:Langevin:score}.

\begin{proposition}[Error bound of the warm-start]
\label{proposition:posterior:warm:start}
Suppose Assumptions~\ref{assumption:semi:log:concave}--\ref{assumption:Lipschitz:prior:likelihood} hold. Assume $2\alpha V_{\SG}^{2}<1$ and $\kappa_{\vy}^{2}C_{\SG}^{2}\geq e$. Let the terminal time satisfy $T\in(\log(1+2V_{\SG}^{2}) / 2,\log(1+\alpha^{-1}) / 2)$. Then for $\varepsilon\in(0,1)$ satisfying $\zeta_{\vy}^{2}\varepsilon^{-1}>1$, it holds that
\begin{equation*}
\bbE\big\{\|q_{T}(\cdot\given\vy)-\hat{q}_{T}^{U}(\cdot\given\vy)\|_{\tv}^{2}\big\}\leq C\varepsilon\,,
\end{equation*}
provided that
\begin{align*}
U
&\geq 6\sigma_{T}^{2}
\exp\bigg\{2\times\frac{\sigma_{T}^{2}+2\mu_{T}^{2}V_{\SG}^{2}}
{\sigma_{T}^{2}-2\mu_{T}^{2}V_{\SG}^{2}}
\log(\kappa_{\vy}^{2}C_{\SG}^{2})\bigg\}
\log(\zeta_{\vy}^{2}\varepsilon^{-1})\,, \\
\varepsilon_{\post}
&\leq \frac{(\mu_{T}^{2}-\alpha\sigma_{T}^{2})^{3}\varepsilon^{2}}{U\kappa_{\vy}}\bigg[
6\times\sigma_{T}^{2}
\exp\bigg\{2\times\frac{\sigma_{T}^{2}+2\mu_{T}^{2}V_{\SG}^{2}}
{\sigma_{T}^{2}-2\mu_{T}^{2}V_{\SG}^{2}}
\log(\kappa_{\vy}^{2}C_{\SG}^{2})\bigg\}\zeta_{\vy}^{2}
+U\bigg]^{-1}
\,.
\end{align*}
where $C$ is a constant only depending on $d$, $B$, $H_{\vy}$, $V_{\SG}$, and $C_{\SG}$. The initial divergence $\zeta_{\vy}^{2}$ and the posterior score matching error at the terminal time $T$ are defined, respectively, as:
\begin{equation*}
\zeta_{\vy}^{2}:= \chi^{2}\big\{\hat{q}_{T}^{0}(\cdot\given\vy)\,\|\,q_{T}(\cdot\given\vy)\big\}\,, \quad
\varepsilon_{\post}^{2}:= \bbE\big\{\|\nabla_{\mathbf{x}}\log q_{T}(\mX_{T}^{\vy}|\vy)-\hat{\vs}_{m}^{S}(T,\mX_{T}^{\vy},\vy)\|_{2}^{2}\big\}\,.
\end{equation*}
The warm-start dynamics~\eqref{section:method:warm:Langevin:score} use the clipped score field~\eqref{eq:warm:score:clipped}, while $\varepsilon_{\post}$ is the estimation error of the unclipped estimator. The expectation in the warm-start error averages over $\mathcal{G}$; The expectation defining $\varepsilon_{\post}$ additionally averages over $\mX_{T}^{\vy}\sim q_{T}(\cdot\given\vy)$.
\end{proposition}

The proof of Proposition~\ref{proposition:posterior:warm:start} is provided in Appendix~\ref{section:proof:warm:start}. The warm-start procedure of Section~\ref{section:method:warm} employs a nested structure consisting of two Langevin dynamics components. The inner Langevin dynamics~\eqref{eq:RGO:Langevin:score} aim to estimate the posterior score function $\nabla\log q_{T}(\cdot\given\vy)$, while the outer Langevin dynamics~\eqref{section:method:warm:Langevin:score} generate particles from the terminal posterior density $q_{T}(\cdot\given\vy)$. Proposition~\ref{proposition:posterior:warm:start} specifies the terminal time $U$ and the posterior score estimation error budget $\varepsilon_{\post}$ for the outer Langevin dynamics~\eqref{section:method:warm:Langevin:score}. In Theorem~\ref{theorem:convergence}, the common score-estimation parameters are chosen so that this budget and Corollary~\ref{corollary:posterior:score} hold simultaneously.

\subsection{Non-asymptotic convergence rate}
\label{section:convergence:main}
Building upon the error decomposition (Proposition~\ref{proposition:error:decomposition}), Corollary~\ref{corollary:posterior:score}, and Proposition~\ref{proposition:posterior:warm:start}, we establish a non-asymptotic convergence bound for the proposed algorithm~\eqref{eq:reversal:score}.

\begin{theorem}[Convergence of posterior sampling]
\label{theorem:convergence}
Suppose Assumptions~\ref{assumption:semi:log:concave}--\ref{assumption:Lipschitz:prior:likelihood} hold, $2\alpha V_{\SG}^{2}<1$, and $\kappa_{\vy}^{2}C_{\SG}^{2}\geq e$. Let $T\in(\log(1+2V_{\SG}^{2}) / 2,\log(1+\alpha^{-1}) / 2)$ and $T_{0}\in(0,\min\{T,1/2\})$. For every sufficiently small $\varepsilon\in(0,1)$, suppose that the hyperparameters are chosen so that Corollary~\ref{corollary:posterior:score} and Proposition~\ref{proposition:posterior:warm:start} apply simultaneously, and choose $R$ as in Proposition~\ref{proposition:error:decomposition}(i) for the resulting total variation tolerance. Then
\begin{align*}
&\bbE\big[\bbW_{2}^{2}\big\{q_{0}(\cdot\given\vy),\calM(\mu_{T_{0}}^{-1})\sharp\hat{q}_{T-T_{0}}^{R}(\cdot\given\vy)\big\}\big] \\
&~~~~\leq C\bigg[\frac{\sigma_{T_{0}}^{2}}{\mu_{T_{0}}^{2}}+\bigg(\frac{\mu_{T_{0}}}{\sigma_{T_{0}}^{2}}\varepsilon+\varepsilon^{1/2}\bigg)
\log\bigg\{\kappa_{\vy}\bigg(\frac{\mu_{T_{0}}}{\sigma_{T_{0}}^{2}}\varepsilon+\varepsilon^{1/2}\bigg)^{-1}\bigg\}\bigg] \,.
\end{align*}
In particular, setting $T_{0}=\sqrt{\varepsilon}$ gives
\begin{equation*}
\bbE\big[\bbW_{2}^{2}\big\{q_{0}(\cdot\given\vy),\calM(\mu_{T_{0}}^{-1})\sharp\hat{q}_{T-T_{0}}^{R}(\cdot\given\vy)\big\}\big]\leq C^{\prime}\varepsilon^{1/2}\log(\varepsilon^{-1})\,.
\end{equation*}
The constants $C$ and $C^{\prime}$ depend only on $d$, $\kappa_{\vy}$, $B$, $r$, $H_{\vy}$, $V_{\SG}$, and $C_{\SG}$.
\end{theorem}

The proof of Theorem~\ref{theorem:convergence} is provided in Appendix~\ref{appendix:proof:theorem:convergence}. The bound comprises two competing terms in $T_{0}$: the early-stopping error $\calO(\sigma_{T_{0}}^{2})\to 0$ as $T_{0}\to 0$, and the score estimation error magnified by $\calO(\sigma_{T_{0}}^{-2})\to\infty$ as $T_{0}\to 0$. Setting $T_{0}=\sqrt{\varepsilon}$ optimally balances this trade-off. This justifies early-stopping as a theoretical necessity for our analysis rather than a computational convenience: terminating the time-reversal at $T-T_{0}$ is required to keep the bound on the estimation error finite. Theorem~\ref{theorem:convergence} also provides insights into related work, such as the closely connected method of~\citet{grenioux2024stochastic}. Specifically, taking the prior score to be exact, in which case every positive budget for $\varepsilon_{\prior}$ is met, recovers a non-asymptotic convergence analysis for the approach of~\citet{grenioux2024stochastic}. 

\section{Numerical Experiments}
\label{section:experiments}
We evaluate the proposed PDPS method on imaging inverse problems. Complete experimental details are deferred to Appendix~\ref{app:experimental_details}. The code is available at~\url{https://github.com/Ruoxuan0077/PDPS}.

\subsection{Experimental setup}
\label{sec:problem_formulations}
We consider three imaging inverse problems on $64 \times 64$ images corrupted by additive Gaussian noise $\vn \sim N(\bm{0}, \sigma^2 \mathbf{I}_n)$ with $\sigma=0.05$.

\textbf{Linear forward operator.}
We first consider the general linear measurement model $\mY = \mA \mX_0 + \vn$ with $\mA \in \mathbb{R}^{n \times d}$, and evaluate our method on two concrete instantiations of $\mA$: \emph{Gaussian deblurring}, where $\mA$ corresponds to convolution with a standard Gaussian kernel; and \emph{motion deblurring}, where $\mA$ corresponds to convolution with a kernel synthesized using publicly available code\footnote{\url{https://github.com/LeviBorodenko/motionblur}} to simulate realistic camera shake effects.

\textbf{Nonlinear forward operator.}
We further consider a nonlinear deblurring task inspired by the GOPRO dataset~\citep{nah2017deep}, in which the forward operator combines temporal aggregation of multiple sharp frames with a nonlinear camera response function. In practice, we adopt the pretrained distillation model of~\cite{tran2021explore} as a deterministic nonlinear forward operator. The full mathematical specification is deferred to Appendix~\ref{app:measurement_details}.

\textbf{Dataset and baselines.}
\label{sec:dataset:baseline}
We use the FFHQ dataset~\citep{karras2019style} at $64 \times 64$ resolution and benchmark against Diffusion Posterior Sampling (DPS)~\citep{chung2023diffusion} and Total Variation (TV) regularization. Both PDPS and DPS employ the same pre-trained prior score network based on the EDM framework~\citep{karras2022elucidating}.

\subsection{Reconstruction results}
\label{sec:sampling_efficacy}
Table~\ref{tab:ffhq} reports average PSNR and SSIM over 128 FFHQ64 images. PDPS consistently outperforms both baselines across all tasks, with particularly pronounced gains in the nonlinear deblurring setting.

\begin{table}[H]
\vspace{-0.8mm}
\small
\centering
\caption{Comparison of PSNR and SSIM on Different Methods}
\begin{tabular}{l ccc ccc ccc} 
\toprule
\multirow{2}{*}{} & \multicolumn{3}{c}{Gaussian Deblur} 
& \multicolumn{3}{c}{Motion Deblur} 
& \multicolumn{3}{c}{Nonlinear Deblur} \\
\cmidrule(lr){2-4} \cmidrule(lr){5-7} \cmidrule(lr){8-10}
& TV & DPS & \textbf{PDPS} & TV & DPS & \textbf{PDPS} & TV & DPS & \textbf{PDPS} \\
\midrule
PSNR & 23.98 & 24.38 & \textbf{26.42} & 24.94 & 26.83 & \textbf{28.86} & 18.66 & 20.96 & \textbf{28.44} \\
SSIM & 0.77 & 0.82 & \textbf{0.87} & 0.81 & 0.88 & \textbf{0.92} & 0.48 & 0.69 & \textbf{0.91} \\
\bottomrule
\end{tabular}\label{tab:ffhq}
\end{table}

For in-depth case studies, we selected six representative images (two per task) to evaluate the reconstruction quality of the different methods. For each image, we performed 24 independent sampling runs for the stochastic methods, PDPS and DPS, whereas TV produces a single deterministic reconstruction. The primary visual result shown for each stochastic method corresponds to the single run achieving the highest combined PSNR and SSIM; for TV, the unique reconstruction is displayed. To enable qualitative comparison and uncertainty quantification, we further report (i) the pixel-wise mean absolute error (MAE) with respect to the ground truth (labeled ``mean err.'') and (ii) the pixel-wise standard deviation (labeled ``std. dev.'') computed across the 24 runs. For TV, the mean-error map reduces to the absolute-error map of its single reconstruction, and its standard-deviation map is identically zero.

\begin{figure}[!ht]
\centering
\begin{subfigure}[t]{0.50\textwidth}
\centering\includegraphics[width=\linewidth]{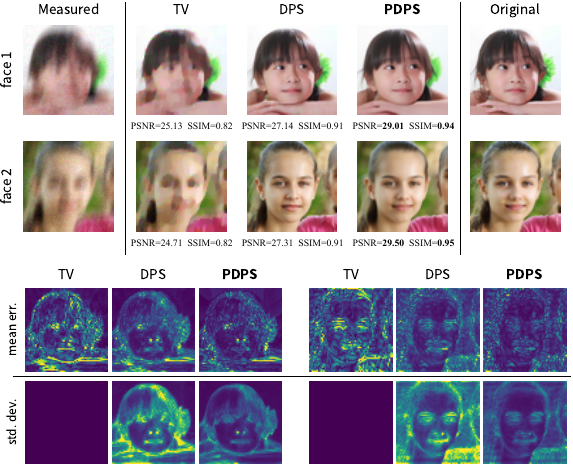}
\caption{Motion deblurring.}
\label{fig:motion_results}
\end{subfigure}\hfill
\begin{subfigure}[t]{0.50\textwidth}
\centering\includegraphics[width=\linewidth]{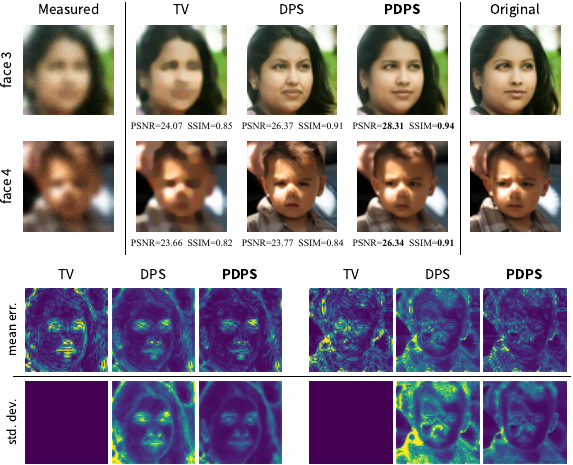}
\caption{Gaussian deblurring.}
\label{fig:gaussian_results}
\end{subfigure}
\hfill
\begin{subfigure}[t]{0.50\textwidth}
\centering\includegraphics[width=\linewidth]{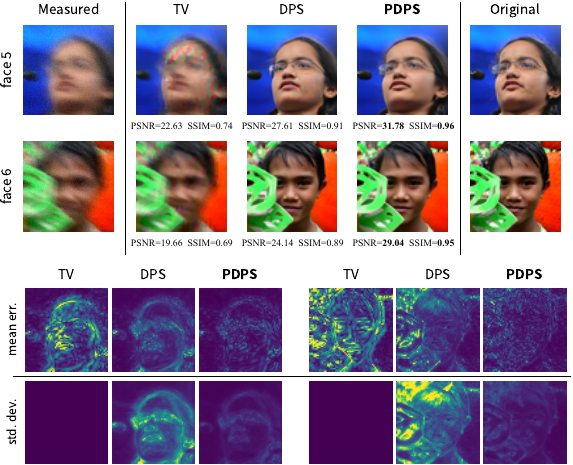}
\caption{Nonlinear deblurring.}
\label{fig:nonlinear_results}
\end{subfigure}
\caption{\footnotesize Reconstruction and uncertainty quantification on FFHQ64. Each panel compares the degraded input, TV, DPS, and PDPS against the ground truth. Bottom rows show the pixel-wise mean absolute error and standard deviation over 24 independent runs for PDPS and DPS; for deterministic TV, the mean-error map is the absolute-error map of its single reconstruction and the standard-deviation map is identically zero.}
\label{fig:ffhq_results}
\end{figure}

\textbf{Motion deblurring.}
As shown in Figure~\ref{fig:motion_results}, TV fails to recover fine details such as hair strands and textures. DPS recovers considerably more detail, but the subject's left hand appears incomplete in ``face 1,'' while the mouth appears somewhat unnatural in ``face 2.'' In contrast, PDPS effectively removes motion blur while faithfully restoring intricate details and sharp edges.

The pixel-wise standard deviation maps enable a visual assessment of uncertainty quantification (UQ) for the stochastic methods. Ideally, higher standard deviation should correlate with higher reconstruction error, reflecting the model's confidence. Since TV is a deterministic optimization baseline, its identically zero standard-deviation map is not interpreted as an uncertainty estimate. For both stochastic methods, regions of elevated standard deviation overlap qualitatively with some high-error regions. In the displayed examples, PDPS has lower mean-error intensity and smoother standard-deviation maps than DPS.

\textbf{Gaussian deblurring.}
For the Gaussian deblurring task (Figure~\ref{fig:gaussian_results}), PDPS consistently yields the most natural reconstructions. For both ``face 3'' and ``face 4,'' the facial appearance in the DPS reconstruction is somewhat unnatural, whereas PDPS produces a more natural face that is closer to the ground truth. The UQ maps show a similar qualitative pattern: PDPS has lower mean-error intensity and smoother standard-deviation maps, with elevated variability concentrated around some facial and edge regions that also exhibit larger error.

\textbf{Nonlinear deblurring.}
The nonlinear results (Figure~\ref{fig:nonlinear_results}) most starkly demonstrate PDPS's advantage, with a substantial quantitative lead: PDPS achieves PSNR $\approx$ 29--32 compared to DPS ($\approx$ 24--28) and TV ($\approx$ 19--22). Visually, TV fails to produce clear reconstructions, and DPS introduces artifacts with erroneous details---affecting facial clarity in ``face 5'' and texture fidelity in ``face 6.'' In contrast, PDPS delivers clear, detailed reconstructions with fewer visible artifacts and closer agreement with the ground truth. The UQ maps provide a consistent qualitative comparison: PDPS shows lower mean-error intensity and smoother sample variability than DPS in both displayed nonlinear examples, while elevated variability remains concentrated around several high-error structures.

\subsection{Ablation: terminal time}
\label{sec:ablation}
We investigate the impact of the terminal time $T$ on sampling performance, validating the theoretical trade-off from Theorem~\ref{theorem:duality}: $T$ must be large enough ($T>\underline{t}$) for the log-Sobolev inequality of $q_{T}(\cdot\given\vy)$, yet small enough ($T<\bar{t}$) for log-concavity of the denoising density $p_{t}(\cdot\given\vx,\vy)$. Figure~\ref{fig:ablation} shows the ablation on the Gaussian deblurring task for two representative images.

The plots reveal three distinct regimes. For very small $T$ (e.g., $T<0.05$), both PSNR and SSIM are poor, consistent with the theoretical prediction that the terminal posterior $q_{T}(\cdot\given\vy)$ either possesses a large log-Sobolev constant or fails to satisfy the inequality entirely, making the warm-start sampler ineffective. As $T$ increases, performance rapidly improves and enters a stable plateau of near-optimal results for $T\approx 0.05$--$1.0$, representing the practical ``sweet spot'' where both theoretical conditions from Theorem~\ref{theorem:duality} are simultaneously met. However, for $T>1.0$, performance begins to degrade as the posterior denoising density $p_{t}(\cdot\given\vx,\vy)$ loses log-concavity, compromising the accuracy of the Monte Carlo score estimator. These findings provide empirical validation of our theoretical analysis and underscore the importance of selecting a well-balanced terminal time.

\begin{figure}[ht]
\centering
\includegraphics[width=0.9\textwidth]{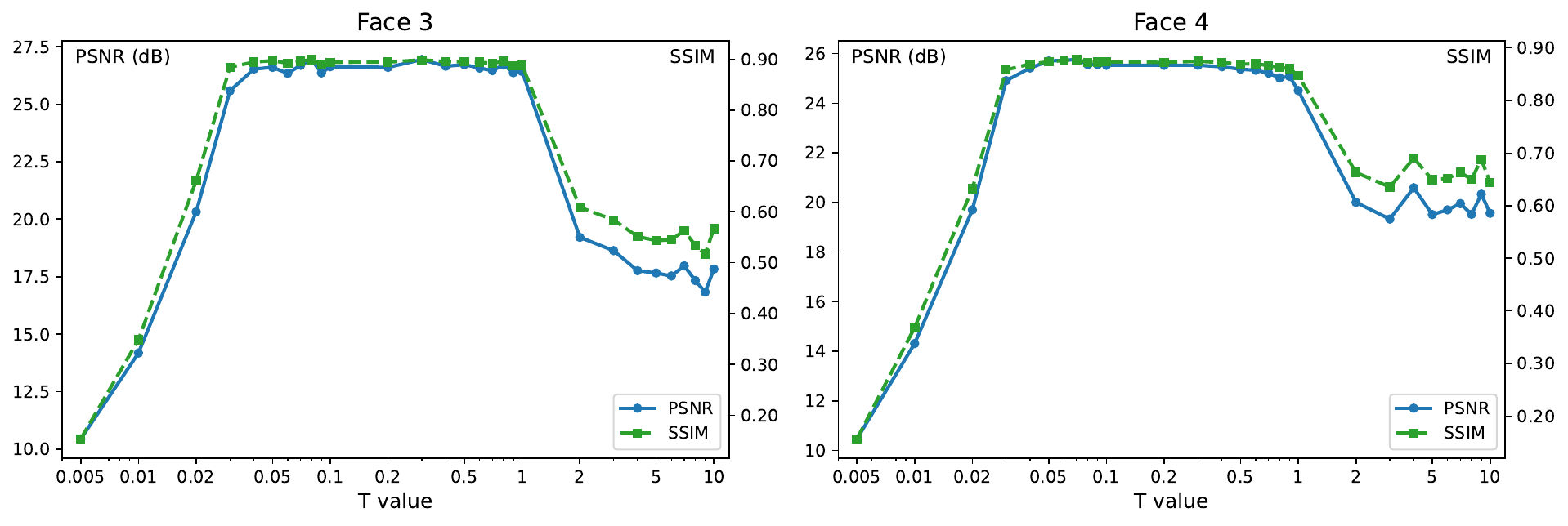}
\caption{\footnotesize Ablation on terminal time $T$ for Gaussian deblurring (face 3 and 4). All runs in this ablation use $T_0=0.001$. Performance peaks for $T\in[0.05, 1.0]$, validating the theoretical trade-off between warm-start convergence and score estimation accuracy.}
\label{fig:ablation}
\end{figure}
\section{Discussion}
\label{section:conclusion}
We have introduced a diffusion-based posterior sampling method (PDPS) in which the posterior score is estimated via a Monte Carlo procedure using Langevin dynamics driven by a pre-trained prior score. A warm-start strategy enables sampling from the terminal posterior distribution at a small diffusion terminal time. We established non-asymptotic convergence guarantees in the 2-Wasserstein distance, providing rigorous error bounds even when the target posterior is multimodal. The analysis reveals the influence of key quantities such as the condition number and the prior score estimation accuracy, and offers practical guidance for hyperparameter selection. Extensive experiments on imaging inverse problems demonstrate competitive performance and effective uncertainty quantification.

Several directions remain for future work. First, the present framework assumes a known forward operator; extending PDPS to blind inverse problems, where the forward operator is unknown or only partially specified, is a natural and practically motivated next step. Second, the same Monte Carlo posterior score machinery can be adapted to data assimilation, in which observations arrive sequentially and the posterior must be updated under a state-space dynamics rather than a fixed measurement model. Third, our analysis is built on an SDE-based time-reversal process; transferring the convergence guarantees and the duality theorem to ODE-based samplers, such as probability-flow ODEs and flow matching, would deliver a deterministic counterpart of PDPS with potentially sharper computational efficiency. Fourth, the alignment of generative models, encompassing inference-time reward guidance, can be cast as posterior sampling in which a reward or preference signal plays the role of the likelihood. PDPS provides a candidate framework for alignment with rigorous guarantees.

\appendix

\section*{Outline of Appendices}
The appendices are organized as follows:

\begin{enumerate}[label=(\roman*)]
\item \textbf{Appendix~\ref{appendix:notation}}: Summary of notation.

\item \textbf{Appendix~\ref{section:related:work}}: Related work, including conditional diffusion models, heuristic posterior score estimation, and hybridization with classical sampling.

\item \textbf{Appendix~\ref{appendix:discussions}}: Technical discussions of the condition number, early stopping, and Monte Carlo-based score estimation.

\item \textbf{Appendix~\ref{section:concrete:examples}}: Discussions on assumptions and concrete examples.

\item \textbf{Appendix~\ref{appendix:proofs:methods}}: Proofs of all propositions and the duality theorem in Section~\ref{section:method}.

\item \textbf{Appendix~\ref{appendix:random:field}}: Construction of the Monte Carlo posterior score estimator as a measurable random field, together with regularity and growth bounds for both the estimated and true posterior scores.

\item \textbf{Appendix~\ref{section:proof:error:decomposition}}: Proof of the error decomposition (Proposition~\ref{proposition:error:decomposition}), including the localized entropy--Girsanov inequality used in the path-space comparison.

\item \textbf{Appendix~\ref{section:proof:posterior:score}}: Proof of the posterior score estimation bound (Proposition~\ref{proposition:posterior:score} and Corollary~\ref{corollary:posterior:score}).

\item \textbf{Appendix~\ref{section:proof:warm:start}}: Proof of the warm-start error bound (Proposition~\ref{proposition:posterior:warm:start}).

\item \textbf{Appendix~\ref{appendix:proof:theorem:convergence}}: Proof of the convergence theorem (Theorem~\ref{theorem:convergence}).

\item \textbf{Appendix~\ref{appendix:auxiliary:lemmas}}: Auxiliary lemmas.

\item \textbf{Appendix~\ref{app:experimental_details}}: Experimental details and additional results, including measurement operators, implementation, hyperparameters, and a cross-dataset robustness study under prior mismatch.
\end{enumerate}

\section{A Summary of Notations}
\label{appendix:notation}
Tables~\ref{table:notation:introduction},~\ref{table:notation:methods} and~\ref{table:notation:convergence} summarize the notation used in Sections~\ref{section:introduction},~\ref{section:method} and~\ref{section:convergence}, respectively, for easy reference and cross-checking.

\begin{table}[htbp]
\caption{The list of notation defined in Section~\ref{section:introduction}.}
\centering
\renewcommand{\arraystretch}{0.7}
\begin{tabular}{@{}p{0.31\textwidth}p{0.65\textwidth}@{}}
\toprule 
Symbols & Description \\ 
\midrule 
$\mX_{0}$ & A random variable representing the unknown signal. \\
$\pi_{0}$ & The prior density of the signal $\mX_{0}$. \\
$\mY$ & A random variable representing the measurement of the signal. \\
$p_{\mY\given\mX_{0}}(\cdot\given\vx_{0})$ & The likelihood of $\mY$ given $\mX_{0}=\vx_{0}$. \\
$\ell_{\vy}(\vx_{0})$ & The negative log-likelihood of $\mY=\vy$ given $\mX_{0}=\vx_{0}$. \\
$q_{0}(\cdot\given\vy)=p_{\mX_{0}\given\mY}(\cdot\given\vy)$ & The conditional density of $\mX_{0}$ given $\mY=\vy$. \\
\bottomrule
\end{tabular}
\label{table:notation:introduction}
\end{table}

\begin{table}[htbp]
\caption{The list of notation defined in Section~\ref{section:method}.}
\centering
\renewcommand{\arraystretch}{0.7}
\begin{tabular}{@{}p{0.31\textwidth}p{0.65\textwidth}@{}}
\toprule 
Symbols & Description \\ 
\midrule 
$\mX_{t}$ & The unconditional forward process~\eqref{eq:method:forward}. \\
$p_{\mX_{t}\given\mX_{0}}(\cdot\given\vx_{0})$ & The conditional density of $\mX_{t}$ given $\mX_{0}=\vx_{0}$. \\
$\pi_{t}$ & The marginal density of the forward process $\mX_{t}$~\eqref{eq:method:marginal}. \\
$T$ & The terminal time of the forward process. \\
$\mX_{t}^{\vy}$ & The posterior forward process given $\mY=\vy$~\eqref{eq:method:forward:y}. \\
$q_{t}(\cdot\given\vy)$ & The marginal density of the posterior forward process $\mX_{t}^{\vy}$~\eqref{eq:method:marginal:y}. \\
$p_{\mX_{t}\given\mY}(\cdot\given\vy)$ & The conditional density of $\mX_{t}$ given $\mY=\vy$. \\
$\bar{\mX}_{t}^{\vy}$ & The posterior time-reversal process~\eqref{eq:reversal}. \\
\midrule
$\mD(t,\vx,\vy)$ & The posterior denoiser~\eqref{eq:posterior:denoiser}, i.e., the conditional expectation of $\mX_{0}$ given $\mX_{t}=\vx$ and $\mY=\vy$. \\
$p_{t}(\cdot\given\vx,\vy)$ & The posterior denoising density~\eqref{eq:posterior:conditional:density}, i.e., the conditional density of $\mX_{0}$ given $\mX_{t}=\vx$ and $\mY=\vy$. \\
$\mX_{0,s}^{\vx,\vy,t}$ & The Langevin dynamics with invariant density $p_{t}(\cdot\given\vx,\vy)$~\eqref{eq:RGO:Langevin} \\
$p_{t}^{s}(\cdot\given\vx,\vy)$ & The marginal density of $\mX_{0,s}^{\vx,\vy,t}$. \\
$\what{\mX}_{0,s}^{\vx,\vy,t}$ & The Langevin dynamics with estimated prior score~\eqref{eq:RGO:Langevin:score} \\
$\hat{p}_{t}^{s}(\cdot\given\vx,\vy)$ & The marginal density of $\what{\mX}_{0,s}^{\vx,\vy,t}$. \\
$S$ & The simulation horizon of~\eqref{eq:RGO:Langevin:score}. \\
$m$ & The number of particles used in Monte Carlo approximation. \\
\makecell[c]{$\what{\mX}_{0,S,1}^{\vx,\vy,t},\ldots,$\\
$\what{\mX}_{0,S,m}^{\vx,\vy,t}$} & Particles generated by~\eqref{eq:RGO:Langevin:score}, following the density $\hat{p}_{t}^{S}(\cdot\given\vx,\vy)$. \\
$\what{\mD}_{m}^{S}(t,\vx,\vy)$ & The Monte Carlo estimator of the posterior denoiser~\eqref{eq:denoiser:MC}. \\
$\hat{\vs}_{m}^{S}(t,\vx,\vy)$ & The Monte Carlo estimator of the posterior score~\eqref{eq:posterior:score:estimate}. \\
$\calG$ & The $\sigma$-field generated by the common inner random seeds (Section~\ref{section:method:score}). \\
\midrule
$q_{T}(\cdot\given\vy)$ & The terminal posterior density, i.e., the marginal density of $\mX_{T}^{\vy}$. \\
$\mX_{T,u}^{\vy}$ & The Langevin dynamics with invariant density $q_{T}(\cdot\given\vy)$~\eqref{section:method:warm:Langevin}. \\
$q_{T}^{u}(\cdot\given\vy)$ & The marginal density of $\mX_{T,u}^{\vy}$. \\
\makecell[c]{$a_{T}(\vx,\vy)$,\\
$\tilde{\vs}_{m}^{S}(T,\vx,\vy)$} & The warm-start clipping radius~\eqref{eq:warm:score:radius} and clipped score field~\eqref{eq:warm:score:clipped}. \\
$\what{\mX}_{T,u}^{\vy}$ & The Langevin dynamics with clipped score $\tilde{\vs}_{m}^{S}(T,\cdot,\vy)$~\eqref{section:method:warm:Langevin:score}. \\
$\hat{q}_{T}^{u}(\cdot\given\vy)$ & The conditional law given $\calG$ of $\what{\mX}_{T,u}^{\vy}$. \\
$U$ & The simulation horizon of~\eqref{section:method:warm:Langevin:score}. \\
$C_{\SG}$ and $V_{\SG}$ & The parameters of sub-Gaussian tails defined in Assumption~\ref{assumption:prior:subGaussian}. \\
$\kappa_{\vy}$ & The condition number of the posterior sampling~\eqref{eq:posterior:score:condition}. \\
$C_{\LSI}(\mu)$ & The log-Sobolev inequality constant of a distribution $\mu$. \\ 
\midrule 
$\what{\mX}_{t}^{\vy}$ & The approximate posterior time-reversal process~\eqref{eq:reversal:score}. \\
$T_{0}$ & The early-stopping time defined in~\eqref{eq:reversal:score}. \\
\bottomrule
\end{tabular}
\label{table:notation:methods}
\end{table}

\begin{table}[htbp]
\caption{The list of notation defined in Section~\ref{section:convergence}.}
\vspace*{5pt}
\centering
\renewcommand{\arraystretch}{0.7}
\begin{tabular}{@{}p{0.31\textwidth}p{0.65\textwidth}@{}}
\toprule 
Symbols & Description \\ 
\midrule 
$\calM$ & The scaling operator, i.e., $\calM(\xi):\vx\mapsto\xi\vx$. \\
$\varepsilon_{\prior}$ & The prior score matching error defined in Assumption~\ref{assumption:prior:score:error}. \\
$B$ & Constant in the polynomial growth bound of the prior score (Assumption~\ref{assumption:prior:bound}). \\
$\eta_{\vy}$ & The initial discrepancy of~\eqref{eq:RGO:Langevin:score} defined in Proposition~\ref{proposition:posterior:score}. \\
$\varepsilon_{\post}$ & The posterior score estimation error at the terminal time, defined in Proposition~\ref{proposition:posterior:warm:start}. \\
$\zeta_{\vy}$ & The initial discrepancy of~\eqref{section:method:warm:Langevin:score} defined in Proposition~\ref{proposition:posterior:warm:start}.\\
\bottomrule
\end{tabular}
\label{table:notation:convergence}
\end{table}

\section{Related Work}
\label{section:related:work}

\subsection{Conditional diffusion models}
\label{section:related:conditional}
Posterior sampling sits within the broader context of conditional generative learning, and a rich body of prior work on conditional diffusion models provides relevant methodological background. Two paradigms can be distinguished: (i) paired-data approaches, which require joint samples $(\mX_{0},\mY)$ and learn a dedicated conditional score network; and (ii) zero-shot approaches, which require only prior samples and introduce the likelihood at inference time.

\subsubsection{Conditional diffusion models using paired-data}
Paired-data approaches learn the conditional score $\nabla\log p_{\mX_{t}\given\mY}(\cdot\given\vy)$ from samples of the joint distribution. Classifier guidance~\citep{app:dhariwal2021diffusion} trains an unconditional score estimator $\widehat{\vs}_{t}(\vx)\approx\nabla\log p_{\mX_{t}}(\vx)$ together with a separate time-dependent classifier for the likelihood $p_{\mY\given\mX_{t}}(\vy\given\vx)$, trained on data pairs $\{(\mX_{t}^{i},\mY^{i})\}_{i=1}^{n}$ obtained by applying the forward corruption to the $\mX_{0}$ component of $\{(\mX_{0}^{i},\mY^{i})\}_{i=1}^{n}$. Minimizing a cross-entropy loss yields $\widehat{p}_{\mY\given\mX_{t}}(\vy\given\vx)$, and the conditional score is assembled via Bayes' rule
\begin{equation*}
\hat{\vs}_{t}(\vx\given\vy)\coloneq \nabla\log\hat{p}_{\mY\given\mX_{t}}(\vy\given\vx)+\hat{\vs}_{t}(\vx)\approx\nabla\log p_{\mX_{t}\given\mY}(\vx\given\vy)\,.
\end{equation*}
Classifier guidance is inherently restricted to discrete conditions such as class labels. Classifier-free guidance~\citep{app:ho2022classifier} removes the external classifier and accommodates continuous $\mY$ by directly learning the conditional score through conditional score matching. Theoretical analyses of classifier-free guidance appear in~\citet{app:chang2024deep,app:fu2024unveil,app:tang2024conditional}, with extensions to inverse problems in~\citet{app:chung2024cfg}. Despite their success in conditional generation, paired-data approaches are poorly matched to inverse problems for two reasons: inverse problems typically involve a single measurement, making the paired samples required to train a posterior score unavailable in practice, and the learned conditional score is tied to a specific likelihood, so any change in the measurement operator requires retraining from scratch.

\subsubsection{Zero-shot conditional diffusion models}
Zero-shot approaches instead train only on prior samples, with no paired data. The resulting prior score acts as a plug-and-play prior applicable across a broad class of inverse problems with arbitrary likelihoods~\citep{app:chung2023diffusion,app:song2023loss,app:Guo2024Gradient,app:Sun2024Provable,app:li2025efficient,app:martin2025pnpflow,app:Zhang2025Improving}; the specific score estimation heuristics employed within this class are reviewed in Appendix~\ref{section:method:previous:work}. Training is thereby decoupled from the likelihood, which enters only at inference time---a procedure referred to as inference-time alignment~\citep{app:kim2025testtime,app:uehara2025inference}. Theoretical guarantees for diffusion guidance in this zero-shot setting have been developed by~\citet{app:Chidambaram2024what,app:wu2024theoretical,app:jiao2025unified,app:moufad2025conditional,app:galashov2025learn}. The present work belongs to this class.

\subsection{Heuristic posterior score estimation}
\label{section:method:previous:work}
Section~\ref{section:preliminary} introduced the decomposition~\eqref{eq:posterior:score:decomposition} of the time-dependent posterior score and observed that existing heuristic estimators all target the time-dependent likelihood term by approximating the denoising density. Here we elaborate on the two heuristic families:
\begin{itemize}
\item Dirac delta approximation~\citep{app:chung2022improving,app:chung2023diffusion,rout2023solving,app:Yu2023FreeDoM,app:bansal2024universal}. This line of work proposes a Dirac delta approximation for the denoising density:
\begin{equation*}
p_{t}(\vx_{0}\given\vx)\approx \delta_{\bbE[\mX_{0}\given\mX_{t}=\vx]}(\vx_{0})\,,
\end{equation*}
where the conditional expectation $\bbE[\mX_{0}\given\mX_{t}=\vx]$ is a by-product of the time-dependent score, as shown by Tweedie's formula~\citep{app:Efron2011Tweedie}. This Dirac delta approximation implies the following approximation for the time-dependent likelihood:
\begin{equation*}
\bbE\big[\exp\{-\ell_{\vy}(\mX_{0})\}\given\mX_{t}=\vx\big]\approx \exp[-\ell_{\vy}\{\bbE(\mX_{0}\given\mX_{t}=\vx)\}]\,.
\end{equation*}
This approximation can be viewed as interchanging the conditional expectation with the nonlinear function $\exp(-\ell_{\vy}(\cdot))$, and thus suffers from a bias known as Jensen's gap~\citep{app:chung2023diffusion}.
\item Gaussian Approximation~\citep{app:song2023pseudoinverse,app:song2023loss,app:Zhang2025Improving}. Another line of work proposes a Gaussian approximation for the denoising density $p_{t}(\cdot\given\vx)$. According to Bayes' rule, the denoising density $p_{t}(\cdot\given\vx)$, i.e., the conditional density of $\mX_{0}$ given $\mX_{t}=\vx$, reads
\begin{equation}\label{eq:denoising:density}
p_{t}(\vx_{0}\given\vx)\propto\exp\bigg\{\log\pi_{0}(\vx_{0})-\frac{\|\vx-\mu_{t}\vx_{0}\|_{2}^{2}}{2\sigma_{t}^{2}}\bigg\}\,.
\end{equation}
Note that $p_{t}(\cdot\given\vx)$ is dominated by the quadratic term for small $t$, meaning it behaves like a Gaussian distribution in this regime. However, for large $t$, its behavior is determined primarily by the prior density $\pi_{0}$. If the prior $\pi_{0}$ is multimodal, the denoising density $p_{t}(\cdot\given\vx)$ will also tend to be multimodal. In this scenario, the Gaussian approximation can lead to a large bias.
\end{itemize}

\par\textbf{Estimating the Gradient of the Log-Likelihood} 
A second critical limitation is that these methods first approximate the time-dependent likelihood; however, the actual target is the gradient of its logarithm. Both of the aforementioned approaches use a direct plug-in gradient estimator. That is, they use the gradient of their likelihood estimator as an estimate of the true gradient.

\par The effectiveness of this plug-in approach is not generally guaranteed. An estimator that approximates a target function with low error in value can still have a gradient that is a poor approximation of the function's true gradient.

\par\textbf{Limitations}
Taking these two issues together: (i) the potential large bias in the likelihood approximation; and (ii) the unreliability of the plug-in gradient estimator, we conclude that these existing methods are not sufficiently robust and accurate.

\subsection{Hybridizing diffusion models with classical sampling methods}
To avoid heuristic approximations of posterior scores, recent work has integrated diffusion models with classical posterior sampling techniques. For instance, \citet{app:Coeurdoux2024Plug,app:Wu2024Principled,app:Xu2024Provably,app:chu2025split,app:zheng2025blade} couple diffusion models with split Gibbs samplers, and \citet{app:Wu2024Principled,app:Xu2024Provably,app:chu2025split,app:zheng2025blade} further provide theoretical guarantees for this approach. Likewise, \citet{app:wu2023practical,app:cardoso2024monte,app:achituve2025inverse,app:chen2025solving,app:kelvinius2025solving,app:kim2025testtime,app:lee2025debiasing,app:skreta2025feynmankac} integrate diffusion models with sequential Monte Carlo methods, with \citet{app:wu2023practical,app:cardoso2024monte,app:chen2025solving} offering rigorous theoretical analysis. Additionally,~\citet{app:Bruna2024Provable} develops the tilted transport technique with theoretical guarantees. However, this method depends on the quadratic structure of the linear-Gaussian measurement model, which restricts its extension to general measurement models.

Our work differs fundamentally from these approaches in that they rely on a probability path connecting the initial distribution to the target posterior that is distinct from the diffusion model's intrinsic time-reversal process.

\section{Technical Discussions}
\label{appendix:discussions}
This section gathers two technical discussions that complement the main analysis: an interpretation of the condition number $\kappa_{\vy}$, and a justification for the early-stopping rule, including a comparison with analogous trade-offs in the unconditional diffusion literature.

\subsection{Interpretation of the condition number}
\label{appendix:discussions:condition:number}
The condition number $\kappa_{\vy}$ defined in~\eqref{eq:posterior:score:condition} serves as a fundamental measure of problem difficulty in our analysis of posterior sampling. As demonstrated by~\citet[Theorem 4.1]{app:purohit2024posterior}, this quantity controls how errors in the prior model are amplified in the posterior. A moderate condition number occurs when the likelihood function $\exp(-\ell_{\vy}(\cdot))$ is concentrated in regions where the prior density $\pi_{0}$ is also high. In this scenario, the denominator of~\eqref{eq:posterior:score:condition} stays bounded away from zero, indicating a well-posed inverse problem where the observed data is consistent with prior beliefs. Conversely, a large condition number typically indicates an ill-posed problem. This happens when the likelihood function is concentrated in regions with low prior probability, causing the denominator of~\eqref{eq:posterior:score:condition} to approach zero. This situation typically arises when the observed data strongly conflicts with the prior information.

\subsection{Early-stopping in the time-reversal process}
\label{appendix:discussions:early:stopping}
The early-stopping time $T_{0}\in(0,T)$ in the approximate time-reversal process~\eqref{eq:reversal:score} is introduced to handle the singularity of the posterior score estimator $\hat{\vs}_{m}^{S}(t,\cdot,\vy)$ near $t=0$. As established in Section~\ref{section:convergence:score}, the estimation error bound diverges as $t$ approaches zero, necessitating termination at $t=T_{0}$ for numerical stability. The optimal choice of $T_{0}$ is addressed in Section~\ref{section:convergence:main}. The numerical instability near $t=0$ has been observed and investigated in the unconditional diffusion literature, such as~\citet{app:Kim2022Soft,app:Pidstrigach2022Score,app:duong2025telegrapher}.

The choice of $T_{0}$ balances two competing errors: the early-stopping bias, which vanishes as $T_{0}\to 0$, and the score estimation error, which diverges in the same limit due to the singularity noted above. Theorem~\ref{theorem:convergence} makes this trade-off explicit and shows that $T_{0}=\sqrt{\varepsilon}$ is optimal for this bound up to logarithmic factors, yielding a polynomial dependence of the overall error on the target accuracy $\varepsilon$. Analogous $T_{0}$-dependent analyses arise in the unconditional diffusion convergence literature~\citep{app:Chen2023Improved,app:Lee2023Convergence,app:li2024faster}, where early-stopping likewise plays a dual role as both regularizer (suppressing score blow-up) and rate-limiter (introducing a controllable bias at the target). The recurrence of the same trade-off in our posterior setting reinforces the view, made quantitative by Proposition~\ref{proposition:error:decomposition}, that the posterior-specific difficulty enters our bounds through $\kappa_{\vy}$, as an amplification factor on top of unconditional convergence machinery.

\subsection{Monte Carlo-based score estimation}
\label{appendix:discussions:mc:score}
Prior research has examined error analysis for Monte Carlo-based score estimation. For example, \citet[Proposition 3.1]{app:He2024Zeroth} decomposes the estimation error into components from Langevin dynamics and Monte Carlo approximation. The error bound established by~\citet{app:He2024Zeroth} is under the assumption that the score function of the target distribution is analytically available, which holds for standard unconditional sampling scenarios. However, the posterior sampling typically requires estimating the unknown prior score function from samples. Our work, Proposition~\ref{proposition:posterior:score} and  Corollary~\ref{corollary:posterior:score}, fills this gap by analyzing how the prior score estimation error $\varepsilon_{\prior}^{2}$ affects the posterior sampling performance and introducing the condition number $\kappa_{\vy}$ as a key characterization of this relationship.

\section{Discussions on Assumptions and Examples}
\label{section:concrete:examples}

\subsection{Discussion on Assumptions}
\label{appendix:assumptions}
Assumptions~\ref{assumption:semi:log:concave} and~\ref{assumption:prior:subGaussian} are accompanied by concise comments in the main text near their statements. This subsection further discusses Assumptions~\ref{assumption:posterior:bound:zero}--\ref{assumption:Lipschitz:prior:likelihood}, which encode the additional regularity and approximation conditions used in the convergence analysis.

\subsubsection{Assumption~\ref{assumption:posterior:bound:zero}}
Assumption~\ref{assumption:posterior:bound:zero} only pins down the magnitude of the posterior score at the origin, which is finite under Assumption~\ref{assumption:semi:log:concave}; the constant $H_{\vy}$ enters the convergence bounds as a normalization factor.

\subsubsection{Assumption~\ref{assumption:prior:score:error}}
Assumption~\ref{assumption:prior:score:error} requires that the $L^{2}$-error of the prior score estimator $\hat{\vs}_{\prior}$ be sufficiently small, where the error is measured with respect to the prior density $\pi_{0}$. In the present setting, numerous samples drawn from $\pi_{0}$ are available for estimation, and methods such as implicit score matching~\citep{app:hyvarinen2005Estimation}, sliced score matching~\citep{app:song2020Sliced}, or denoising score matching~\citep{app:vincent2011connection} yield an estimator $\hat{\vs}_{\prior}$ with theoretical guarantees in this metric. The assumption could equivalently be replaced by explicit score matching bounds~\citep{app:Oko2023Diffusion,app:ding2024characteristic,app:fu2024unveil,app:Tang2024Adaptivity} obtained through standard non-parametric regression techniques using deep neural networks~\citep{app:bauer2019deep,app:schmidt2020nonparametric,app:kohler2021rate,app:Jiao2023deep}; the present formulation is retained for clarity of presentation.

\subsubsection{Assumption~\ref{assumption:prior:bound}}
Assumption~\ref{assumption:prior:bound} bounds the prior score growth to be polynomial, ensuring that $\log\pi_{0}(\vx_{0})$ does not decay faster than $\exp(-c\|\vx_{0}\|_{2}^{r+1})$ for some $c>0$. The same bound on the estimator $\hat{\vs}_{\prior}$ is a requirement on the trained score network rather than on the prior. It complements Assumption~\ref{assumption:prior:subGaussian}: the former prevents tails from being too light, while the latter prevents them from being too heavy. Together, they characterize the permissible tail behavior---for instance, when $r=1$, the prior must have approximately Gaussian tails. Gaussian distribution, Gaussian mixture (Example~\ref{example:gaussian:mixture}), and  Gaussian convolution (Example~\ref{example:gaussian:convolution}) both satisfy this with $r=1$ (Appendix~\ref{appendix:examples}).

\subsubsection{Assumption~\ref{assumption:Lipschitz:prior:likelihood}}
Assumption~\ref{assumption:Lipschitz:prior:likelihood} is standard in posterior sampling error analyses such as~\citet[H1]{app:laumont2022bayesian} and~\citet[Assumptions 3.6 and 3.8]{app:cai2024nf}, and more broadly in Wasserstein convergence guarantees for diffusion or flow-based models~\citep{app:Kwon2022Score,app:chen2023probability,app:beyler2025convergence,app:Huang2025Convergence,app:kremling2025nonasymptotic}; in our setting it controls the Wasserstein error of the Langevin dynamics in Proposition~\ref{proposition:posterior:score}.

A sufficient condition is that $\hat{\vs}_{\prior}$ and $\nabla\ell_{\vy}$ each be uniformly Lipschitz continuous. On the likelihood side, this reduces to a uniform upper bound on $\nabla^{2}\ell_{\vy}$, which the running likelihood examples (Appendices~\ref{section:concrete:examples:linear} and~\ref{section:concrete:examples:general}) verify directly. On the prior side, $\hat{\vs}_{\prior}$ is the trained estimator and is not itself a property of the prior; it is typical instead to require that the true prior score $\nabla\log\pi_{0}$ be uniformly Lipschitz, and to combine an $L^{2}$-consistent estimator (i.e.\ Assumption~\ref{assumption:prior:score:error}) with a procedure that preserves Lipschitzness, so that the property transfers from $\nabla\log\pi_{0}$ to $\hat{\vs}_{\prior}$. The running prior examples (Appendices~\ref{section:concrete:examples:gaussian:mixture} and~\ref{section:concrete:examples:gaussian:convolution}) exhibit such a uniformly Lipschitz $\nabla\log\pi_{0}$.

In practice, where $\hat{\vs}_{\prior}$ is parameterized by a deep neural network, weight clipping~\citep{app:Arjovsky2017Wasserstein}, gradient penalty~\citep{app:Gulrajani2017Improved}, spectral normalization~\citep{app:miyato2018spectral}, and the Sobolev-regularized framework of~\citet[Theorem IV.3]{app:Ding2025Semi} provide concrete procedures that enforce the Lipschitz constraint while controlling the approximation error~\citep{app:chen2022distribution,app:huang2022error,app:Jiao2023deep,app:Ding2025Semi}.

\subsection{Concrete examples}
\label{appendix:examples}
The following examples, covering both the likelihood and the prior side, provide case-by-case verification of the relevant assumptions from Sections~\ref{section:method} and \ref{section:convergence}.

\subsubsection{Linear forward operator with Gaussian noise}
\label{section:concrete:examples:linear}
The simplest forward model is linear with isotropic Gaussian observation noise, in which case the negative log-likelihood is quadratic and its Hessian is constant.

\begin{example}[Linear forward operator with Gaussian noise]
\label{example:linear:gaussian}
Consider a linear forward model $\mY=\mA\mX_{0}+\vn$ with $\vn\sim N(\bzero,\mI_{n})$ and $\mA\in\bbR^{n\times d}$. The Hessian of the negative log-likelihood is the constant matrix $\nabla^{2}\ell_{\vy}(\vx_{0})=\mA^{\top}\mA$.
\end{example}

This Hessian admits the uniform two-sided eigenvalue bound
\begin{equation*}
\lambda_{\min}(\mA^{\top}\mA)\,\mI_{d}\preceq\nabla^{2}\ell_{\vy}(\vx_{0})\preceq\lambda_{\max}(\mA^{\top}\mA)\,\mI_{d}\,,\qquad\vx_{0}\in\bbR^{d}\,.
\end{equation*}
The lower bound supplies the likelihood-side contribution required by Assumption~\ref{assumption:semi:log:concave}: combined via Bayes' rule with any prior whose Hessian of $\log\pi_{0}$ is uniformly bounded above, it yields the lower bound on $-\nabla^{2}\log q_{0}(\cdot\given\vy)$ required there. The matching upper bound makes $\nabla\ell_{\vy}$ uniformly Lipschitz, which supplies the likelihood-side of the sufficient condition for Assumption~\ref{assumption:Lipschitz:prior:likelihood} stated in Appendix~\ref{appendix:assumptions}.

\subsubsection{General nonlinear forward operator with Gaussian noise}
\label{section:concrete:examples:general}
Many imaging forward operators are nonlinear; the same Hessian framework extends provided the operator and its first two derivatives are uniformly bounded.

\begin{example}[General forward operator with Gaussian noise]
\label{example:general:operator:gaussian}
Consider a general forward model $\mY=\calF(\mX_{0})+\vn$ with $\vn\sim N(\bzero,\mI_{n})$ and a twice-differentiable operator $\calF\colon\bbR^{d}\to\bbR^{n}$. A direct calculation gives
\begin{equation*}
\nabla^{2}\ell_{\vy}(\vx_{0})=\nabla\calF(\vx_{0})^{\top}\nabla\calF(\vx_{0})-\{\vy-\calF(\vx_{0})\}^{\top}\nabla^{2}\calF(\vx_{0})\,,
\end{equation*}
where $\nabla\calF$ and $\nabla^{2}\calF$ denote the Jacobian and Hessian of $\calF$.
\end{example}

If $\calF$, $\nabla\calF$, and $\nabla^{2}\calF$ are uniformly bounded in operator norm, then each term in the Hessian formula is bounded, so $\nabla^{2}\ell_{\vy}$ is uniformly bounded above and below on $\bbR^{d}$. This condition is met by a broad class of nonlinear operators; for instance, a deep neural network with $\tanh$ or $\mathrm{sigmoid}$ activations has uniformly bounded function values and derivatives of all orders. The resulting two-sided bound furnishes simultaneously the lower bound for the likelihood-side contribution required by Assumption~\ref{assumption:semi:log:concave} and the upper bound that makes $\nabla\ell_{\vy}$ uniformly Lipschitz, supplying the likelihood-side of the sufficient condition for Assumption~\ref{assumption:Lipschitz:prior:likelihood} stated in Appendix~\ref{appendix:assumptions}.

\subsubsection{Gaussian mixture prior}
\label{section:concrete:examples:gaussian:mixture}
Gaussian mixtures are a canonical non-log-concave, multimodal prior. They are foundational in statistics, physics, and machine learning, largely due to their role as universal approximators: any continuous density can be approximated arbitrarily well by a Gaussian mixture.

\begin{example}[Gaussian mixture]
\label{example:gaussian:mixture}
A mixture of Gaussians $\{N(\vm_{k},\sigma_{k}^{2}\mI_{d})\}_{k=1}^{K}$ has density
\begin{equation}\label{eq:gaussian:mixture}
\pi_{0}(\vx_{0})=\sum_{k=1}^{K}w_{k}\,\gamma_{d,\sigma_{k}^{2}}(\vx_{0}-\vm_{k})\,,
\end{equation}
where $\gamma_{d,\sigma^{2}}$ denotes the density of $N(\bzero,\sigma^{2}\mI_{d})$ and the weights satisfy $w_{k}>0$ and $\sum_{k=1}^{K}w_{k}=1$.
\end{example}

A direct calculation expresses the mixture score as a convex combination of component scores,
\begin{equation}\label{eq:concrete:gm:score}
\vs(\vx_{0})\coloneq\nabla\log\pi_{0}(\vx_{0})=\sum_{k=1}^{K}\xi_{k}(\vx_{0})\,\vs_{k}(\vx_{0})\,,
\qquad \vs_{k}(\vx_{0})\coloneq-\frac{\vx_{0}-\vm_{k}}{\sigma_{k}^{2}}\,,
\end{equation}
where the responsibility weights
\begin{equation}\label{eq:concrete:gm:weight}
\xi_{k}(\vx_{0})=\frac{w_{k}\,\gamma_{d,\sigma_{k}^{2}}(\vx_{0}-\vm_{k})}{\sum_{j=1}^{K}w_{j}\,\gamma_{d,\sigma_{j}^{2}}(\vx_{0}-\vm_{j})}
\end{equation}
lie in $(0,1)$ and sum to unity. From~\eqref{eq:concrete:gm:score} and the triangle inequality, the score satisfies the linear growth bound
\begin{equation}\label{eq:concrete:gm:score:growth}
\|\vs(\vx_{0})\|_{2}\leq\bigg\{\sum_{k=1}^{K}\frac{\xi_{k}(\vx_{0})}{\sigma_{k}^{2}}\bigg\}\|\vx_{0}\|_{2}+\sum_{k=1}^{K}\frac{\xi_{k}(\vx_{0})}{\sigma_{k}^{2}}\|\vm_{k}\|_{2}\,,
\end{equation}
in which both sums on the right are bounded above by $\max_{k}\sigma_{k}^{-2}\big(1+\|\vm_{k}\|_{2}\big)$, uniformly in $\vx_{0}$, since $\xi_{k}\in(0,1)$. Identities~\eqref{eq:concrete:gm:score} and~\eqref{eq:concrete:gm:score:growth} drive the verifications below.

Differentiating~\eqref{eq:concrete:gm:score} and using the identities $\nabla\xi_{k}=\xi_{k}(\vs_{k}-\vs)$ and $\nabla\vs_{k}=-\sigma_{k}^{-2}\mI_{d}$ yields
\begin{equation}\label{eq:concrete:gm:hessian}
\nabla\vs(\vx_{0})=\underbrace{\sum_{k=1}^{K}\xi_{k}(\vx_{0})\big\{\vs_{k}(\vx_{0})-\vs(\vx_{0})\big\}\vs_{k}(\vx_{0})^{\top}}_{\text{(I)}}-\underbrace{\sum_{k=1}^{K}\frac{\xi_{k}(\vx_{0})}{\sigma_{k}^{2}}\mI_{d}}_{\text{(II)}}\,.
\end{equation}
Term~(I) in~\eqref{eq:concrete:gm:hessian} is the covariance of the component scores under the responsibilities. Hence
\begin{align} \label{eq:concrete:gm:hessian:covariance}
\nabla \vs(\vx_{0}) =-\sum_{k=1}^{K}\frac{\xi_{k}(\vx_{0})}{\sigma_{k}^{2}}\mI_{d}
+\frac{1}{2}\sum_{k,j=1}^{K}\xi_{k}(\vx_{0})\xi_{j}(\vx_{0})
\{\vs_{k}(\vx_{0})-\vs_{j}(\vx_{0})\}
\{\vs_{k}(\vx_{0})-\vs_{j}(\vx_{0})\}^{\top}\,. 
\end{align}
The first term has operator norm at most $\max_{k}\sigma_{k}^{-2}$. Let $f_{k}(\vx_{0})=w_{k}\gamma_{d,\sigma_{k}^{2}}(\vx_{0}-\vm_{k})$. For every pair $(k,j)$,
\begin{equation*}
\xi_{k}\xi_{j}
\leq\frac{f_{k}f_{j}}{(f_{k}+f_{j})^{2}}
\leq\exp\{-|\log(f_{k}/f_{j})|\}\,.
\end{equation*}
If $\sigma_{k}=\sigma_{j}$, then $\vs_{k}-\vs_{j}$ is constant. Otherwise, $\log(f_{k}/f_{j})$ has a nonzero quadratic leading term, whereas $\vs_{k}-\vs_{j}$ is affine in $\vx_{0}$. Thus $\xi_{k}\xi_{j}\|\vs_{k}-\vs_{j}\|_{2}^{2}$ is uniformly bounded. Since $K$ is finite, the Hessian in~\eqref{eq:concrete:gm:hessian:covariance} is uniformly bounded in operator norm. Consequently, the mixture score is globally Lipschitz and $\nabla^{2}\log\pi_{0}\preceq\alpha\mI_{d}$ for some $\alpha>0$, providing the prior-side curvature bound for Assumption~\ref{assumption:semi:log:concave}. 

Substituting~\eqref{eq:gaussian:mixture} and completing the square,
\begin{align*}
&\int\exp\bigg(\frac{\|\vx_{0}\|_{2}^{2}}{V_{\SG}^{2}}\bigg)\pi_{0}(\vx_{0})\d\vx_{0} \\
&~~~=\sum_{k=1}^{K}w_{k}(2\pi\sigma_{k}^{2})^{-d/2}\int\exp\bigg\{-\bigg(\frac{1}{2\sigma_{k}^{2}}-\frac{1}{V_{\SG}^{2}}\bigg)\|\vx_{0}\|_{2}^{2}+\frac{1}{\sigma_{k}^{2}}\langle\vx_{0},\vm_{k}\rangle-\frac{\|\vm_{k}\|_{2}^{2}}{2\sigma_{k}^{2}}\bigg\}\d\vx_{0}\,.
\end{align*}
The quadratic form in the exponent is negative definite precisely when ${1}/{2\sigma_{k}^{2}}-{1}/{V_{\SG}^{2}}>0$ for every $k$, that is, when $V_{\SG}^{2}>2\max_{k}\sigma_{k}^{2}$. Under this condition, each term is a finite Gaussian integral and the sum converges, establishing Assumption~\ref{assumption:prior:subGaussian}.

The linear growth bound~\eqref{eq:concrete:gm:score:growth} together with the uniform bound on its right-hand side gives $\|\vs(\vx_{0})\|_{2}=O(\|\vx_{0}\|_{2})$, so the true-score half of Assumption~\ref{assumption:prior:bound} holds with $r=1$.

The Hessian analysis used for Assumption~\ref{assumption:semi:log:concave} additionally yields the matching lower bound $\nabla^{2}\log\pi_{0}\succeq-\alpha\mI_{d}$, so the score $\vs=\nabla\log\pi_{0}$ is uniformly Lipschitz, supplying the prior-side of the sufficient condition for Assumption~\ref{assumption:Lipschitz:prior:likelihood} stated in Appendix~\ref{appendix:assumptions}.

\subsubsection{Gaussian convolution prior}
\label{section:concrete:examples:gaussian:convolution}
The Gaussian convolution prior has been considered in several related studies~\citep{app:Chen2021Dimension,app:ding2024characteristic,app:grenioux2024stochastic,app:saremi2024chain,app:beyler2025convergence}. It is particularly relevant in computer vision and image processing, where pixel values are bounded and observations are perturbed by small Gaussian noise.

\begin{example}[Gaussian convolution]
\label{example:gaussian:convolution}
Let $\nu$ be a probability distribution on $\bbR^{d}$ with $\supp(\nu)\subseteq[0,1]^{d}$ and let $\sigma>0$. The prior density is the Gaussian convolution
\begin{equation}\label{eq:gaussian:convolution}
\pi_{0}(\vx_{0})=\int\gamma_{d,\sigma^{2}}(\vx_{0}-\vz)\d\nu(\vz)\,.
\end{equation}
Because $\nu$ may be arbitrary (e.g., multimodal) and $\sigma>0$ may be small, this family includes non-log-concave priors.
\end{example}

The verifications below make use of the Tweedie-type identity
\begin{equation}\label{eq:concrete:gc:score}
\vs(\vx_{0})\coloneq\nabla\log\pi_{0}(\vx_{0})=-\frac{1}{\sigma^{2}}\big\{\vx_{0}-\bar{\vz}(\vx_{0})\big\}\,, \qquad \bar{\vz}(\vx_{0})\coloneq\frac{\int\vz\,\gamma_{d,\sigma^{2}}(\vx_{0}-\vz)\d\nu(\vz)}{\pi_{0}(\vx_{0})}\,,
\end{equation}
where $\bar{\vz}(\vx_{0})=\bbE_{\nu}(\mZ\mid\vx_{0})\in[0,1]^{d}$ is the posterior mean of $\mZ\sim\nu$ given $\vx_{0}$ (the inclusion follows because $\supp(\nu)\subseteq[0,1]^{d}$ and the posterior mean lies in its convex hull).

Differentiating~\eqref{eq:concrete:gc:score} gives the Tweedie Hessian identity
\begin{equation*}
\nabla^{2}\log\pi_{0}(\vx_{0})=-\frac{1}{\sigma^{2}}\mI_{d}+\frac{1}{\sigma^{4}}\mathrm{Cov}_{\nu}(\mZ\mid\vx_{0})\,,
\end{equation*}
where $\mathrm{Cov}_{\nu}(\mZ\mid\vx_{0})\succeq\bzero$ is the posterior covariance of $\mZ$ given $\vx_{0}$. Because $\supp(\nu)\subseteq[0,1]^{d}$, this covariance is uniformly bounded: $\mathrm{Cov}_{\nu}(\mZ\mid\vx_{0})\preceq d\,\mI_{d}$. Combining the two sides yields
\begin{equation*}
-\frac{1}{\sigma^{2}}\mI_{d}\preceq\nabla^{2}\log\pi_{0}(\vx_{0})\preceq\bigg(\frac{d}{\sigma^{4}}-\frac{1}{\sigma^{2}}\bigg)\mI_{d}\,, \qquad \vx_{0}\in\bbR^{d}\,,
\end{equation*}
so this supplies the prior-side curvature bound used, together with the likelihood-side lower Hessian bound, to verify Assumption~\ref{assumption:semi:log:concave} for a complete posterior model. See also~\citet[Propositions 3.2 and 3.5]{app:ding2024characteristic}.

By Fubini's theorem and completion of the square,
\begin{align*}
&\int\exp\bigg(\frac{\|\vx_{0}\|_{2}^{2}}{V_{\SG}^{2}}\bigg)\pi_{0}(\vx_{0})\d\vx_{0} \\
&~~~=\int(2\pi\sigma^{2})^{-d/2}\bigg[\int\exp\bigg\{-\bigg(\frac{1}{2\sigma^{2}}-\frac{1}{V_{\SG}^{2}}\bigg)\|\vx_{0}\|_{2}^{2}+\frac{1}{\sigma^{2}}\langle\vx_{0},\vz\rangle-\frac{\|\vz\|_{2}^{2}}{2\sigma^{2}}\bigg\}\d\vx_{0}\bigg]\d\nu(\vz)\,.
\end{align*}
Whenever $V_{\SG}^{2}>2\sigma^{2}$, the quadratic form in the inner integral is negative definite, so the inner Gaussian integral is finite and depends continuously on $\vz$. Since $\supp(\nu)$ is compact, the outer integral is finite, establishing Assumption~\ref{assumption:prior:subGaussian}.

Since $\bar{\vz}(\vx_{0})\in[0,1]^{d}$ implies $\|\bar{\vz}(\vx_{0})\|_{2}\leq\sqrt{d}$, identity~\eqref{eq:concrete:gc:score} yields $\|\vs(\vx_{0})\|_{2}\leq\sigma^{-2}(\|\vx_{0}\|_{2}+\sqrt{d})$, hence $\|\vs(\vx_{0})\|_{2}=O(\|\vx_{0}\|_{2})$, and the true-score half of Assumption~\ref{assumption:prior:bound} holds with $r=1$~\citep[Proposition~3.5]{app:ding2024characteristic}.

The two-sided Hessian bound used for Assumption~\ref{assumption:semi:log:concave} also implies that the score $\vs=\nabla\log\pi_{0}$ is uniformly Lipschitz, supplying the prior-side of the sufficient condition for Assumption~\ref{assumption:Lipschitz:prior:likelihood} stated in Appendix~\ref{appendix:assumptions}.

\section{Proofs of Propositions in Section~\ref{section:method}}
\label{appendix:proofs:methods}
In this section, we provide proofs of the theoretical results in Section~\ref{section:method}. 

\subsection{Proof of Proposition~\ref{proposition:method:score:denoiser}}
\label{appendix:proof:denoiser}
Recall the transition distribution of ~\eqref{eq:method:forward:solution}. We have 
\begin{equation*}
p_{\mX_{t}\given\mX_{0}}(\vx\given\vx_{0})\propto\exp\bigg(-\frac{\|\vx-\mu_{t}\vx_{0}\|_{2}^{2}}{2\sigma_{t}^{2}}\bigg)\,,
\end{equation*}
which implies
\begin{equation}\label{eq:proof:lemma:tweedie:posterior:1}
\nabla_{\vx}p_{\mX_{t}\given\mX_{0}}(\vx\given\vx_{0})
=\frac{\mu_{t}\vx_{0}-\vx}{\sigma_{t}^{2}}p_{\mX_{t}\given\mX_{0}}(\vx\given\vx_{0})\,.
\end{equation}
Using Chapman-Kolmogorov identity yields
\begin{align}
q_{t}(\vx\given\vy)
&=\int p_{\mX_{t}\given\mX_{0},\mY}(\vx\given\vx_{0},\vy)p_{\mX_{0}\given\mY}(\vx_{0}\given\vy)\d\vx_{0} \nonumber \\
&=\int p_{\mX_{t}\given\mX_{0}}(\vx\given\vx_{0})p_{\mX_{0}\given\mY}(\vx_{0}\given\vy)\d\vx_{0}\,, \label{eq:proof:lemma:tweedie:posterior:2}
\end{align} 
where the second equality holds from the fact that $\mX_{t}$ is conditional independent of $\mY$ given $\mX_{0}$. Then we find
\begin{align*}
\nabla\log q_{t}(\vx\given\vy)
&=\frac{1}{q_{t}(\vx\given\vy)}\int\nabla p_{\mX_{t}\given\mX_{0}}(\vx\given\vx_{0})p_{\mX_{0}\given\mY}(\vx_{0}\given\vy)\d\vx_{0} \\
&=\frac{1}{q_{t}(\vx\given\vy)}\int\frac{\mu_{t}\vx_{0}-\vx}{\sigma_{t}^{2}}p_{\mX_{t}\given\mX_{0}}(\vx\given\vx_{0})p_{\mX_{0}\given\mY}(\vx_{0}\given\vy)\d\vx_{0} \\
&=\frac{1}{q_{t}(\vx\given\vy)}\int\frac{\mu_{t}\vx_{0}-\vx}{\sigma_{t}^{2}}p_{\mX_{t},\mX_{0}\given\mY}(\vx,\vx_{0}\given\vy)\d\vx_{0} \\
&=\int\frac{\mu_{t}\vx_{0}-\vx}{\sigma_{t}^{2}}p_{\mX_{0}\given\mX_{t},\mY}(\vx_{0}\given\vx,\vy)\d\vx_{0}\,,
\end{align*}
where the first equality holds from~\eqref{eq:proof:lemma:tweedie:posterior:2}, the second equality invokes~\eqref{eq:proof:lemma:tweedie:posterior:1}, the third equality uses the fact that $\mX_{t}$ is conditional independent of $\mY$ given $\mX_{0}$, and the last equality is due to the definition of the conditional density. This completes the proof.

\subsection{Proof of Proposition~\ref{proposition:appendix:RGO:Hessian}}
\label{appendix:proof:log:concave:RGO}
According to Bayes' rule, we have
\begin{align*}
p_{\mX_{0}\given\mX_{t},\mY}(\vx_{0}\given\vx,\vy)
&=\frac{1}{p_{\mX_{t}\given\mY}(\vx\given\vy)}p_{\mX_{t}\given\mX_{0},\mY}(\vx\given\vx_{0},\vy)p_{\mX_{0}\given\mY}(\vx_{0}\given\vy) \\
&=\frac{1}{p_{\mX_{t}\given\mY}(\vx\given\vy)}p_{\mX_{t}\given\mX_{0}}(\vx\given\vx_{0})p_{\mX_{0}\given\mY}(\vx_{0}\given\vy)\,,
\end{align*} 
where the last equality invokes the fact that $\mX_{t}$ is conditional independent of $\mY$ given $\mX_{0}$. As a consequence, 
\begin{equation*}
p_{\mX_{0}\given\mX_{t},\mY}(\vx_{0}\given\vx,\vy)
\propto \exp\bigg(-\frac{\|\vx-\mu_{t}\vx_{0}\|_{2}^{2}}{2\sigma_{t}^{2}}\bigg)p_{\mX_{0}\given\mY}(\vx_{0}\given\vy)\,,
\end{equation*}
which implies 
\begin{equation*}
\nabla_{\vx_{0}}^{2}\log p_{\mX_{0}\given\mX_{t},\mY}(\vx_{0}\given\vx,\vy)
=-\frac{\mu_{t}^{2}}{\sigma_{t}^{2}}\mI_{d}+\nabla_{\vx_{0}}^{2}\log p_{\mX_{0}\given\mY}(\vx_{0}\given\vy)\,.
\end{equation*}
Combining the above equality with Assumption~\ref{assumption:semi:log:concave} completes the proof.

\subsection{Proof of Proposition~\ref{proposition:sub:Gaussian:posterior}}
\label{appendix:proof:sub:Gaussian:posterior}
It is straightforward that 
\begin{align*}
&\int \exp\bigg(\frac{\|\mathbf{x}_0\|_2^2}{V_{\text{SG}}^2}\bigg) q_0(\mathbf{x}_0\given\mathbf{y}) \, \mathrm{d}\mathbf{x}_0 \\
&~~~~~~= \int \exp\bigg(\frac{\|\mathbf{x}_0\|_2^2}{V_{\text{SG}}^2}\bigg) \frac{\exp\{-\ell_{\mathbf{y}}(\mathbf{x}_0)\}}{\int \exp\{-\ell_{\mathbf{y}}(\mathbf{x}_0)\} \pi_0(\mathbf{x}_0) \mathrm{d}\mathbf{x}_0} \pi_0(\mathbf{x}_0) \, \mathrm{d}\mathbf{x}_0 \\
&~~~~~~\le \frac{\sup_{\mathbf{x}_0} \exp\{-\ell_{\mathbf{y}}(\mathbf{x}_0)\}}{\int \exp\{-\ell_{\mathbf{y}}(\mathbf{x}_0)\} \pi_0(\mathbf{x}_0) \mathrm{d}\mathbf{x}_0} \int \exp\bigg(\frac{\|\mathbf{x}_0\|_2^2}{V_{\text{SG}}^2}\bigg) \pi_0(\mathbf{x}_0) \, \mathrm{d}\mathbf{x}_0 \le \kappa_{\mathbf{y}} C_{\text{SG}}\,,
\end{align*}
where the first equality follows from the Bayes' rule and the second inequality holds from Assumption~\ref{assumption:prior:subGaussian}. This completes the proof.

\subsection{Proof of Proposition~\ref{proposition:log:sobolev}}
\label{appendix:proof:log:sobolev}
Before proceeding, we introduce two auxiliary lemmas.

\begin{lemma}
\label{lemma:appendix:sub:gaussian}
Let $\mu$ be a probability distribution on $\bbR^{d}$ having sub-Gaussian tails, that is, there exist constants $V_{\SG}>0$ and $C_{\SG}>0$, such that 
\begin{equation*}
\int\exp\bigg(\frac{\|\vx\|_{2}^{2}}{V_{\SG}^{2}}\bigg)\d\mu(\vx)\leq C_{\SG}\,.
\end{equation*}
Then it holds that 
\begin{equation*}
\iint\exp\bigg(\frac{\|\vx-\vx^{\prime}\|_{2}^{2}}{2V_{\SG}^{2}}\bigg)\d\mu(\vx)\d\mu(\vx^{\prime})\leq C_{\SG}^{2}\,.
\end{equation*}
\end{lemma}

\begin{proof}[Proof of Lemma~\ref{lemma:appendix:sub:gaussian}]
It is straightforward that 
\begin{align*}
&\iint\exp\bigg(\frac{\|\vx-\vx^{\prime}\|_{2}^{2}}{2V_{\SG}^{2}}\bigg)\d\mu(\vx)\d\mu(\vx^{\prime}) \\
&~~~~~~\leq\iint\exp\bigg(\frac{\|\vx\|_{2}^{2}+\|\vx^{\prime}\|_{2}^{2}}{V_{\SG}^{2}}\bigg)\d\mu(\vx)\d\mu(\vx^{\prime}) \\
&~~~~~~=\bigg\{\int\exp\bigg(\frac{\|\vx\|_{2}^{2}}{V_{\SG}^{2}}\bigg)\d\mu(\vx)\bigg\}\bigg\{\int\exp\bigg(\frac{\|\vx^{\prime}\|_{2}^{2}}{V_{\SG}^{2}}\bigg)\d\mu(\vx^{\prime})\bigg\}\leq C_{\SG}^{2}\,,
\end{align*}
which completes the proof.
\end{proof}

\begin{lemma}[{\citet[Theorem 2]{app:Chen2021Dimension}}]
\label{lemma:appendix:LSI}
Let $\mu$ be a probability distribution on $\bbR^{d}$ for which there exist constants $V_{\SG}>0$ and $C_{\SG}>0$, such that
\begin{equation*}
\iint\exp\bigg(\frac{\|\vx-\vx^{\prime}\|_{2}^{2}}{2V_{\SG}^{2}}\bigg)\d\mu(\vx)\d\mu(\vx^{\prime})\leq C_{\SG}^{2}\,.
\end{equation*}
Then for $\sigma^{2}>2V_{\SG}^{2}$, the convolution distribution 
\begin{equation*}
\mu_{\sigma}(\vz)\coloneq \int\gamma_{d,\sigma^{2}}(\vz-\vx)\d\mu(\vx)
\end{equation*}
satisfies the log-Sobolev inequality with 
\begin{equation*}
C_{\LSI}(\mu_{\sigma})\leq 3\sigma^{2}\bigg\{\frac{\sigma^{2}}{\sigma^{2}-2V_{\SG}^{2}}+\exp\bigg(\frac{2V_{\SG}^{2}}{\sigma^{2}-2V_{\SG}^{2}}\log C_{\SG}^{2}\bigg)\bigg\}\bigg\{1+\frac{2V_{\SG}^{2}}{\sigma^{2}-2V_{\SG}^{2}}\log C_{\SG}^{2}\bigg\}\,.
\end{equation*}
\end{lemma}

Now, define an auxiliary distribution $\calM(\mu_{t})\sharp q_{0}(\vz_{t}\given\vy)\coloneq \mu_{t}^{-d}q_{0}(\mu_{t}^{-1}\vz_{t}\given\vy)$. We first verify the two-point condition of Lemma~\ref{lemma:appendix:LSI} for this distribution. Indeed,
\begin{align}
&\iint\exp\bigg(\frac{\|\vz_{t}-\vz_{t}^{\prime}\|_{2}^{2}}{2\mu_{t}^{2}V_{\SG}^{2}}\bigg)\calM(\mu_{t})\sharp q_{0}(\vz_{t}\given\vy)\calM(\mu_{t})\sharp q_{0}(\vz_{t}^{\prime}\given\vy)\d\vz_{t}\d\vz_{t}^{\prime} \nonumber \\
&~~~~~~=\iint\exp\bigg(\frac{\|\vx_{0}-\vx_{0}^{\prime}\|_{2}^{2}}{2V_{\SG}^{2}}\bigg)q_{0}(\vx_{0}\given\vy)q_{0}(\vx_{0}^{\prime}\given\vy)\d\vx_{0}\d\vx_{0}^{\prime}\leq\kappa_{\vy}^{2}C_{\SG}^{2}\,, \label{eq:lemma:log:sobolev:1}
\end{align}
where the inequality holds from Proposition~\ref{proposition:sub:Gaussian:posterior} and Lemma~\ref{lemma:appendix:sub:gaussian}. We next turn to verify the log-Sobolev inequality of $q_{t}(\cdot\given\vy)$. By recalling~\eqref{eq:method:marginal:y}, we find 
\begin{equation*}
q_{t}(\vx_{t}\given\vy)=\int\gamma_{d,\sigma_{t}^{2}}(\vx_{t}-\mu_{t}\vx_{0})q_{0}(\vx_{0}\given\vy)\d\vx_{0}=\int\gamma_{d,\sigma_{t}^{2}}(\vx_{t}-\vz_{t})\calM(\mu_{t})\sharp q_{0}(\vz_{t}\given\vy)\d\vz_{t}\,,
\end{equation*}
where we used a change of variable $\vz_{t}\coloneq \mu_{t}\vx_{0}$. Substituting~\eqref{eq:lemma:log:sobolev:1} into Lemma~\ref{lemma:appendix:LSI}, it follows that 
\begin{equation*}
C_{\LSI}\{q_{t}(\cdot\given\vy)\}
\leq 3\sigma_{t}^{2}\bigg\{\frac{\sigma_{t}^{2}}{\sigma_{t}^{2}-2\mu_{t}^{2}V_{\SG}^{2}}+\exp\bigg(\frac{2\mu_{t}^{2}V_{\SG}^{2}}{\sigma_{t}^{2}-2\mu_{t}^{2}V_{\SG}^{2}}\log\kappa_{\vy}^{2}C_{\SG}^{2}\bigg)\bigg\}^{2}\,,
\end{equation*}
where we used the fact 
\begin{equation*}
1+\frac{2\mu_{t}^{2}V_{\SG}^{2}}{\sigma_{t}^{2}-2\mu_{t}^{2}V_{\SG}^{2}}\log\kappa_{\vy}^{2}C_{\SG}^{2}\leq\frac{\sigma_{t}^{2}}{\sigma_{t}^{2}-2\mu_{t}^{2}V_{\SG}^{2}}+\exp\bigg(\frac{2\mu_{t}^{2}V_{\SG}^{2}}{\sigma_{t}^{2}-2\mu_{t}^{2}V_{\SG}^{2}}\log\kappa_{\vy}^{2}C_{\SG}^{2}\bigg)\,.
\end{equation*}
Here we used $\sigma_{t}^{2}-2\mu_{t}^{2}V_{\SG}^{2}>0$ for $t>\underline{t}\coloneq \frac{1}{2}\log(1+2V_{\SG}^{2})$. Since $\kappa_{\vy}^{2}C_{\SG}^{2}\geq e$ by assumption, we have $\log(\kappa_{\vy}^{2}C_{\SG}^{2})\geq1$. Further, note that
\begin{equation*}
\frac{\sigma_{t}^{2}}{\sigma_{t}^{2}-2\mu_{t}^{2}V_{\SG}^{2}}\leq \frac{\sigma_{t}^{2}+2\mu_{t}^{2}V_{\SG}^{2}}{\sigma_{t}^{2}-2\mu_{t}^{2}V_{\SG}^{2}}\leq\exp\bigg(\frac{\sigma_{t}^{2}+2\mu_{t}^{2}V_{\SG}^{2}}{\sigma_{t}^{2}-2\mu_{t}^{2}V_{\SG}^{2}}\log\kappa_{\vy}^{2}C_{\SG}^{2}\bigg)\,.
\end{equation*}
This completes the proof.

\subsection{Proof of Theorem~\ref{theorem:duality}}
\label{appendix:proof:theorem:duality}
Recall from Proposition~\ref{proposition:appendix:RGO:Hessian} that the posterior denoising density $p_{t}(\cdot\given\vx,\vy)$ is log-concave for every $t\in(0,\bar{t})$, with $\underline{t}$ and $\bar{t}$ defined in the statement of the theorem. Similarly, by Proposition~\ref{proposition:log:sobolev}, the terminal posterior density $q_{T}(\cdot\given\vy)$ satisfies a log-Sobolev inequality for every $T>\underline{t}$. Therefore, to ensure conclusions (i) and (ii) simultaneously, it suffices to show that the interval $(\underline{t},\bar{t})$ is non-empty, i.e., $\underline{t}<\bar{t}$. This is equivalent to $1+2V_{\SG}^{2}<1+\alpha^{-1}$, that is, $2\alpha V_{\SG}^{2}<1$, exactly the assumption of the theorem. Consequently, for any $T\in(\underline{t},\bar{t})$, both (i) $p_{t}(\cdot\given\vx,\vy)$ is log-concave for all $t\in(0,T]$, and (ii) $q_{T}(\cdot\given\vy)$ satisfies a log-Sobolev inequality. This completes the proof.

\section{The Monte Carlo Score Field and Score Regularity}
\label{appendix:random:field}
This appendix formalizes the random score field convention of Section~\ref{section:method:score}. The estimator $\hat{\vs}_{m}^{S}$ in~\eqref{eq:posterior:score:estimate} is defined at each fixed query point $(t,\vx)$ through the $m$ inner Langevin particles of~\eqref{eq:RGO:Langevin:score}, whereas the warm-start dynamics~\eqref{section:method:warm:Langevin:score} and the approximate time-reversal process~\eqref{eq:reversal:score} require a single drift function of $(t,\vx)$. Reusing one random seed per particle index at every query point, with independent seeds across indices, turns the pointwise estimators into such a function. Appendix~\ref{appendix:random:field:construction} constructs the estimated score field and establishes its regularity, and Appendix~\ref{appendix:random:field:true:score} establishes the matching regularity of the true posterior score.

\subsection{Construction of the estimated score field}
\label{appendix:random:field:construction}
The construction begins with the following assumption on the initialization distributions in Algorithm~\ref{alg:unbiased:posterior:score}, which is in force throughout the theoretical analysis.

\begin{assumption}[Coupled measurable initialization]\label{assumption:random:field:initialization}
For a fixed measurement $\vy\in\bbR^{n}$, the map $(t,\vx)\mapsto\hat p_t^0(\cdot\given\vx,\vy)$ is a Borel probability kernel on $(0,\infty)\times\bbR^d$ admitting a jointly Borel map
$\Phi_0:(0,\infty)\times\bbR^d\times[0,1]\to\bbR^d$ such that, for $V$ uniformly distributed on $[0,1]$,
\begin{equation*}
\Phi_0(t,\vx,V)\sim\hat p_t^0(\cdot\given\vx,\vy)\,.
\end{equation*}
Moreover, for every compact interval $I\subset(0,\infty)$ there are a constant $L_I<\infty$ and a measurable $A_I:[0,1]\to[0,\infty)$ with $\bbE\{A_I(V)^2\}<\infty$ such that, for all $t\in I$, $\vx,\vx'\in\bbR^d$, and $v\in[0,1]$,
\begin{equation*}
\|\Phi_0(t,\vx,v)-\Phi_0(t,\vx',v)\|_2\leq L_I\|\vx-\vx'\|_2\,,
\qquad
\|\Phi_0(t,\bzero,v)\|_2\leq A_I(v)\,.
\end{equation*}
\end{assumption}

Realizing a probability kernel through a uniform seed is a classical randomization construction on standard Borel spaces. Define the canonical environment
\begin{equation*}
\Omega_{\rm MC}:=\big([0,1]\times C([0,\infty);\bbR^d)\big)^{m}\,,
\end{equation*}
endowed with the product law under which the coordinate pairs $\{(V_i,\mB_i^{\rm in})\}_{i=1}^m$ are independent, each $V_i$ is uniformly distributed on $[0,1]$, and each $\mB_i^{\rm in}$ is a $d$-dimensional Brownian motion. For each query point $(t,\vx)$, let $\what{\mX}_{0,\cdot,i}^{\vx,\vy,t}$ be the strong solution of~\eqref{eq:RGO:Langevin:score} initialized at $\Phi_0(t,\vx,V_i)$ and driven by $\mB_i^{\rm in}$. Substituting the endpoints into~\eqref{eq:denoiser:MC}--\eqref{eq:posterior:score:estimate} defines the random score field $\hat{\vs}_m^S(t,\vx,\vy;\omega)$ and realizes the seed $\sigma$-field $\calG:=\sigma\{V_i,\mB_i^{\rm in}:1\leq i\leq m\}$ of Section~\ref{section:method:score}. Whenever the exact and estimated inner dynamics~\eqref{eq:RGO:Langevin}--\eqref{eq:RGO:Langevin:score} are synchronously coupled, they share the initial variable and the Brownian motion.

For the outer simulations, we enlarge $\Omega_{\rm MC}$ by three further independent coordinates: the initial variable of the warm-start dynamics, with density $\hat q_T^0(\cdot\given\vy)$, and two independent $d$-dimensional Brownian motions driving~\eqref{section:method:warm:Langevin:score} and~\eqref{eq:reversal:score}. These coordinates, together with reference variables such as $\mX_{t}^{\vy}\sim q_{t}(\cdot\given\vy)$, are independent of $\calG$. For a fixed environment $\omega$, denote by $\hat q_T^{u,\omega}(\cdot\given\vy)$ and $\hat q_t^{\omega}(\cdot\given\vy)$ regular versions of the conditional distributions of $\what{\mX}_{T,u}^{\vy}$ and $\what{\mX}_{t}^{\vy}$ given $\calG$, and set $\hat q_{T-T_0}^{R,\omega}:=\calT_R\sharp\hat q_{T-T_0}^{\omega}$. The argument $\omega$ is suppressed outside this appendix, and every expectation in the main text also averages over the environment, in accordance with the convention of Section~\ref{section:convergence}. The next lemma records the properties needed in the analysis.

\begin{lemma}[Measurable random score field]\label{lemma:random:score:field}
Suppose Assumptions~\ref{assumption:Lipschitz:prior:likelihood} and~\ref{assumption:random:field:initialization} hold, and let $I\subset(0,\infty)$ be a compact interval.
\begin{enumerate}[label=(\roman*)]
\item The maps $(\omega,t,\vx)\mapsto\what{\mX}_{0,S,i}^{\vx,\vy,t}$ and $(\omega,t,\vx)\mapsto\hat{\vs}_m^S(t,\vx,\vy)$ are jointly measurable on $\Omega_{\rm MC}\times I\times\bbR^d$. For each fixed $(t,\vx)$, the endpoints $\what{\mX}_{0,S,1}^{\vx,\vy,t},\ldots,\what{\mX}_{0,S,m}^{\vx,\vy,t}$ are i.i.d.\ with common law $\hat p_t^S(\cdot\given\vx,\vy)$.
\item There are a deterministic constant $L_I'<\infty$ and an almost surely finite $\calG$-measurable variable $K_I$ such that, for all $t\in I$ and $\vx,\vx'\in\bbR^d$,
\begin{align*}
\|\hat{\vs}_m^S(t,\vx,\vy)-\hat{\vs}_m^S(t,\vx',\vy)\|_2
\leq L_I'\|\vx-\vx'\|_2\,, \quad
\|\hat{\vs}_m^S(t,\vx,\vy)\|_2\leq K_I(1+\|\vx\|_2)\,.
\end{align*}
\item Suppose additionally that Assumptions~\ref{assumption:semi:log:concave} and~\ref{assumption:posterior:bound:zero} hold and that $T\in I$ satisfies $T<\bar{t}$, so that the radius $a_{T}$ in~\eqref{eq:warm:score:radius} is well defined. Then the clipped field $\tilde{\vs}_{m}^{S}(T,\cdot,\vy)$ in~\eqref{eq:warm:score:clipped} satisfies the properties of items (i) and (ii).
\item The conditional laws $\hat q_T^{u,\omega}$ and $\hat q_t^{\omega}$ admit versions that are measurable in $\omega$. In particular, the total-variation and truncated Wasserstein functionals appearing in the analysis are random variables with well-defined expectations.
\end{enumerate}
Consequently, after enlarging the outer filtrations at time zero by $\calG$, the approximate time-reversal process~\eqref{eq:reversal:score} admits a unique non-explosive strong solution conditionally on $\calG$, and so does the clipped warm-start dynamics~\eqref{section:method:warm:Langevin:score} under the conditions of item (iii).
\end{lemma}

\begin{proof}[Proof of Lemma~\ref{lemma:random:score:field}]
Write $\vh_{\vy}=\hat{\vs}_{\prior}-\nabla\ell_{\vy}$, $\rho_t=\mu_t^2/\sigma_t^2$, $c_t=\mu_t/\sigma_t^2$, and $\rho_I^*=\sup_{t\in I}\rho_t$, $c_I^*=\sup_{t\in I}c_t$, the latter two finite on the compact interval $I$. The inner drift $\vb(t,\vx,\vz)=\vh_{\vy}(\vz)-\rho_t\vz+c_t\vx$ is $(G+\rho_t)$-Lipschitz in $\vz$ by Assumption~\ref{assumption:Lipschitz:prior:likelihood}, so the inner SDE admits a unique non-explosive strong solution. Picard iteration for the associated integral equation, performed simultaneously in $(t,\vx)$, consists of jointly measurable maps by Assumption~\ref{assumption:random:field:initialization}, and its pointwise limit gives the joint measurability. Since the $m$ solutions are measurable functionals of the i.i.d.\ pairs $(V_i,\mB_i^{\rm in})$, the fixed-query endpoints are i.i.d.\ with the stated law. This proves part (i).

For part (ii), fix $1\leq i\leq m$ and $t\in I$. For $s\in[0,S]$, the solution at a query point $\vx$ satisfies the integral equation
\begin{equation}\label{eq:lemma:random:field:integral}
\what{\mX}_{0,s,i}^{\vx,\vy,t}=\Phi_0(t,\vx,V_i)+\int_0^s\big\{\vh_{\vy}(\what{\mX}_{0,v,i}^{\vx,\vy,t})-\rho_t\what{\mX}_{0,v,i}^{\vx,\vy,t}+c_t\vx\big\}\d v+\sqrt{2}\,\mB_i^{\rm in}(s) \,.
\end{equation}
Evolve the solutions at query points $\vx$ and $\vx'$ with the same pair $(V_i,\mB_i^{\rm in})$, and write $\Delta_s:=\what{\mX}_{0,s,i}^{\vx,\vy,t}-\what{\mX}_{0,s,i}^{\vx',\vy,t}$. Subtracting the two copies of~\eqref{eq:lemma:random:field:integral}, the Brownian terms cancel, and
\begin{equation*}
\Delta_s=\Phi_0(t,\vx,V_i)-\Phi_0(t,\vx',V_i)+\int_0^s\big\{\vh_{\vy}(\what{\mX}_{0,v,i}^{\vx,\vy,t})-\vh_{\vy}(\what{\mX}_{0,v,i}^{\vx',\vy,t})-\rho_t\Delta_v\big\}\d v+sc_t(\vx-\vx')\,.
\end{equation*}
The initialization term is bounded by $L_I\|\vx-\vx'\|_2$ by Assumption~\ref{assumption:random:field:initialization}, the integrand is bounded by $(G+\rho_t)\|\Delta_v\|_2$ by Assumption~\ref{assumption:Lipschitz:prior:likelihood}. Hence, for $s\in[0,S]$,
\begin{equation*}
\|\Delta_s\|_2\leq(L_I+c_I^*S)\|\vx-\vx'\|_2+(G+\rho_I^*)\int_0^s\|\Delta_v\|_2\d v\,,
\end{equation*}
and Gronwall's inequality~\citep[Section B.2]{app:evans2010partial} gives
\begin{equation*}
\sup_{0\leq s\leq S}\|\Delta_s\|_2\leq\Lambda_I\|\vx-\vx'\|_2\,, \qquad \Lambda_I:=(L_I+c_I^*S)e^{(G+\rho_I^*)S}\,.
\end{equation*}
At the query point $\vx=\bzero$, the condition $\|\Phi_0(t,\bzero,V_i)\|_2\leq A_I(V_i)$ in Assumption~\ref{assumption:random:field:initialization} and the linear growth bound $\|\vh_{\vy}(\vz)-\rho_{t}\vz\|_{2}\leq\|\vh_{\vy}(\bzero)\|_{2}+(G+\rho_{I}^{*})\|\vz\|_{2}$, applied to~\eqref{eq:lemma:random:field:integral}, yield
\begin{equation*}
\|\what{\mX}_{0,s,i}^{\bzero,\vy,t}\|_2\leq A_{I}(V_{i})+S\|\vh_{\vy}(\bzero)\|_{2}+\sqrt{2}\sup_{0\leq v\leq S}\|\mB_{i}^{\rm in}(v)\|_{2}+(G+\rho_{I}^{*})\int_0^s\|\what{\mX}_{0,v,i}^{\bzero,\vy,t}\|_2\d v\,,
\end{equation*}
and Gronwall's inequality again gives
\begin{equation*}
\sup_{t\in I}\sup_{0\leq s\leq S}
\|\what{\mX}_{0,s,i}^{\bzero,\vy,t}\|_2
\leq M_{i}:=\Big\{A_{I}(V_{i})+S\|\vh_{\vy}(\bzero)\|_{2}+\sqrt{2}\sup_{0\leq v\leq S}\|\mB_{i}^{\rm in}(v)\|_{2}\Big\}e^{(G+\rho_{I}^{*})S}\,,
\end{equation*}
where $M_{i}$ is $\calG$-measurable and almost surely finite. Combining the Lipschitz bound with the growth bound gives $\|\what{\mX}_{0,S,i}^{\vx,\vy,t}\|_2\leq M_i+\Lambda_I\|\vx\|_2$, and substituting into~\eqref{eq:denoiser:MC}--\eqref{eq:posterior:score:estimate} proves the two bounds with
\begin{equation*}
L_I':=\sup_{t\in I}\sigma_{t}^{-2}+c_{I}^{*}\Lambda_{I}\,, \qquad
K_I:=\sup_{t\in I}\sigma_{t}^{-2}+c_{I}^{*}\bigg(\frac{1}{m}\sum_{i=1}^{m}M_{i}+\Lambda_{I}\bigg)\,.
\end{equation*}
This proves part (ii). For part (iii), the projection $\Pi_{a}(\vv)$ is $1$-Lipschitz in $\vv$ and in the radius $a$, and the radius $a_{T}(\cdot,\vy)$ in~\eqref{eq:warm:score:radius} is Lipschitz with linear growth, so the clipped field inherits the properties of parts (i) and (ii).

For part (iv), the parametric Picard construction makes each outer solution jointly measurable in the canonical coordinates and in $\omega$. Integrating the outer coordinates gives the asserted Borel probability kernels. The total variation between two kernels is measurable as a supremum over a countable generating algebra, and $\bbW_2^2$ and the push-forward under $\calT_R$ are Borel on the corresponding spaces of probability measures. Finally, standard results for globally Lipschitz drifts with linear growth give the concluding well-posedness statement once $\calG$ is included in the outer filtrations at time zero. This completes the proof.
\end{proof}

\subsection{Regularity and growth of the true posterior score}
\label{appendix:random:field:true:score}
This subsection establishes both the Hessian regularity and an explicit growth bound for the true posterior score $\nabla\log q_{t}(\cdot\given\vy)$ at every $t\in(0,\bar t)$.

\begin{lemma}[Regularity and growth of the true posterior score]
\label{lemma:true:score:regularity}
Suppose Assumptions~\ref{assumption:semi:log:concave} and~\ref{assumption:posterior:bound:zero} hold. For each $t\in(0,\bar{t})$, where $\bar{t}=\log(1+\alpha^{-1}) / 2$ as in Proposition~\ref{proposition:appendix:RGO:Hessian},
\begin{equation}\label{eq:true:score:hessian:bound}
-\frac{1}{\sigma_{t}^{2}}\mI_{d}\preceq\nabla^{2}\log q_{t}(\vx\given\vy)\preceq\frac{\alpha}{\mu_{t}^{2}-\alpha\sigma_{t}^{2}}\mI_{d}\,, \quad \vx\in\bbR^{d}\,.
\end{equation}
Moreover,
\begin{equation}\label{eq:true:score:growth:bound:H}
\|\nabla\log q_{t}(\vx\given\vy)\|_{2}^{2}
\leq\frac{4(d+H_{\vy}^{2}+\|\vx\|_{2}^{2})}
{\sigma_{t}^{4}(\mu_{t}^{2}-\alpha\sigma_{t}^{2})^{2}}\,, \quad \vx\in\bbR^{d}\,.
\end{equation}
Consequently, on every compact interval $I\subset(0,\bar{t})$, the score is Lipschitz in $\vx$ and of linear growth, uniformly in $t\in I$.
\end{lemma}

\begin{proof}[Proof of Lemma~\ref{lemma:true:score:regularity}]
\noindent\emph{Step 1. The Hessian identity and the bounds~\eqref{eq:true:score:hessian:bound}.}
Recall from Proposition~\ref{proposition:method:score:denoiser} that
\begin{equation}\label{eq:true:score:tweedie}
\nabla\log q_{t}(\vx\given\vy)=-\frac{1}{\sigma_{t}^{2}}\vx+\frac{\mu_{t}}{\sigma_{t}^{2}}\mD(t,\vx,\vy)\,, \qquad \mD(t,\vx,\vy)=\int\vx_{0}\,p_{t}(\vx_{0}\given\vx,\vy)\d\vx_{0}\,.
\end{equation}
Combining Bayes' rule in Appendix~\ref{appendix:proof:log:concave:RGO} with~\eqref{eq:proof:lemma:tweedie:posterior:1} and~\eqref{eq:true:score:tweedie} gives
\begin{equation}\label{eq:true:score:density:grad}
\nabla_{\vx}\log p_{t}(\vx_{0}\given\vx,\vy)
=\frac{\mu_{t}\vx_{0}-\vx}{\sigma_{t}^{2}}-\nabla\log q_{t}(\vx\given\vy)
=\frac{\mu_{t}}{\sigma_{t}^{2}}\big\{\vx_{0}-\mD(t,\vx,\vy)\big\}\,.
\end{equation}
Since $t<\bar{t}$, Proposition~\ref{proposition:appendix:RGO:Hessian} shows that the posterior denoising density is strongly log-concave with
\begin{equation}\label{eq:true:score:slc}
-\nabla_{\vx_{0}}^{2}\log p_{t}(\vx_{0}\given\vx,\vy)\succeq\frac{\mu_{t}^{2}-\alpha\sigma_{t}^{2}}{\sigma_{t}^{2}}\mI_{d}\succ\bzero\,, \quad (\vx_{0},\vx)\in\bbR^{d}\times\bbR^{d}\,.
\end{equation}
Using $\nabla_{\vx}p_{t}(\vx_{0}\given\vx,\vy)=p_{t}(\vx_{0}\given\vx,\vy)\nabla_{\vx}\log p_{t}(\vx_{0}\given\vx,\vy)$, identity~\eqref{eq:true:score:density:grad}, and the centering property $\int\{\vx_{0}-\mD(t,\vx,\vy)\}p_{t}(\vx_{0}\given\vx,\vy)\d\vx_{0}=\bzero$, we obtain
\begin{align*}
\nabla_{\vx}\mD(t,\vx,\vy)
&=\frac{\mu_{t}}{\sigma_{t}^{2}}\int\big\{\vx_{0}-\mD(t,\vx,\vy)\big\}\big\{\vx_{0}-\mD(t,\vx,\vy)\big\}^{\top}p_{t}(\vx_{0}\given\vx,\vy)\d\vx_{0} \\
&=\frac{\mu_{t}}{\sigma_{t}^{2}}\cov(\mX_{0}\given\mX_{t}=\vx,\mY=\vy)\,.
\end{align*}
Differentiating~\eqref{eq:true:score:tweedie} in $\vx$ and substituting the above identity yield
\begin{equation}
\nabla^{2}\log q_{t}(\vx\given\vy)
=-\frac{1}{\sigma_{t}^{2}}\mI_{d}+\frac{\mu_{t}^{2}}{\sigma_{t}^{4}}\cov(\mX_{0}\given\mX_{t}=\vx,\mY=\vy)\,. \label{eq:true:score:hessian:identity}
\end{equation}
The covariance in~\eqref{eq:true:score:hessian:identity} is positive semidefinite, which proves the lower bound in~\eqref{eq:true:score:hessian:bound}. For the upper bound, the Brascamp--Lieb inequality~\citep{app:Brascamp1976Extensions} together with~\eqref{eq:true:score:slc} gives
\begin{equation*}
\cov(\mX_{0}\given\mX_{t}=\vx,\mY=\vy)\preceq\frac{\sigma_{t}^{2}}{\mu_{t}^{2}-\alpha\sigma_{t}^{2}}\mI_{d}\,,
\end{equation*}
and substituting this bound into~\eqref{eq:true:score:hessian:identity} proves the upper bound in~\eqref{eq:true:score:hessian:bound}:
\begin{equation*}
\nabla^{2}\log q_{t}(\vx\given\vy)\preceq\bigg\{-\frac{1}{\sigma_{t}^{2}}+\frac{\mu_{t}^{2}}{\sigma_{t}^{2}(\mu_{t}^{2}-\alpha\sigma_{t}^{2})}\bigg\}\mI_{d}=\frac{\alpha}{\mu_{t}^{2}-\alpha\sigma_{t}^{2}}\mI_{d}\,.
\end{equation*}

\noindent\emph{Step 2. The second moment of $p_{t}(\cdot\given\bzero,\vy)$.}
By the definition~\eqref{eq:posterior:conditional:density} and Bayes' rule $q_{0}(\vx_{0}\given\vy)\propto\pi_{0}(\vx_{0})\exp\{-\ell_{\vy}(\vx_{0})\}$, the gradient of the posterior denoising density in $\vx_{0}$ takes the form
\begin{equation}\label{eq:true:score:rgo:grad}
\nabla_{\vx_{0}}\log p_{t}(\vx_{0}\given\vx,\vy)
=\nabla\log q_{0}(\vx_{0}\given\vy)+\frac{\mu_{t}}{\sigma_{t}^{2}}(\vx-\mu_{t}\vx_{0})\,.
\end{equation}
Taking $\vx=\vx_{0}=\bzero$ in~\eqref{eq:true:score:rgo:grad} and applying Assumption~\ref{assumption:posterior:bound:zero} give
\begin{equation}\label{eq:true:score:origin:grad}
\|\nabla_{\vx_{0}}\log p_{t}(\bzero\given\bzero,\vy)\|_{2}=\|\nabla\log q_{0}(\bzero\given\vy)\|_{2}\leq H_{\vy}\,.
\end{equation}
Owing to~\eqref{eq:true:score:slc}, the potential $\vx_{0}\mapsto-\log p_{t}(\vx_{0}\given\bzero,\vy)$ is $(\mu_{t}^{2}-\alpha\sigma_{t}^{2})\sigma_{t}^{-2}$-strongly convex, and the gradient monotonicity of strongly convex functions gives
\begin{equation*}
\big\langle\nabla_{\vx_{0}}\log p_{t}(\bzero\given\bzero,\vy)-\nabla_{\vx_{0}}\log p_{t}(\vx_{0}\given\bzero,\vy),\,\vx_{0}\big\rangle\geq\frac{\mu_{t}^{2}-\alpha\sigma_{t}^{2}}{\sigma_{t}^{2}}\|\vx_{0}\|_{2}^{2}\,, \quad \vx_{0}\in\bbR^{d}\,.
\end{equation*}
Rearranging this inequality, bounding the term $\langle\nabla_{\vx_{0}}\log p_{t}(\bzero\given\bzero,\vy),\vx_{0}\rangle$ via the Cauchy--Schwarz inequality and~\eqref{eq:true:score:origin:grad}, and applying Young's inequality 
$$H_{\vy}\|\vx_{0}\|_{2}\leq\frac{\mu_{t}^{2}-\alpha\sigma_{t}^{2}}{2\sigma_{t}^{2}}\|\vx_{0}\|_{2}^{2}+\frac{\sigma_{t}^{2}H_{\vy}^{2}}{2(\mu_{t}^{2}-\alpha\sigma_{t}^{2})} \, ,$$
we find, pointwise in $\vx_{0}$,
\begin{align}
-\big\langle\vx_{0},\nabla_{\vx_{0}}\log p_{t}(\vx_{0}\given\bzero,\vy)\big\rangle
&\geq\frac{\mu_{t}^{2}-\alpha\sigma_{t}^{2}}{\sigma_{t}^{2}}\|\vx_{0}\|_{2}^{2}-H_{\vy}\|\vx_{0}\|_{2} \nonumber \\
&\geq\frac{\mu_{t}^{2}-\alpha\sigma_{t}^{2}}{2\sigma_{t}^{2}}\|\vx_{0}\|_{2}^{2}-\frac{\sigma_{t}^{2}H_{\vy}^{2}}{2(\mu_{t}^{2}-\alpha\sigma_{t}^{2})}\,. \label{eq:true:score:pointwise:lb}
\end{align}
On the other hand, integrating the identity $$\mathrm{div}\{\vx_{0}p_{t}(\vx_{0}\given\bzero,\vy)\}=d\,p_{t}(\vx_{0}\given\bzero,\vy)+\langle\vx_{0},\nabla_{\vx_{0}}\log p_{t}(\vx_{0}\given\bzero,\vy)\rangle p_{t}(\vx_{0}\given\bzero,\vy)$$ over $\bbR^{d}$, where the left-hand side vanishes by the Gauss--Green theorem~\citep[Theorem 1 in Section C.2]{app:evans2010partial} and the Gaussian decay of $p_{t}(\cdot\given\bzero,\vy)$, we arrive at
\begin{equation}\label{eq:true:score:gauss:green}
-\int\big\langle\vx_{0},\nabla_{\vx_{0}}\log p_{t}(\vx_{0}\given\bzero,\vy)\big\rangle p_{t}(\vx_{0}\given\bzero,\vy)\d\vx_{0}=d\,.
\end{equation}
Integrating~\eqref{eq:true:score:pointwise:lb} against $p_{t}(\cdot\given\bzero,\vy)$ and invoking~\eqref{eq:true:score:gauss:green} then yield
\begin{equation}\label{eq:true:score:origin:moment}
\int\|\vx_{0}\|_{2}^{2}\,p_{t}(\vx_{0}\given\bzero,\vy)\d\vx_{0}
\leq\frac{2d\sigma_{t}^{2}}{\mu_{t}^{2}-\alpha\sigma_{t}^{2}}
+\frac{\sigma_{t}^{4}H_{\vy}^{2}}{(\mu_{t}^{2}-\alpha\sigma_{t}^{2})^{2}}\,.
\end{equation}

\noindent\emph{Step 3. The growth bound~\eqref{eq:true:score:growth:bound:H}.}
Both endpoints of~\eqref{eq:true:score:hessian:bound} are bounded in absolute value by $\{\sigma_{t}^{2}(\mu_{t}^{2}-\alpha\sigma_{t}^{2})\}^{-1}$, owing to $\mu_{t}^{2}-\alpha\sigma_{t}^{2}\leq1$ for the left endpoint and to $\alpha\sigma_{t}^{2}<\mu_{t}^{2}\leq1$ for the right endpoint. Hence $\|\nabla^{2}\log q_{t}(\vz\given\vy)\|_{\op}\leq\{\sigma_{t}^{2}(\mu_{t}^{2}-\alpha\sigma_{t}^{2})\}^{-1}$ for each $\vz\in\bbR^{d}$, and integrating the Hessian along the segment from $\bzero$ to $\vx$ gives
\begin{align}\label{eq:true:score:increment}
&\|\nabla\log q_{t}(\vx\given\vy)-\nabla\log q_{t}(\bzero\given\vy)\|_{2}\notag\\
&~~~~~~~~~~~~~~~~~=\bigg\|\int_{0}^{1}\nabla^{2}\log q_{t}(s\vx\given\vy)\vx\d s\bigg\|_{2}
\leq\frac{\|\vx\|_{2}}{\sigma_{t}^{2}(\mu_{t}^{2}-\alpha\sigma_{t}^{2})}\,.
\end{align}
Taking $\vx=\bzero$ in~\eqref{eq:true:score:tweedie}, applying Jensen's inequality to $\mD(t,\bzero,\vy)$, and invoking~\eqref{eq:true:score:origin:moment} give
\begin{equation}\label{eq:true:score:origin:score}
\|\nabla\log q_{t}(\bzero\given\vy)\|_{2}^{2}
=\frac{\mu_{t}^{2}}{\sigma_{t}^{4}}\|\mD(t,\bzero,\vy)\|_{2}^{2}
\leq\frac{2d\mu_{t}^{2}}{\sigma_{t}^{2}(\mu_{t}^{2}-\alpha\sigma_{t}^{2})}
+\frac{\mu_{t}^{2}H_{\vy}^{2}}{(\mu_{t}^{2}-\alpha\sigma_{t}^{2})^{2}}\,.
\end{equation}
Combining~\eqref{eq:true:score:increment} and~\eqref{eq:true:score:origin:score}, we conclude that
\begin{align*}
\|\nabla\log q_{t}(\vx\given\vy)\|_{2}^{2}
&\leq2\|\nabla\log q_{t}(\bzero\given\vy)\|_{2}^{2}+2\|\nabla\log q_{t}(\vx\given\vy)-\nabla\log q_{t}(\bzero\given\vy)\|_{2}^{2} \\
&\leq\frac{4d\mu_{t}^{2}}{\sigma_{t}^{2}(\mu_{t}^{2}-\alpha\sigma_{t}^{2})}
+\frac{2\mu_{t}^{2}H_{\vy}^{2}}{(\mu_{t}^{2}-\alpha\sigma_{t}^{2})^{2}}
+\frac{2\|\vx\|_{2}^{2}}{\sigma_{t}^{4}(\mu_{t}^{2}-\alpha\sigma_{t}^{2})^{2}} \\
&\leq\frac{4d+2H_{\vy}^{2}+2\|\vx\|_{2}^{2}}{\sigma_{t}^{4}(\mu_{t}^{2}-\alpha\sigma_{t}^{2})^{2}}
\leq\frac{4(d+H_{\vy}^{2}+\|\vx\|_{2}^{2})}{\sigma_{t}^{4}(\mu_{t}^{2}-\alpha\sigma_{t}^{2})^{2}}\,,
\end{align*}
where the first inequality uses $\|\va+\vb\|_{2}^{2}\leq2\|\va\|_{2}^{2}+2\|\vb\|_{2}^{2}$, and the third inequality uses $\mu_{t}^{2}\sigma_{t}^{2}(\mu_{t}^{2}-\alpha\sigma_{t}^{2})\leq1$ for the first term and $\mu_{t}^{2}\sigma_{t}^{4}\leq1$ for the second term.
\end{proof}

\section{Proof of the Error Decomposition}
\label{section:proof:error:decomposition}
In this section, we aim to provide a proof of the error decomposition of the posterior estimation (Proposition~\ref{proposition:error:decomposition}). Recall the time-reversal process~\eqref{eq:reversal}:
\begin{equation*}
\d\bar{\mX}_{t}^{\vy}=\big\{\bar{\mX}_{t}^{\vy}+2\nabla_{\vx}\log\bar{q}_{t}(\bar{\mX}_{t}^{\vy}\given\vy)\big\}\dt+\sqrt{2}\d\mB_{t}\,, \quad \bar{\mX}_{0}^{\vy}\sim q_{T}(\cdot\given\vy)\,, ~ t\in(0,T-T_0)\,,
\end{equation*}
where $T_{0}\in(0,T)$ is the early-stopping time, and $\bar{q}_{t}(\cdot\given\vy)\coloneq  q_{T-t}(\cdot\given\vy)$. In practical applications, the posterior score $\nabla_{\vx}\log\bar{q}_{t}(\cdot\given\vy)$ is intractable, and we only have an estimator $\hat{\vs}_{m}^{S}(T-t,\cdot,\vy)\approx\nabla\log\bar{q}_{t}(\cdot\given\vy)$~\eqref{eq:posterior:score:estimate}. The time-reversal process with this estimated score is given as
\begin{equation*}
\d\what{\mX}_{t}^{\vy}=\big\{\what{\mX}_{t}^{\vy}+2\hat{\vs}_{m}^{S}(T-t,\what{\mX}_{t}^{\vy},\vy)\big\}\dt+\sqrt{2}\d\mB_{t}\,, \quad \what{\mX}_{0}^{\vy}:=\what{\mX}_{T,U}^{\vy}\sim\hat{q}_{T}^{U}(\cdot\given\vy)\,, ~ t\in(0,T-T_{0})\,.
\end{equation*}
Denote by $\hat{q}_{t}(\cdot\given\vy)$ the conditional density of $\what{\mX}^{\vy}_{t}$ given $\calG$ for each $t\in(0,T-T_{0})$. 

\subsection{Error of the early-stopping}
\label{section:proof:error:decomposition:early:stopping}
\begin{lemma}\label{lemma:section:proof:early:stopping:0}
Suppose Assumption~\ref{assumption:prior:subGaussian} holds. Let $T_{0}\in (0,\min\{T,1/2\})$ and $\delta\in(0,1)$ small enough that $R\geq 1$, where $R^{2}=(4\mu_{T_{0}}^{2}V_{\SG}^{2}+16\sigma_{T_{0}}^{2})\log(\kappa_{\vy}\delta^{-1})$. Then
\begin{equation*}
\bbE\big[\bbW_{2}^{2}\big\{q_{0}(\cdot\given\vy),\calM(\mu_{T_{0}}^{-1})\sharp\hat{q}_{T-T_{0}}^{R}(\cdot\given\vy)\big\}\big]\leq C\bigg[\frac{\sigma_{T_{0}}^{2}}{\mu_{T_{0}}^{2}}+\delta\bigg\{1+\log\bigg(\frac{\kappa_{\vy}}{\delta}\bigg)\bigg\}\bigg]\,,
\end{equation*}
provided that
\begin{equation*}
\bbE\big\{\|q_{T_{0}}(\cdot\given\vy)-\hat{q}_{T-T_{0}}(\cdot\given\vy)\|_{\tv}^{2}\big\}\leq\delta^{2}\,.
\end{equation*}
Here $C$ is a constant only depending on $d$, $V_{\SG}$, and $C_{\SG}$.
\end{lemma}

\begin{proof}[Proof of Lemma~\ref{lemma:section:proof:early:stopping:0}]
Applying Lemmas~\ref{lemma:section:proof:early:stopping:1}--\ref{lemma:section:proof:early:stopping:3} with Lemma \ref{lemma:section:proof:early:stopping:3} conditional on $\calG$, we have
\begin{align*}
&\bbW_{2}^{2}\big\{q_{0}(\cdot\given\vy),\calM(\mu_{T_{0}}^{-1})\sharp\hat{q}_{T-T_{0}}^{R}(\cdot\given\vy)\big\} \\
&~~\leq 3\bbW_{2}^{2}\big\{q_{0}(\cdot\given\vy),\calM(\mu_{T_{0}}^{-1})\sharp q_{T_{0}}(\cdot\given\vy)\big\}+3\bbW_{2}^{2}\big\{\calM(\mu_{T_{0}}^{-1})\sharp q_{T_{0}}(\cdot\given\vy),\calM(\mu_{T_{0}}^{-1})\sharp q_{T_{0}}^{R}(\cdot\given\vy)\big\} \\
&~~\quad+3\bbW_{2}^{2}\big\{\calM(\mu_{T_{0}}^{-1})\sharp q_{T_{0}}^{R}(\cdot\given\vy),\calM(\mu_{T_{0}}^{-1})\sharp\hat{q}_{T-T_{0}}^{R}(\cdot\given\vy)\big\} \\
&~~\leq \frac{3d\sigma_{T_{0}}^{2}}{\mu_{T_{0}}^{2}}+\frac{3D}{\mu_{T_{0}}^{2}}\kappa_{\vy}\exp\bigg(-\frac{R^{2}}{4\mu_{T_{0}}^{2}V_{\SG}^{2}+16\sigma_{T_{0}}^{2}}\bigg)+\frac{12R^{2}}{\mu_{T_{0}}^{2}}\|q_{T_{0}}(\cdot\given\vy)-\hat{q}_{T-T_{0}}(\cdot\given\vy)\|_{\tv} \\
&~~\leq C\bigg\{\frac{\sigma_{T_{0}}^{2}}{\mu_{T_{0}}^{2}}+\kappa_{\vy}\exp\bigg(-\frac{R^{2}}{4\mu_{T_{0}}^{2}V_{\SG}^{2}+16\sigma_{T_{0}}^{2}}\bigg)+R^{2}\|q_{T_{0}}(\cdot\given\vy)-\hat{q}_{T-T_{0}}(\cdot\given\vy)\|_{\tv}\bigg\}\,,
\end{align*}
where the first inequality is owing to the triangular inequality of 2-Wasserstein distance, and the last inequality used the fact that $\mu_{T_{0}}^{-1}\leq e$ for $T_{0}\leq\frac{1}{2}$. Here the constant $C$ only depends on $d$, $V_{\SG}$, and $C_{\SG}$. Taking expectations over $\calG$ and by Cauchy--Schwarz inequality gives
\begin{equation*}
\bbE\big\{\|q_{T_{0}}(\cdot\given\vy)-\hat{q}_{T-T_{0}}(\cdot\given\vy)\|_{\tv}\big\}\leq\bbE^{1/2}\big\{\|q_{T_{0}}(\cdot\given\vy)-\hat{q}_{T-T_{0}}(\cdot\given\vy)\|_{\tv}^{2}\big\}\leq\delta\,.
\end{equation*}
The choice of $R$ makes the exponential term equal to $\delta$, and gives $R^{2}\leq(4V_{\SG}^{2}+16)\log(\kappa_{\vy}\delta^{-1})$ owing to $\mu_{T_{0}},\sigma_{T_{0}}\leq1$. Combining the three terms completes the proof.
\end{proof}

\begin{lemma}\label{lemma:section:proof:early:stopping:1}
For each $T_{0}>0$, it follows that 
\begin{equation*}
\bbW_{2}^{2}\big\{q_{0}(\cdot\given\vy),\calM(\mu_{T_{0}}^{-1})\sharp q_{T_{0}}(\cdot\given\vy)\big\}\leq \frac{d\sigma_{T_{0}}^{2}}{\mu_{T_{0}}^{2}}\,.
\end{equation*}
\end{lemma}

\begin{proof}[Proof of Lemma~\ref{lemma:section:proof:early:stopping:1}]
We first produce a coupling of $q_{0}(\cdot\given\vy)$ and $\calM(\mu_{T_{0}}^{-1})\sharp q_{T_{0}}(\cdot\given\vy)$ as follows. Let $\mX_{0}^{\vy}\sim q_{0}(\cdot\given\vy)$, and let $\vepsilon\sim N(\bzero,\mI_{d})$ be independent of $\mX_{0}^{\vy}$. Then 
\begin{equation*}
\mZ_{T_{0}}^{\vy}\coloneq \mX_{0}^{\vy}+\frac{\sigma_{T_{0}}}{\mu_{T_{0}}}\vepsilon\sim\calM(\mu_{T_{0}}^{-1})\sharp q_{T_{0}}(\cdot\given\vy)\,.
\end{equation*}
As a consequence, 
\begin{equation*}
\bbW_{2}^{2}\big\{q_{0}(\cdot\given\vy),\calM(\mu_{T_{0}}^{-1})\sharp q_{T_{0}}(\cdot\given\vy)\big\}\leq \bbE\big(\|\mX_{0}^{\vy}-\mZ_{T_{0}}^{\vy}\|_{2}^{2}\big)=\frac{\sigma_{T_{0}}^{2}}{\mu_{T_{0}}^{2}}\bbE\big(\|\vepsilon\|_{2}^{2}\big)=\frac{d\sigma_{T_{0}}^{2}}{\mu_{T_{0}}^{2}}\,,
\end{equation*}
which completes the proof.
\end{proof}

\begin{lemma}\label{lemma:section:proof:early:stopping:2}
Suppose Assumption~\ref{assumption:prior:subGaussian} holds. For each $T_{0}>0$ and $R\geq0$, it follows that
\begin{equation*}
\bbW_{2}^{2}\big\{\calM(\mu_{T_{0}}^{-1})\sharp q_{T_{0}}(\cdot\given\vy),\calM(\mu_{T_{0}}^{-1})\sharp q_{T_{0}}^{R}(\cdot\given\vy)\big\}\leq \frac{D}{\mu_{T_{0}}^{2}}\kappa_{\vy}\exp\bigg(-\frac{R^{2}}{4\mu_{T_{0}}^{2}V_{\SG}^{2}+16\sigma_{T_{0}}^{2}}\bigg)\,,
\end{equation*}
where $D$ is a constant only depending on $d$, $V_{\SG}$, and $C_{\SG}$.
\end{lemma}

\begin{proof}[Proof of Lemma~\ref{lemma:section:proof:early:stopping:2}]
Let $\mX_{T_{0}}^{\vy}\sim q_{T_{0}}(\cdot\given\vy)$. Then $(\mu_{T_{0}}^{-1}\mX_{T_{0}}^{\vy},\mu_{T_{0}}^{-1}\mX_{T_{0}}^{\vy}\bbone\{\|\mX_{T_{0}}^{\vy}\|_{2}\leq R\})$ is a coupling of $\calM(\mu_{T_{0}}^{-1})\sharp q_{T_{0}}(\cdot\given\vy)$ and $\calM(\mu_{T_{0}}^{-1})\sharp q^{R}_{T_{0}}(\cdot\given\vy)$. Therefore, 
\begin{align*}
&\bbW_{2}^{2}\big\{\calM(\mu_{T_{0}}^{-1})\sharp q_{T_{0}}(\cdot\given\vy),\calM(\mu_{T_{0}}^{-1})\sharp q^{R}_{T_{0}}(\cdot\given\vy)\big\} \\
&~~~~~~\leq\bbE\big(\|\mu_{T_{0}}^{-1}\mX_{T_{0}}^{\vy}-\mu_{T_{0}}^{-1}\mX_{T_{0}}^{\vy}\bbone\{\|\mX_{T_{0}}^{\vy}\|_{2}\leq R\}\|_{2}^{2}\big) \\
&~~~~~~=\frac{1}{\mu_{T_{0}}^{2}}\int\|\vx-\vx\bbone\{\|\vx\|_{2}\leq R\}\|_{2}^{2}q_{T_{0}}(\vx\given\vy)\d\vx \\
&~~~~~~=\frac{1}{\mu_{T_{0}}^{2}}\int\|\vx\|_{2}^{2}\bbone\{\|\vx\|_{2}> R\}q_{T_{0}}(\vx\given\vy)\d\vx \\
&~~~~~~\leq\frac{1}{\mu_{T_{0}}^{2}}\bbE^{1/2}\big(\|\mX_{T_{0}}^{\vy}\|_{2}^{4}\big)\bbP^{1/2}\{\|\mX_{T_{0}}^{\vy}\|_{2}> R\} \leq \frac{D}{\mu_{T_{0}}^{2}}\kappa_{\vy}\exp\bigg(-\frac{R^{2}}{4\mu_{T_{0}}^{2}V_{\SG}^{2}+16\sigma_{T_{0}}^{2}}\bigg)\,,
\end{align*}
where the first inequality holds from the definition of 2-Wasserstein distance, the second inequality follows from Cauchy-Schwarz inequality, and the last inequality is due to Lemmas~\ref{lemma:fourth:moment:t} and~\ref{lemma:tail:proba:T}. This completes the proof.
\end{proof}

\begin{lemma}\label{lemma:section:proof:early:stopping:3}
For each $T_{0}>0$ and $R\geq 1$, it follows almost surely that  
\begin{equation*}
\bbW_{2}^{2}\big\{\calM(\mu_{T_{0}}^{-1})\sharp q_{T_{0}}^{R}(\cdot\given\vy),\calM(\mu_{T_{0}}^{-1})\sharp\hat{q}_{T-T_{0}}^{R}(\cdot\given\vy)\big\}\leq
\frac{4R^{2}}{\mu_{T_{0}}^{2}}\|q_{T_{0}}(\cdot\given\vy)-\hat{q}_{T-T_{0}}(\cdot\given\vy)\|_{\tv}\,.
\end{equation*}
\end{lemma}

\begin{proof}[Proof of Lemma~\ref{lemma:section:proof:early:stopping:3}]
Let $\mX_{T_{0}}^{\vy,R}\sim q_{T_{0}}^{R}(\cdot\given\vy)$ and $\what{\mX}_{T-T_{0}}^{\vy,R}\sim\hat{q}_{T-T_{0}}^{R}(\cdot\given\vy)$ be optimal coupled conditionally on $\calG$. This means
\begin{equation}\label{eq:lemma:section:proof:early:stopping:3:1}
\bbW_{2}^{2}\big\{q_{T_{0}}^{R}(\cdot\given\vy),\hat{q}_{T-T_{0}}^{R}(\cdot\given\vy)\big\}=\bbE\big(\|\mX_{T_{0}}^{\vy,R}-\what{\mX}_{T-T_{0}}^{\vy,R}\|_{2}^{2}\big)\,.
\end{equation}
It is apparent that $\mu_{T_{0}}^{-1}\mX_{T_{0}}^{\vy,R}\sim\calM(\mu_{T_{0}}^{-1})\sharp q_{T_{0}}^{R}(\cdot\given\vy)$ and $\mu_{T_{0}}^{-1}\what{\mX}_{T-T_{0}}^{\vy,R}\sim\calM(\mu_{T_{0}}^{-1})\sharp\hat{q}_{T-T_{0}}^{R}(\cdot\given\vy)$. Hence,
\begin{align}
&\bbW_{2}^{2}\big\{\calM(\mu_{T_{0}}^{-1})\sharp q_{T_{0}}^{R}(\cdot\given\vy),\calM(\mu_{T_{0}}^{-1})\sharp\hat{q}_{T-T_{0}}^{R}(\cdot\given\vy)\big\} \nonumber \\
&~~~~~~\leq\bbE\big(\|\mu_{T_{0}}^{-1}\mX_{T_{0}}^{\vy,R}-\mu_{T_{0}}^{-1}\what{\mX}_{T-T_{0}}^{\vy,R}\|_{2}^{2}\big)=\frac{1}{\mu_{T_{0}}^{2}}\bbW_{2}^{2}\big\{q_{T_{0}}^{R}(\cdot\given\vy),\hat{q}_{T-T_{0}}^{R}(\cdot\given\vy)\big\}\,, \label{eq:lemma:section:proof:early:stopping:3:2}
\end{align}
where the equality holds from~\eqref{eq:lemma:section:proof:early:stopping:3:1}. On the other hand,
\begin{align}
&\bbW_{2}^{2}\big\{q_{T_{0}}^{R}(\cdot\given\vy),\hat{q}_{T-T_{0}}^{R}(\cdot\given\vy)\big\}
\leq4R^{2}\|q_{T_{0}}^{R}(\cdot\given\vy)-\hat{q}_{T-T_{0}}^{R}(\cdot\given\vy)\|_{\tv} \nonumber \\
&~~~~~~~~~~~~~~\leq 4R^{2}\|q_{T_{0}}(\cdot\given\vy)-\hat{q}_{T-T_{0}}(\cdot\given\vy)\|_{\tv}\,, \label{eq:lemma:section:proof:early:stopping:3:3}
\end{align}
where the first inequality is due to~\citet[Theorem 6.15]{app:Villani2009Optimal}, and the second inequality follows from the data processing inequality~\citep[Proposition 2.2.13]{app:Duchi2025Statistics}. Combining~\eqref{eq:lemma:section:proof:early:stopping:3:2} and~\eqref{eq:lemma:section:proof:early:stopping:3:3} completes the proof.
\end{proof}

\subsection{Error decomposition for TV-distance}
\label{section:proof:error:decomposition:others}

The proof in this subsection compares the path laws of two SDEs whose drifts differ by the score estimation error. We use the following standard entropy comparison, which controls the path-space KL divergence by the squared discrepancy between the two drifts.

\begin{lemma}
\label{lemma:entropy:girsanov}
Let $\mZ_{s}$ and $\bar{\mZ}_{s}$ be two stochastic processes defined, respectively, by the following SDEs:
\begin{equation}\label{eq:lemma:entropy:girsanov:1}
\begin{aligned}
\d\mZ_{s}&=\vb(s,\mZ_{s})\ds+\sqrt{2}\d\mB_{s}\,, \quad \mZ_{0}\sim\mu_{0}\,, \\
\d\bar{\mZ}_{s}&=\bar{\vb}(s,\bar{\mZ}_{s})\ds+\sqrt{2}\d\bar{\mB}_{s}\,, \quad \bar{\mZ}_{0}\sim\mu_{0}\,,
\end{aligned}
\end{equation}
where $(\mB_{s})_{s\geq0}$ and $(\bar{\mB}_{s})_{s\geq0}$ are $d$-dimensional Brownian motions. Denote by $\mu_{s}$ the distribution of $\mZ_{s}$, and denote by $\nu$ and $\bar{\nu}$ the laws of $(\mZ_{s})_{s\in[0,L]}$ and $(\bar{\mZ}_{s})_{s\in[0,L]}$ on $C([0,L];\bbR^{d})$. Assume that both drifts are Borel in time and locally Lipschitz in space uniformly over $s\in[0,L]$, and that, for some $C_{L}<\infty$,
\begin{equation*}
\|\vb(s,\vx)\|_{2}+\|\bar{\vb}(s,\vx)\|_{2}\leq C_{L}(1+\|\vx\|_{2})
\qquad\text{for all }(s,\vx)\in[0,L]\times\bbR^{d}.
\end{equation*}
In particular, both equations admit non-explosive strong solutions on $[0,L]$ and are well posed. Then it holds that
\begin{equation}\label{eq:entropy:girsanov:general}
\kl(\nu\,\|\,\bar{\nu})
\leq\frac{1}{4}\int_{0}^{L}\bbE_{\mZ_{s}\sim\mu_{s}}\big\{\|\vb(s,\mZ_{s})-\bar{\vb}(s,\mZ_{s})\|_{2}^{2}\big\}\ds\,,
\end{equation}
where $L>0$ is the time horizon, and the right-hand side is allowed to be infinite.
\end{lemma}

\begin{proof}[Proof of Lemma~\ref{lemma:entropy:girsanov}]
Uniform linear growth and continuity of the non-explosive paths imply that the time-integrated squared drift discrepancy is finite almost surely under both path laws, while local Lipschitz continuity gives the required well-posedness. We may therefore apply the entropy estimate of~\citet[Lemma~4.4(i) and Remark~4.5]{app:Lacker2023Hierarchies} with $P^{1}=\nu$, $P^{2}=\bar{\nu}$, and diffusion matrix $\sigma=\sqrt{2}I_{d}$. Since the initial laws coincide, it follows that
\begin{align*}
\kl(\nu\,\|\,\bar{\nu})
&\leq\frac{1}{2}\bbE_{\nu}\int_{0}^{L}
\big\|(\sqrt{2}I_{d})^{-1}\big(\vb(s,\mZ_{s})-\bar{\vb}(s,\mZ_{s})\big)\big\|_{2}^{2}\ds \\
&=\frac{1}{4}\int_{0}^{L}\bbE_{\mZ_{s}\sim\mu_{s}}
\big\{\|\vb(s,\mZ_{s})-\bar{\vb}(s,\mZ_{s})\|_{2}^{2}\big\}\ds\,,
\end{align*}
where the last identity follows from Tonelli's theorem.
\end{proof}

The following lemma indicates that the TV-error in time $T-T_{0}$ consists of two parts: the posterior score estimation error and the warm-start error.
\begin{lemma}\label{lemma:appendix:decomposition:score:error}
Suppose Assumptions~\ref{assumption:semi:log:concave},~\ref{assumption:posterior:bound:zero},~\ref{assumption:Lipschitz:prior:likelihood}, and~\ref{assumption:random:field:initialization} hold. For each $0<T_{0}<T<\log(1+\alpha^{-1}) / 2$, the following inequality holds:
\begin{align*}
&\bbE\big\{\|q_{T_{0}}(\cdot\given\vy)-\hat{q}_{T-T_{0}}(\cdot\given\vy)\|_{\tv}^{2}\big\} \\
&~~~~~~\leq\underbrace{\int_{T_{0}}^{T}\bbE\big\{\|\nabla\log q_{t}(\mX_{t}^{\vy}\given\vy)-\hat{\vs}_{m}^{S}(t,\mX_{t}^{\vy},\vy)\|_{2}^{2}\big\}\dt}_{\text{posterior score estimation}}+\underbrace{2\bbE\big\{\|q_{T}(\cdot\given\vy)-\hat{q}_{T}^{U}(\cdot\given\vy)\|_{\tv}^{2}\big\}}_{\text{warm-start}}\,,
\end{align*}
where the expectation is taken with respect to $\mX_{t}^{\vy}\sim q_{t}(\cdot\given\vy)$ and over $\calG$.
\end{lemma}

\begin{proof}[Proof of Lemma~\ref{lemma:appendix:decomposition:score:error}]
Condition on $\calG$ and fix a realization outside the null set of Lemma~\ref{lemma:random:score:field}, so that all hatted laws below are the corresponding conditional laws. Denote by $\nu$ and $\hat{\nu}$ the laws on $C([0,T-T_{0}];\bbR^{d})$ of the exact and the approximate time-reversal processes $(\bar{\mX}_{t}^{\vy})_{t\in[0,T-T_{0}]}$ and $(\what{\mX}_{t}^{\vy})_{t\in[0,T-T_{0}]}$, and denote by $\tilde{\nu}$ the law of the approximate process initialized at $q_{T}(\cdot\given\vy)$ in place of $\hat{q}_{T}^{U}(\cdot\given\vy)$. The time-$(T-T_{0})$ marginals of $\nu$ and $\hat{\nu}$ are $q_{T_{0}}(\cdot\given\vy)$ and $\hat{q}_{T-T_{0}}(\cdot\given\vy)$. Moreover, the Brownian motion driving the time-reversal stage is independent of $\calG$ and of the warm-start output by the construction of Appendix~\ref{appendix:random:field}, so $\tilde{\nu}$ and $\hat{\nu}$ are the compositions of their respective initial distributions with a common Markov transition kernel. The data processing inequality~\citep[Proposition 2.2.13]{app:Duchi2025Statistics} thus implies
\begin{equation}\label{eq:path:space:dpi}
\|q_{T_{0}}(\cdot\given\vy)-\hat{q}_{T-T_{0}}(\cdot\given\vy)\|_{\tv}\leq\|\nu-\hat{\nu}\|_{\tv}\,, \qquad
\|\tilde{\nu}-\hat{\nu}\|_{\tv}\leq\|q_{T}(\cdot\given\vy)-\hat{q}_{T}^{U}(\cdot\given\vy)\|_{\tv}\,.
\end{equation}
Combining~\eqref{eq:path:space:dpi} with the triangular inequality and Pinsker's inequality~\citep[Lemma 2.5]{app:Tsybakov2009Introduction} yields
\begin{align}
\|q_{T_{0}}(\cdot\given\vy)-\hat{q}_{T-T_{0}}(\cdot\given\vy)\|_{\tv}^{2}
&\leq2\|\nu-\tilde{\nu}\|_{\tv}^{2}+2\|\tilde{\nu}-\hat{\nu}\|_{\tv}^{2} \nonumber \\
&\leq\kl(\nu\,\|\,\tilde{\nu})+2\|q_{T}(\cdot\given\vy)-\hat{q}_{T}^{U}(\cdot\given\vy)\|_{\tv}^{2}\,. \label{eq:path:space:decomposition}
\end{align}
We next bound $\kl(\nu\,\|\,\tilde{\nu})$ by Lemma~\ref{lemma:entropy:girsanov}. The two path laws share the initial distribution $q_{T}(\cdot\given\vy)$ and the diffusion coefficient $\sqrt{2}$. For $t=T-s$ ranging over the compact interval $[T_{0},T]$, both drifts are Lipschitz in $\vx$ and of linear growth, owing to Lemma~\ref{lemma:true:score:regularity} for the true score and to Lemma~\ref{lemma:random:score:field} for the estimated field. Lemma~\ref{lemma:entropy:girsanov} then gives
\begin{align}
\kl(\nu\,\|\,\tilde{\nu})
&\leq\frac{1}{4}\int_{0}^{T-T_{0}}\bbE\big[\big\|2\big\{\nabla\log\bar{q}_{s}(\bar{\mX}_{s}^{\vy}\given\vy)-\hat{\vs}_{m}^{S}(T-s,\bar{\mX}_{s}^{\vy},\vy)\big\}\big\|_{2}^{2}\big]\ds \nonumber \\
&=\int_{T_{0}}^{T}\bbE\big\{\|\nabla\log q_{t}(\mX_{t}^{\vy}\given\vy)-\hat{\vs}_{m}^{S}(t,\mX_{t}^{\vy},\vy)\|_{2}^{2}\big\}\dt\,, \label{eq:path:space:girsanov}
\end{align}
where the expectations are conditional on $\calG$ and taken with respect to $\bar{\mX}_{s}^{\vy}\sim\bar{q}_{s}(\cdot\given\vy)$, the integrand is the squared drift difference of the two SDEs as required by~\eqref{eq:entropy:girsanov:general}, and the equality substitutes $t=T-s$ together with $\bar{q}_{s}(\cdot\given\vy)=q_{T-s}(\cdot\given\vy)$. Taking expectations over $\calG$ on both sides of~\eqref{eq:path:space:decomposition} and~\eqref{eq:path:space:girsanov}, and applying Tonelli's theorem, complete the proof.
\end{proof}

\subsection{Proof of Proposition~\ref{proposition:error:decomposition}}
\begin{proof}[Proof of Proposition~\ref{proposition:error:decomposition}]
Combining Lemmas~\ref{lemma:section:proof:early:stopping:0} and~\ref{lemma:appendix:decomposition:score:error} completes the proof.
\end{proof}

\section{Proof of the Posterior Score Estimation Bound}
\label{section:proof:posterior:score}
In the section, we provide a proof of Proposition~\ref{proposition:posterior:score}, which propose error bounds for the posterior score estimation. 

\subsection{Posterior score estimation error decomposition}
We begin by introducing an error decomposition as the following lemma, which divides the posterior score estimation error into the error of the Langevin dynamics with estimated error~\eqref{eq:RGO:Langevin:score}, and the error of Monte Carlo approximation~\eqref{eq:denoiser:MC}. The proof of this lemma is inspired by~\citet[Proposition 1]{app:He2024Zeroth}.

\begin{lemma}\label{lemma:section:proof:posterior:score:1}
Suppose Assumptions~\ref{assumption:prior:subGaussian},~\ref{assumption:Lipschitz:prior:likelihood}, and~\ref{assumption:random:field:initialization} hold. For each $t\in(0,T]$,
\begin{align*}
&\bbE\big\{\|\nabla\log q_{t}(\mX_{t}^{\vy} \given \vy)-\hat{\vs}_{m}^{S}(t,\mX_{t}^{\vy},\vy)\|_{2}^{2}\big\} \\
&~~~~~~~~~~~~\leq\frac{2\mu_{t}^{2}}{\sigma_{t}^{4}}\bigg(\underbrace{\bbE\big[\bbW_{2}^{2}\big\{\hat{p}_{t}^{S}(\cdot \given \mX_{t}^{\vy},\vy),p_{t}(\cdot \given \mX_{t}^{\vy},\vy)\big\}\big]}_{\text{error of Langevin dynamics}}+\underbrace{\frac{\kappa_{\vy}C_{\SG}V_{\SG}^{2}}{m}}_{\text{Monte Carlo}}\bigg)\,,
\end{align*}
where the expectation on the left-hand side averages over $\calG$ and $\mX_{t}^{\vy}\sim q_{t}(\cdot\given\vy)$, which are independent by the construction of Appendix~\ref{appendix:random:field}, and the expectation on the right-hand side is taken with respect to $\mX_{t}^{\vy}$, as the kernels $\hat{p}_{t}^{S}$ and $p_{t}$ are deterministic.
\end{lemma}

\begin{proof}[Proof of Lemma~\ref{lemma:section:proof:posterior:score:1}]
It follows from~\eqref{eq:posterior:score} that for each $t>0$,
\begin{equation}\label{eq:lemma:section:proof:posterior:score:1:0}
\|\nabla\log q_{t}(\vx \given \vy)-\hat{\vs}_{m}^{S}(t,\vx,\vy)\|_{2}^{2}=\frac{\mu_{t}^{2}}{\sigma_{t}^{4}}\|\mD(t,\vx,\vy)-\what{\mD}_{m}^{S}(t,\vx,\vy)\|_{2}^{2}\,.
\end{equation}
It remains to estimate the error of the estimated posterior denoiser. Since $\mX_{t}^{\vy}$ is independent of $\calG$, it suffices to bound the expectation over $\calG$ of the right-hand side of~\eqref{eq:lemma:section:proof:posterior:score:1:0} at each fixed query point $\vx$, and to integrate the resulting bound with respect to $\mX_{t}^{\vy}\sim q_{t}(\cdot\given\vy)$ at the end. By Lemma~\ref{lemma:random:score:field}, the endpoints $\what{\mX}_{0,S,1}^{\vx,\vy,t},\ldots,\what{\mX}_{0,S,m}^{\vx,\vy,t}$ entering $\what{\mD}_{m}^{S}(t,\vx,\vy)$ are i.i.d.\ with law $\hat{p}_{t}^{S}(\cdot\given\vx,\vy)$, so the expectation over $\calG$ of any functional of the endpoint family is determined by this $m$-fold product law. Below, this expectation is computed with the endpoint family realized by the coupling of Step 1.

\noindent\textit{Step 1. Construct a Wasserstein coupling.}
Fix the query point $(t,\vx)$, and let $(\mX_{0,i}^{\vx,\vy,t},\what{\mX}_{0,S,i}^{\vx,\vy,t})_{i=1}^{m}$ be i.i.d.\ pairs drawn from an optimal coupling of $p_{t}(\cdot\given\vx,\vy)$ and $\hat{p}_{t}^{S}(\cdot\given\vx,\vy)$. Then each of the two families is i.i.d., the family $(\what{\mX}_{0,S,i}^{\vx,\vy,t})_{i=1}^{m}$ has the same joint law as the endpoint family entering $\what{\mD}_{m}^{S}(t,\vx,\vy)$, and for each $1\leq i\leq m$,
\begin{equation}\label{eq:lemma:section:proof:posterior:score:1:1}
\bbE\big(\|\mX_{0,i}^{\vx,\vy,t}-\what{\mX}_{0,S,i}^{\vx,\vy,t}\|_{2}^{2}\big)=\bbW_{2}^{2}\big\{\hat{p}_{t}^{S}(\cdot \given \vx,\vy),p_{t}(\cdot \given \vx,\vy)\big\}\,.
\end{equation}
 By the definition of the estimated posterior denoiser~\eqref{eq:denoiser:MC}, we find
\begin{align*}
&\|\mD(t,\vx,\vy)-\what{\mD}_{m}^{S}(t,\vx,\vy)\|_{2}^{2}
=\bigg\|\mD(t,\vx,\vy)-\frac{1}{m}\sum_{i=1}^{m}\mX_{0,i}^{\vx,\vy,t}+\frac{1}{m}\sum_{i=1}^{m}\mX_{0,i}^{\vx,\vy,t}-\what{\mD}_{m}^{S}(t,\vx,\vy)\bigg\|_{2}^{2} \\
&~~~~~~~~~~~~\le 2\bigg\|\mD(t,\vx,\vy)-\frac{1}{m}\sum_{i=1}^{m}\mX_{0,i}^{\vx,\vy,t}\bigg\|_{2}^{2}+2\bigg\|\frac{1}{m}\sum_{i=1}^{m}(\mX_{0,i}^{\vx,\vy,t}-\what{\mX}_{0,S,i}^{\vx,\vy,t})\bigg\|_{2}^{2}\,.
\end{align*}
Taking expectation over the product law of the $m$ coupled pairs on both sides of the above inequality yields
\begin{align}
&\bbE\big\{\|\mD(t,\vx,\vy)-\what{\mD}_{m}^{S}(t,\vx,\vy)\|_{2}^{2}\big\} \nonumber \\
&~~~~~~\le\underbrace{2\bbE\bigg\{\bigg\|\mD(t,\vx,\vy)-\frac{1}{m}\sum_{i=1}^{m}\mX_{0,i}^{\vx,\vy,t}\bigg\|_{2}^{2}\bigg\}}_{\text{(i)}}+\underbrace{2\bbE\bigg\{\bigg\|\frac{1}{m}\sum_{i=1}^{m}(\mX_{0,i}^{\vx,\vy,t}-\what{\mX}_{0,S,i}^{\vx,\vy,t})\bigg\|_{2}^{2}\bigg\}}_{\text{(ii)}}\,. \label{eq:lemma:section:proof:posterior:score:1:2}
\end{align}
Here the term (i) measures the error of the Monte Carlo approximation~\eqref{eq:denoiser:MC}, while the term (ii) reveals the error of the Langevin dynamics with estimated score~\eqref{eq:RGO:Langevin:score}. 

\noindent\textit{Step 2. Bound the error of the Monte Carlo approximation.}
Let $\mX_{0}^{\vx,\vy,t}$ be a random copy of $\mX_{0,1}^{\vx,\vy,t}$ and independent of $\mX_{0,1:m}^{\vx,\vy,t}$. It is apparent that $\mD(t,\vx,\vy)=\bbE(\mX_{0}^{\vx,\vy,t})$. Consequently, 
\begin{align}
&\bbE\bigg\{\bigg\|\mD(t,\vx,\vy)-\frac{1}{m}\sum_{i=1}^{m}\mX_{0,i}^{\vx,\vy,t}\bigg\|_{2}^{2}\bigg\} \nonumber =\frac{1}{m^{2}}\bbE\bigg[\bigg\|\sum_{i=1}^{m}\{\bbE(\mX_{0}^{\vx,\vy,t})-\mX_{0,i}^{\vx,\vy,t}\}\bigg\|_{2}^{2}\bigg] \nonumber \\
&~~~~~~~~=\frac{1}{m^{2}}\sum_{i=1}^{m}\sum_{j=1}^{m}\bbE\big\{\langle\bbE(\mX_{0}^{\vx,\vy,t})-\mX_{0,i}^{\vx,\vy,t},\bbE(\mX_{0}^{\vx,\vy,t})-\mX_{0,j}^{\vx,\vy,t}\rangle\big\} \nonumber \\
&~~~~~~~~=\frac{1}{m^{2}}\sum_{i=1}^{m}\bbE\big\{\|\bbE(\mX_{0}^{\vx,\vy,t})-\mX_{0,i}^{\vx,\vy,t}\|_{2}^{2}\big\} \nonumber \\
&~~~~~~~~=\frac{1}{m}\trace\{\cov(\mX_{0,1}^{\vx,\vy,t})\}=\frac{1}{m}\trace\{\cov(\mX_{0}\given\mX_{t}=\vx,\mY=\vy)\}\,, \label{eq:lemma:section:proof:posterior:score:1:3}
\end{align}
where the third equality is due to the independence of distinct particles and the centering $\bbE(\mX_{0}^{\vx,\vy,t})=\bbE(\mX_{0,i}^{\vx,\vy,t})$ for each $1\leq i\leq m$. Note that 
\begin{align}
&\int\cov(\mX_{0} \given \mX_{t}=\vx,\mY=\vy)q_{t}(\vx \given \vy)\d\vx \nonumber \\
&~~~~~~~~~~~~=\cov(\mX_{0} \given \mY=\vy)-\cov\big\{\bbE(\mX_{0} \given \mX_{t},\mY=\vy)\big\} \nonumber \\
&~~~~~~~~~~~~\preceq\cov(\mX_{0} \given \mY=\vy)\preceq\bbE(\mX_{0}\mX_{0}^{\top} \given \mY=\vy)\,, \label{eq:lemma:section:proof:posterior:score:1:4}
\end{align}
where the equality holds from the law of total variance, and the first inequality used the fact that covariance matrix is semi-positive definite. As a consequence,
\begin{align}
&\int\trace(\cov(\mX_{0} \given \mX_{t}=\vx,\mY=\vy))q_{t}(\vx \given \vy)\d\vx \nonumber \\
&~~~~~~~~~~~~\leq\trace\{\bbE(\mX_{0}\mX_{0}^{\top} \given \mY=\vy)\}=\int\|\vx_{0}\|_{2}^{2}q_{0}(\vx_{0} \given \vy)\d\vx_{0} \nonumber \\
&~~~~~~~~~~~~=V_{\SG}^{2}\int\bigg(\frac{\|\vx_{0}\|_{2}^{2}}{V_{\SG}^{2}}+1\bigg)q_{0}(\vx_{0} \given \vy)\d\vx_{0}-V_{\SG}^{2} \nonumber \\
&~~~~~~~~~~~~\leq V_{\SG}^{2}\int\exp\bigg(\frac{\|\vx_{0}\|_{2}^{2}}{V_{\SG}^{2}}\bigg)q_{0}(\vx_{0} \given \vy)\d\vx_{0}\leq\kappa_{\vy}C_{\SG}V_{\SG}^{2}\,, \label{eq:lemma:section:proof:posterior:score:1:5}
\end{align}
where the first inequality is due to~\eqref{eq:lemma:section:proof:posterior:score:1:4}, the second inequality follows from $1+z\leq\exp(z)$ for each $z\in\bbR$, and the last inequality invokes Proposition~\ref{proposition:sub:Gaussian:posterior}. Multiplying both sides of~\eqref{eq:lemma:section:proof:posterior:score:1:3} by $q_{t}(\vx \given \vy)$, integrating with respect to $\vx$, and substituting~\eqref{eq:lemma:section:proof:posterior:score:1:5} imply
\begin{equation}\label{eq:lemma:section:proof:posterior:score:1:6}
\int\bbE\bigg\{\bigg\|\mD(t,\vx,\vy)-\frac{1}{m}\sum_{i=1}^{m}\mX_{0,i}^{\vx,\vy,t}\bigg\|_{2}^{2}\bigg\}q_{t}(\vx \given \vy)\d\vx\leq\frac{\kappa_{\vy}C_{\SG}V_{\SG}^{2}}{m}\,.
\end{equation}

\noindent\textit{Step 3. Bound the error of the Langevin dynamics with estimated score.}
For the term (ii) in~\eqref{eq:lemma:section:proof:posterior:score:1:2}, we have 
\begin{align}
&\bbE\bigg\{\bigg\|\frac{1}{m}\sum_{i=1}^{m}(\mX_{0,i}^{\vx,\vy,t}-\what{\mX}_{0,S,i}^{\vx,\vy,t})\bigg\|_{2}^{2}\bigg\} \nonumber \\
&~~~~~~~~~~~~=\frac{1}{m^{2}}\sum_{i=1}^{m}\sum_{j=1}^{m}\bbE\big(\langle\mX_{0,i}^{\vx,\vy,t}-\what{\mX}_{0,S,{i}}^{\vx,\vy,t},\mX_{0,j}^{\vx,\vy,t}-\what{\mX}_{0,S,j}^{\vx,\vy,t} \rangle\big) \nonumber \\
&~~~~~~~~~~~~~\leq\frac{1}{2m^{2}}\sum_{i=1}^{m}\sum_{j=1}^{m}\bbE\big(\|\mX_{0,i}^{\vx,\vy,t}-\what{\mX}_{0,S,i}^{\vx,\vy,t}\|_{2}^{2}+\|\mX_{0,j}^{\vx,\vy,t}-\what{\mX}_{0,S,j}^{\vx,\vy,t}\|_{2}^{2}\big) \nonumber \\
&~~~~~~~~~~~~~=\bbW_{2}^{2}\big\{\hat{p}_{t}^{S}(\cdot \given \vx,\vy),p_{t}(\cdot \given \vx,\vy)\big\}\,, \label{eq:lemma:section:proof:posterior:score:1:7}
\end{align}
where the first inequality holds from the Cauchy-Schwarz inequality, and the last equality is due to~\eqref{eq:lemma:section:proof:posterior:score:1:1}.

\noindent\textit{Step 4. Conclusion.}
Multiplying both sides of~\eqref{eq:lemma:section:proof:posterior:score:1:2} by $q_t(\vx\given\vy)$, integrating with respect to $\vx$, and then applying~\eqref{eq:lemma:section:proof:posterior:score:1:6} and~\eqref{eq:lemma:section:proof:posterior:score:1:7} yield
\begin{align*}
&\bbE\big\{\|\mD(t,\mX_{t}^{\vy},\vy)-\what{\mD}_{m}^{S}(t,\mX_{t}^{\vy},\vy)\|_{2}^{2}\big\} \\
&~~~~~~~~~~~~~~~~~~~\leq2\bbE\big[\bbW_{2}^{2}\big\{\hat{p}_{t}^{S}(\cdot \given \mX_{t}^{\vy},\vy),p_{t}(\cdot \given \mX_{t}^{\vy},\vy)\big\}\big]+\frac{2\kappa_{\vy}C_{\SG}V_{\SG}^{2}}{m}\,,
\end{align*}
where the expectation on the left-hand side averages over $\calG$ and $\mX_{t}^{\vy}\sim q_{t}(\cdot\given\vy)$, owing to the reduction stated before Step 1. Combining the above inequality with~\eqref{eq:lemma:section:proof:posterior:score:1:0} completes the proof.
\end{proof}

\subsection{Convergence of the score-based Langevin in Wasserstein distance}
In this subsection, we aim to offer a bound for the first term in Lemma~\ref{lemma:section:proof:posterior:score:1}, which is the Wasserstein error of Langevin dynamics with estimated score~\eqref{eq:RGO:Langevin:score}.

Before presenting, we propose the following auxiliary lemma, which connects the Wasserstein error between the laws of two SDEs to the $L^{2}$-error between their drift terms. The proof of this lemma is based on the technique of Wasserstein coupling.

\begin{lemma}
\label{lemma:section:proof:posterior:score:W2:LMC}
Let $\mZ_{s}$ and $\bar{\mZ}_{s}$ be two stochastic processes defined, respectively, by the following SDEs:
\begin{equation}\label{eq:lemma:section:proof:posterior:score:W2:LMC:1}
\begin{aligned}
\d\mZ_{s}&=\vb(\mZ_{s})\ds+\sqrt{2}\d\mB_{s}\,, \quad \mZ_{0}\sim\mu_{0}\,, \\
\d\bar{\mZ}_{s}&=\bar{\vb}(\bar{\mZ}_{s})\ds+\sqrt{2}\d\mB_{s}\,, \quad \bar{\mZ}_{0}\sim\mu_{0}\,.
\end{aligned} 
\end{equation}
Denote by $\mu_{s}$ the distribution of $\mZ_{s}$, and denote by $\bar{\mu}_{s}$ the distribution of $\bar{\mZ}_{s}$. Assume that the two equations admit non-explosive strong solutions on $[0,S]$. Assume further that $\bar{\vb}$ is $\beta$-Lipschitz. Then it holds that
\begin{equation*}
\bbW_{2}^{2}(\mu_{S},\bar{\mu}_{S})\leq S\exp(2\beta S)\int_{0}^{S}\bbE_{\mZ_{s}\sim\mu_{s}}\big\{\|\vb(\mZ_{s})-\bar{\vb}(\mZ_{s})\|_{2}^{2}\big\}\ds\,,
\end{equation*}
where $S>0$ is the time horizon.
\end{lemma}

\begin{proof}[Proof of Lemma~\ref{lemma:section:proof:posterior:score:W2:LMC}]
We evolve these two SDEs~\eqref{eq:lemma:section:proof:posterior:score:W2:LMC:1} from the same initial particle  $\bar{\mZ}_{0}=\mZ_{0}\sim\mu_{0}$, and use the same Brownian motion. Denote the joint distribution $\gamma_{0}(\vz_{0},\bar{\vz}_{0})=\delta\{\vz_{0}=\bar{\vz}_{0}\}\mu_{0}(\vz_{0})$, which is a coupling of $(\mu_{0},\mu_{0})$. Let $\gamma_{s}$ denote the law of $(\mZ_{s},\bar{\mZ}_{s})$~\eqref{eq:lemma:section:proof:posterior:score:W2:LMC:1}. Notice that $\gamma_{s}$ is a coupling of $(\mu_{s},\bar{\mu}_{s})$. We next aim to bound $\bbE(\|\mZ_{S}-\bar{\mZ}_{S}\|_{2}^{2})$. First, it follows from It{\^o}'s formula that 
\begin{align*}
\d\|\mZ_{s}-\bar{\mZ}_{s}\|_{2}^{2}
&=2\langle\mZ_{s}-\bar{\mZ}_{s},\d\mZ_{s}-\d\bar{\mZ}_{s}\rangle \\
&=2\langle\mZ_{s}-\bar{\mZ}_{s},\vb(\mZ_{s})-\bar{\vb}(\bar{\mZ}_{s})\rangle\ds \\
&\leq 2\|\mZ_{s}-\bar{\mZ}_{s}\|_{2}\|\vb(\mZ_{s})-\bar{\vb}(\bar{\mZ}_{s})\|_{2}\ds\,,
\end{align*}
where the second equality used~\eqref{eq:lemma:section:proof:posterior:score:W2:LMC:1} and the fact that the Brownian motions in two SDEs are the same, and the inequality holds from Cauchy-Schwarz inequality.
 For $\rho>0$, put $r_{\rho}(s)=(\|\mZ_{s}-\bar{\mZ}_{s}\|_{2}^{2}+\rho)^{1/2}$. Since the noise cancels, $s\mapsto\mZ_{s}-\bar{\mZ}_{s}$ is absolutely continuous, and the preceding display gives, for almost every $s$,
\begin{align*}
\frac{\d}{\ds}r_{\rho}(s)
\leq\|\vb(\mZ_{s})-\bar{\vb}(\bar{\mZ}_{s})\|_{2}
\leq\|\vb(\mZ_{s})-\bar{\vb}(\mZ_{s})\|_{2}+\beta r_{\rho}(s)\,,
\end{align*}
where the last inequality uses the triangular inequality, the $\beta$-Lipschitz continuity of $\bar{\vb}$, and $\|\mZ_{s}-\bar{\mZ}_{s}\|_{2}\leq r_{\rho}(s)$. Applying Gronwall's inequality~\citep[Section B.2]{app:evans2010partial} to $r_{\rho}$, whose initial value is $r_{\rho}(0)=\sqrt{\rho}$ since $\mZ_{0}=\bar{\mZ}_{0}$, and letting $\rho\downarrow0$ yields
\begin{equation*}
\|\mZ_{S}-\bar{\mZ}_{S}\|_{2}\leq\exp(\beta S)\int_{0}^{S}\|\vb(\mZ_{s})-\bar{\vb}(\mZ_{s})\|_{2}\ds\,,
\end{equation*}
for each $S>0$. Here we used the fact that $\mZ_{0}=\bar{\mZ}_{0}$. Taking expectation on both sides of the inequality with respect to $(\mZ_{S},\bar{\mZ}_{S})\sim\gamma_{S}$, and using Jensen's inequality yields
\begin{align*}
\bbW_{2}^{2}(\mu_{S},\bar{\mu}_{S})&\leq\bbE_{(\mZ_{S},\bar{\mZ}_{S})\sim\gamma_{S}}(\|\mZ_{S}-\bar{\mZ}_{S}\|_{2}^{2}) \\
&\leq S\exp(2\beta S)\int_{0}^{S}\bbE_{\mZ_{s}\sim\mu_{s}}\big\{\|\vb(\mZ_{s})-\bar{\vb}(\mZ_{s})\|_{2}^{2}\big\}\ds\,,
\end{align*}
where the first inequality follows from the definition of the Wasserstein distance. This completes the proof.
\end{proof}

With the aid of Lemma~\ref{lemma:section:proof:posterior:score:W2:LMC}, we propose an error analysis in Wasserstein distance for Langevin dynamics with estimated score~\eqref{eq:RGO:Langevin:score}.

\begin{lemma}
\label{lemma:section:proof:posterior:score:2}
Suppose Assumptions~\ref{assumption:semi:log:concave},~\ref{assumption:Lipschitz:prior:likelihood}, and~\ref{assumption:random:field:initialization} hold. Let $0<T_{0}<T<\frac{1}{2}\log(1+\alpha^{-1})$. For each time $t\in(T_{0},T]$,
\begin{align*}
&\bbE\big[\bbW_{2}^{2}\big\{
    p_{t}(\cdot \given \mX_{t}^{\vy},\vy),
    \hat{p}_{t}^{S}(\cdot \given \mX_{t}^{\vy},\vy)
  \big\}\big] \\
&\quad\leq 2\exp\bigg\{-\frac{2(\mu_{t}^{2}-\alpha\sigma_{t}^{2})S}{\sigma_{t}^{2}}\bigg\}\eta_{\vy}^{2} \\
&\qquad+2S\exp\bigg\{2\bigg(G+\frac{\mu_{t}^{2}}{\sigma_{t}^{2}}\bigg)S\bigg\}
  \int_{0}^{S}\int\bbE\big\{
    \|\nabla\log\pi_{0}(\mX_{0,s}^{\vx,\vy,t}) \\
&\hspace{12em}
    -\hat{\vs}_{\prior}(\mX_{0,s}^{\vx,\vy,t})\|_{2}^{2}
  \big\}q_{t}(\vx \given \vy)\d\vx\ds\,,
\end{align*}
where the initial discrepancy $\eta_{\vy}<\infty$ satisfies
\begin{equation*}
\eta_{\vy}^{2}\geq\sup_{t\in(T_{0},T]}\bbE\big[\bbW_{2}^{2}\big\{\hat{p}_{t}^{0}(\cdot \given \mX_{t}^{\vy},\vy),p_{t}(\cdot \given \mX_{t}^{\vy},\vy)\big\}\big]\,.
\end{equation*}
\end{lemma}

\begin{proof}[Proof of Lemma~\ref{lemma:section:proof:posterior:score:2}]
By the triangular inequality, we find that for each fixed $\vx\in\bbR^{d}$,
\begin{align}
&\bbW_{2}^{2}\big\{p_{t}(\cdot \given \vx,\vy),\hat{p}_{t}^{S}(\cdot \given \vx,\vy)\big\} \nonumber \\
&~~~~~~~~\leq 2\underbrace{\bbW_{2}^{2}\big\{p_{t}(\cdot \given \vx,\vy),p_{t}^{S}(\cdot \given \vx,\vy)\big\}}_{\text{convergence of Langevin dynamics}}+2\underbrace{\bbW_{2}^{2}\big\{p_{t}^{S}(\cdot \given \vx,\vy),\hat{p}_{t}^{S}(\cdot \given \vx,\vy)\big\}}_{\text{score estimation error}}\,, \label{eq:section:proof:posterior:score:2:1}
\end{align}
where $p_{t}^{S}(\cdot \given \vx,\vy)$ is the marginal density of $\mX_{0,S}^{\vx,\vy,t}$ defined in~\eqref{eq:RGO:Langevin}. As specified in Appendix~\ref{appendix:random:field}, the exact and estimated inner dynamics are both initialized from the common distribution $\hat p_t^0(\cdot\given\vx,\vy)$.

\noindent\emph{Step 1. The convergence of Langevin dynamics.} 
For the first term in the right-hand side of~\eqref{eq:section:proof:posterior:score:2:1},
\begin{equation*}
\bbW_{2}^{2}\big\{p_{t}(\cdot \given \vx,\vy),p_{t}^{S}(\cdot \given \vx,\vy)\big\}\leq\exp\bigg\{-\frac{2(\mu_{t}^{2}-\alpha\sigma_{t}^{2})S}{\sigma_{t}^{2}}\bigg\}\bbW_{2}^{2}\big\{p_{t}(\cdot \given \vx,\vy),\hat{p}_{t}^{0}(\cdot \given \vx,\vy)\big\}\,,
\end{equation*}
where the inequality holds from Proposition~\ref{proposition:appendix:RGO:Hessian} and Lemma~\ref{lemma:Langevin:W2}. Multiplying both sides of the inequality by $q_{t}(\vx \given \vy)$, and integrating with respect to $\vx$ yield
\begin{equation}\label{eq:section:proof:posterior:score:2:2}
\bbE\big[\bbW_{2}^{2}\big\{p_{t}(\cdot \given \mX_{t}^{\vy},\vy),p_{t}^{S}(\cdot \given \mX_{t}^{\vy},\vy)\big\}\big]\leq\exp\bigg\{-\frac{2(\mu_{t}^{2}-\alpha\sigma_{t}^{2})S}{\sigma_{t}^{2}}\bigg\}\eta_{\vy}^{2}\,. 
\end{equation}

\noindent\emph{Step 2. The error of the score matching.} 
We turn to consider the second term in the right-hand side of~\eqref{eq:section:proof:posterior:score:2:1}. Denote by $\bar{\vb}$ the drift term of~\eqref{eq:RGO:Langevin:score}, that is,
\begin{equation*}
\bar{\vb}(t,\vx_{0})\coloneq \hat{\vs}_{\prior}(\vx_{0})+\frac{\mu_{t}}{\sigma_{t}^{2}}(\vx-\mu_{t}\vx_{0})-\nabla\ell_{\vy}(\vx_{0})\,, \quad \vx_{0}\in\bbR^{d}\,.
\end{equation*}
It follows from Assumption~\ref{assumption:Lipschitz:prior:likelihood} that 
\begin{align*}
\|\bar{\vb}(t,\vx_{0})-\bar{\vb}(t,\vx_{0}^{\prime})\|_{2}\leq\bigg(G+\frac{\mu_{t}^{2}}{\sigma_{t}^{2}}\bigg)\|\vx_{0}-\vx_{0}^{\prime}\|_{2}\,.
\end{align*}
Combining this Lipschitz continuity with Lemma~\ref{lemma:section:proof:posterior:score:W2:LMC} implies
\begin{align}
&\bbE\big[\bbW_{2}^{2}\big\{p_{t}^{S}(\cdot \given \mX_{t}^{\vy},\vy),\hat{p}_{t}^{S}(\cdot \given \mX_{t}^{\vy},\vy)\big\}\big] \notag \\
&~~~~~=\int\bbW_{2}^{2}\big\{p_{t}^{S}(\cdot \given \vx,\vy),\hat{p}_{t}^{S}(\cdot \given \vx,\vy)\big\}q_{t}(\vx \given \vy)\d\vx \nonumber \\
&~~~~~\leq S\exp\bigg\{2\bigg(G+\frac{\mu_{t}^{2}}{\sigma_{t}^{2}}\bigg)S\bigg\}\int_{0}^{S}\int\bbE\big\{\|\nabla\log\pi_{0}(\mX_{0,s}^{\vx,\vy,t}) \notag \\
&~~~~~~~~~~~~~~~~~~~~~~~~~~~~~~~~~~~~~~~~~~~~~~~~~~~~~~~~~~~~~~~~~~~~-\hat{\vs}_{\prior}(\mX_{0,s}^{\vx,\vy,t})\|_{2}^{2}\big\}q_{t}(\vx \given \vy)\d\vx\ds\,, \label{eq:section:proof:posterior:score:2:3}
\end{align}
where the expectation is taken with respect to $\mX_{0,s}^{\vx,\vy,t}\sim p_{t}^{s}(\cdot \given \vx,\vy)$.

\noindent\emph{Step 3. Conclusion.} 
Multiplying both sides of~\eqref{eq:section:proof:posterior:score:2:1} by $q_t(\vx\given\vy)$, integrating with respect to $\vx$, and then applying~\eqref{eq:section:proof:posterior:score:2:2} and~\eqref{eq:section:proof:posterior:score:2:3} yield the claimed inequality.
\end{proof}

The following lemma offers a bound for the $L^{2}$-error of the prior score estimator, which is inspired by~\citet[Theorem 1]{app:Tang2024Adaptivity},~\citet{app:jiang2025simulation} and~\citet[Theorem 3.2]{app:ding2024nonlinear}.

\begin{lemma}[Score estimation error]
\label{lemma:section:proof:posterior:score:3}
Suppose Assumptions~\ref{assumption:semi:log:concave},~\ref{assumption:prior:subGaussian},~\ref{assumption:prior:score:error}, and~\ref{assumption:prior:bound} hold. Let $0<T_{0}<T<\frac{1}{2}\log(1+\alpha^{-1})$. For each time $t\in(T_{0},T]$, it holds that
\begin{align*}
&\int_{0}^{S}\int\bbE\big\{\|\nabla\log\pi_{0}(\mX_{0,s}^{\vx,\vy,t})-\hat{\vs}_{\prior}(\mX_{0,s}^{\vx,\vy,t})\|_{2}^{2}\big\}q_{t}(\vx \given \vy)\d\vx\ds \\
&~~~~~~~~~~~~\leq CS^{1/2}\bigg\{\frac{\sigma_{t}^{2}\eta_{\vy}^{2}}{2(\mu_{t}^{2}-\alpha\sigma_{t}^{2})}+S\bigg\}^{1/2}\kappa_{\vy}^{1/2}\varepsilon_{\prior}^{1/2}\,,
\end{align*}
where $C$ is a constant only depending on $B$, $r$, $V_{\SG}$ and $C_{\SG}$, and the initial discrepancy $\eta_{\vy}$ satisfies
\begin{equation*}
\eta_{\vy}^{2}\geq\sup_{t\in(T_{0},T]}\bbE\big[\chi^{2}\big\{\hat{p}_{t}^{0}(\cdot \given \mX_{t}^{\vy},\vy)\|p_{t}(\cdot \given \mX_{t}^{\vy},\vy)\big\}\big]\,.
\end{equation*}
\end{lemma}

\begin{proof}[Proof of Lemma~\ref{lemma:section:proof:posterior:score:3}]
It is straightforward that 
\begin{align}
&\int_{0}^{S}\int\bbE\big\{\|\nabla\log\pi_{0}(\mX_{0,s}^{\vx,\vy,t})-\hat{\vs}_{\prior}(\mX_{0,s}^{\vx,\vy,t})\|_{2}^{2}\big\}q_{t}(\vx \given \vy)\d\vx\ds \nonumber \\
&~~~=\int_{0}^{S}\bigg\{\iint\|\nabla\log\pi_{0}(\vx_{0})-\hat{\vs}_{\prior}(\vx_{0})\|_{2}^{2}\frac{p_{t}^{s}(\vx_{0} \given \vx,\vy)}{p_{t}(\vx_{0} \given \vx,\vy)}p_{t}(\vx_{0} \given \vx,\vy)\d\vx_{0}q_{t}(\vx \given \vy)\d\vx\bigg\}\ds \nonumber \\
&~~~\leq\bigg\{\underbrace{\iint\|\nabla\log\pi_{0}(\vx_{0})-\hat{\vs}_{\prior}(\vx_{0})\|_{2}^{2}p_{t}(\vx_{0} \given \vx,\vy)\d\vx_{0}q_{t}(\vx \given \vy)\d\vx}_{\text{(i)}}\bigg\}^{1/4} \nonumber \\
&\quad\times\bigg\{\underbrace{\iint\|\nabla\log\pi_{0}(\vx_{0})-\hat{\vs}_{\prior}(\vx_{0})\|_{2}^{6}p_{t}(\vx_{0} \given \vx,\vy)\d\vx_{0}q_{t}(\vx \given \vy)\d\vx}_{\text{(ii)}}\bigg\}^{1/4} \nonumber \\
&\quad\times\underbrace{\int_{0}^{S}\bigg[\iint\bigg\{\frac{p_{t}^{s}(\vx_{0} \given \vx,\vy)}{p_{t}(\vx_{0} \given \vx,\vy)}\bigg\}^{2}p_{t}(\vx_{0} \given \vx,\vy)\d\vx_{0}q_{t}(\vx \given \vy)\d\vx\bigg]^{1/2}\ds}_{\text{(iii)}}\,, \label{eq:lemma:section:proof:posterior:score:3:1}
\end{align}
where the inequality holds from H{\"o}lder's inequality.

\noindent\emph{Step 1. Bound the term (i) in~\eqref{eq:lemma:section:proof:posterior:score:3:1}.} 
It follows from Assumption~\ref{assumption:prior:score:error} that 
\begin{align}
&\iint\|\nabla\log\pi_{0}(\vx_{0})-\hat{\vs}_{\prior}(\vx_{0})\|_{2}^{2}p_{t}(\vx_{0} \given \vx,\vy)\d\vx_{0}q_{t}(\vx \given \vy)\d\vx \nonumber \\
&~~~~~~~~=\int\|\nabla\log\pi_{0}(\vx_{0})-\hat{\vs}_{\prior}(\vx_{0})\|_{2}^{2}\bigg\{\frac{q_{0}(\vx_{0} \given \vy)}{\pi_{0}(\vx_{0})}\bigg\}\pi_{0}(\vx_{0})\d\vx_{0} \nonumber \\
&~~~~~~~~\leq\sup_{\vx_{0}}\bigg\{\frac{q_{0}(\vx_{0} \given \vy)}{\pi_{0}(\vx_{0})}\bigg\}\int\|\nabla\log\pi_{0}(\vx_{0})-\hat{\vs}_{\prior}(\vx_{0})\|_{2}^{2}\pi_{0}(\vx_{0})\d\vx_{0}\leq \kappa_{\vy}\varepsilon_{\prior}^{2}\,,\label{eq:lemma:section:proof:posterior:score:3:2}
\end{align}
where the first inequality is the $L^{\infty}$--$L^{1}$ H{\"o}lder inequality, and the final bound follows from~\eqref{eq:posterior:score:condition} and Assumption~\ref{assumption:prior:score:error}.

\noindent\emph{Step 2. Bound the term (ii) in~\eqref{eq:lemma:section:proof:posterior:score:3:1}.} 
By an argument similar to~\eqref{eq:lemma:section:proof:posterior:score:3:2}, we have 
\begin{align}
&\iint\|\nabla\log\pi_{0}(\vx_{0})-\hat{\vs}_{\prior}(\vx_{0})\|_{2}^{6}p_{t}(\vx_{0} \given \vx,\vy)\d\vx_{0}q_{t}(\vx \given \vy)\d\vx \nonumber \\
&~~~~~~~~\leq \kappa_{\vy}\int\|\nabla\log\pi_{0}(\vx_{0})-\hat{\vs}_{\prior}(\vx_{0})\|_{2}^{6}\pi_{0}(\vx_{0})\d\vx_{0} \nonumber \\
&~~~~~~~~\leq 2048B^{6}\kappa_{\vy}\int\big(1+\|\vx_{0}\|_{2}^{6r}\big)\pi_{0}(\vx_{0})\d\vx_{0} \nonumber \\
&~~~~~~~~= 2048B^{6}\kappa_{\vy}+2048B^{6}\kappa_{\vy}\bbE\big(\|\mX_{0}\|_{2}^{6r}\big)\,, \label{eq:lemma:section:proof:posterior:score:3:3}
\end{align}
where the second inequality follows from Assumption~\ref{assumption:prior:bound}. Using the same argument as Lemma~\ref{lemma:fourth:moment}, we have
\begin{equation}\label{eq:lemma:section:proof:posterior:score:3:5}
\bbE\big(\|\mX_{0}\|_{2}^{6r}\big)\leq 2(6r)^{3r+1}C_{\SG}V_{\SG}^{6r}\,,
\end{equation}
Substituting~\eqref{eq:lemma:section:proof:posterior:score:3:5} into~\eqref{eq:lemma:section:proof:posterior:score:3:3} implies
\begin{equation}\label{eq:lemma:section:proof:posterior:score:3:6}
\iint\|\nabla\log\pi_{0}(\vx_{0})-\hat{\vs}_{\prior}(\vx_{0})\|_{2}^{6}p_{t}(\vx_{0} \given \vx,\vy)\d\vx_{0}q_{t}(\vx \given \vy)\d\vx\lesssim B^{6}\kappa_{\vy}(1+C_{\SG}V_{\SG}^{6r})\,.
\end{equation}

\noindent\emph{Step 3. Bound the term (iii) in~\eqref{eq:lemma:section:proof:posterior:score:3:1}.} 
By the definition of the $\chi^{2}$-divergence, we find
\begin{align}
&\int_{0}^{S}\bigg[\iint\bigg\{\frac{p_{t}^{s}(\vx_{0} \given \vx,\vy)}{p_{t}(\vx_{0} \given \vx,\vy)}\bigg\}^{2}p_{t}(\vx_{0} \given \vx,\vy)\d\vx_{0}q_{t}(\vx \given \vy)\d\vx\bigg]^{1/2}\ds \nonumber \\ 
&~~~~~~~~=\int_{0}^{S}\bigg(\bbE\big[\chi^{2}\big\{p_{t}^{s}(\cdot \given \mX_{t}^{\vy},\vy)\|p_{t}(\cdot \given \mX_{t}^{\vy},\vy)\big\}\big]+1\bigg)^{1/2}\ds \nonumber \\
&~~~~~~~~\leq S^{1/2}\bigg(\int_{0}^{S}\bbE\big[\chi^{2}\big\{p_{t}^{s}(\cdot \given \mX_{t}^{\vy},\vy)\|p_{t}(\cdot \given \mX_{t}^{\vy},\vy)\big\}\big]\ds+S\bigg)^{1/2} \nonumber \\
&~~~~~~~~\leq S^{1/2}\bigg(\bbE\big[\chi^{2}\big\{\hat{p}_{t}^{0}(\cdot \given \mX_{t}^{\vy},\vy)\|p_{t}(\cdot \given \mX_{t}^{\vy},\vy)\big\}\big]\int_{0}^{S}\exp\bigg\{-\frac{2(\mu_{t}^{2}-\alpha\sigma_{t}^{2})s}{\sigma_{t}^{2}}\bigg\}\ds+S\bigg)^{1/2} \nonumber \\
&~~~~~~~~\leq S^{1/2}\bigg\{\frac{\sigma_{t}^{2}\eta_{\vy}^{2}}{2(\mu_{t}^{2}-\alpha\sigma_{t}^{2})}+S\bigg\}^{1/2}\,, \label{eq:lemma:section:proof:posterior:score:3:7}
\end{align}
where the expectation is taken with respect to $\mX_{t}^{\vy}\sim q_{t}(\cdot \given \vy)$, the first inequality invokes Jensen's inequality, and the second inequality holds from Proposition~\ref{proposition:appendix:RGO:Hessian} and Lemma~\ref{lemma:Langevin:chisq}.

\noindent\emph{Step 4. Conclusion.} 
Substituting~\eqref{eq:lemma:section:proof:posterior:score:3:2},~\eqref{eq:lemma:section:proof:posterior:score:3:6} and~\eqref{eq:lemma:section:proof:posterior:score:3:7} into~\eqref{eq:lemma:section:proof:posterior:score:3:1} yields the desired results.
\end{proof}

\subsection{Proof of Proposition~\ref{proposition:posterior:score} and Corollary~\ref{corollary:posterior:score}}
\begin{proof}[Proof of Proposition~\ref{proposition:posterior:score}]
Combining Lemmas~\ref{lemma:section:proof:posterior:score:1} and~\ref{lemma:section:proof:posterior:score:2},~\ref{lemma:section:proof:posterior:score:3} implies
\begin{align}
&\bbE\big\{\|\nabla\log q_{t}(\mX_{t}^{\vy} \given \vy)-\hat{\vs}_{m}^{S}(t,\mX_{t}^{\vy},\vy)\|_{2}^{2}\big\} \nonumber \\
&~~~~~~~~\leq C\frac{\mu_{t}^{2}}{\sigma_{t}^{4}}\bigg[\frac{\kappa_{\vy}}{m}+\exp\bigg\{-\frac{2(\mu_{t}^{2}-\alpha\sigma_{t}^{2})}{\sigma_{t}^{2}}S\bigg\}\eta_{\vy}^{2} \nonumber \\
&~~~~~~~~\quad\quad+S^{3/2}\bigg(\frac{\sigma_{t}^{2}\eta_{\vy}^{2}}{\mu_{t}^{2}-\alpha\sigma_{t}^{2}}+S\bigg)^{1/2}\kappa_{\vy}^{1/2}\exp\bigg\{2\bigg(G+\frac{\mu_{t}^{2}}{\sigma_{t}^{2}}\bigg)S\bigg\}\varepsilon_{\prior}^{1/2}\bigg]\,, \label{eq:posterior:score:pointwise}
\end{align}
where $t\in(0,\frac{1}{2}\log(1+\alpha^{-1}))$, and $C$ is a constant only depending on $B$, $r$, $V_{\SG}$ and $C_{\SG}$. Note that $\mu_{t}=\exp(-t)$ is decreasing and $\sigma_{t}^{2}=1-\exp(-2t)$ increases as $t$ grows. As a consequence, 
\begin{align*}
&\bbE\big\{\|\nabla\log q_{t}(\mX_{t}^{\vy} \given \vy)-\hat{\vs}_{m}^{S}(t,\mX_{t}^{\vy},\vy)\|_{2}^{2}\big\} \\
&~~~~~~~~\leq C\frac{\mu_{T_{0}}^{2}}{\sigma_{T_{0}}^{4}}\bigg[\frac{\kappa_{\vy}}{m}+\exp\bigg\{-\frac{2(\mu_{T}^{2}-\alpha\sigma_{T}^{2})}{\sigma_{T}^{2}}S\bigg\}\eta_{\vy}^{2} \\
&~~~~~~~~\quad\quad+S^{3/2}\bigg(\frac{\sigma_{T}^{2}\eta_{\vy}^{2}}{\mu_{T}^{2}-\alpha\sigma_{T}^{2}}+S\bigg)^{1/2}\kappa_{\vy}^{1/2}\exp\bigg\{2\bigg(G+\frac{\mu_{T_{0}}^{2}}{\sigma_{T_{0}}^{2}}\bigg)S\bigg\}\varepsilon_{\prior}^{1/2}\bigg]\,,
\end{align*}
for each $t\in(T_{0},T]$ with $0<T_{0}<T<\frac{1}{2}\log(1+\alpha^{-1})$. This completes the proof.
\end{proof}

\begin{proof}[Proof of Corollary~\ref{corollary:posterior:score}]
Integrating the bound of Proposition~\ref{proposition:posterior:score} over $t\in(T_{0},T)$ multiplies each of the three terms by at most $T$. For the Monte Carlo term, the condition $m\geq T\kappa_{\vy}\varepsilon^{-2}$ gives $T\kappa_{\vy}m^{-1}\leq\varepsilon^{2}$. For the Langevin convergence term, the condition on $S$ gives
\begin{equation*}
T\exp\bigg\{-\frac{2(\mu_{T}^{2}-\alpha\sigma_{T}^{2})}{\sigma_{T}^{2}}S\bigg\}\eta_{\vy}^{2}
\leq e^{-1}T\eta_{\vy}^{2}\exp\big\{-\log(T\eta_{\vy}^{2}\varepsilon^{-2})\big\}
=e^{-1}\varepsilon^{2}\,.
\end{equation*}
For the prior score estimation term, taking the square root of the condition on $\varepsilon_{\prior}$ gives
\begin{equation*}
TS^{3/2}\bigg\{\frac{\sigma_{T}^{2}\eta_{\vy}^{2}}{\mu_{T}^{2}-\alpha\sigma_{T}^{2}}+S\bigg\}^{1/2}\exp\bigg\{2\bigg(G+\frac{\mu_{T_{0}}^{2}}{\sigma_{T_{0}}^{2}}\bigg)S\bigg\}\kappa_{\vy}^{1/2}\varepsilon_{\prior}^{1/2}\leq\varepsilon^{2}\,.
\end{equation*}
Summing the three contributions completes the proof.
\end{proof}

\section{Proof of the Warm-Start Error Bound}
\label{section:proof:warm:start}
In this section, we provide a proof of Proposition~\ref{proposition:posterior:warm:start}, which proposes error bounds for the warm-start strategy. First, the clipping radius in~\eqref{eq:warm:score:clipped} is specified as
\begin{equation}\label{eq:warm:score:radius}
a_{T}(\vx,\vy):=\frac{\{{4}(d+H_{\vy}^{2}+\|\vx\|_{2}^{2})\}^{1/2}}{\sigma_{T}^{2}(\mu_{T}^{2}-\alpha\sigma_{T}^{2})}\,,
\end{equation}
where $H_{\vy}$ is given in Assumption~\ref{assumption:posterior:bound:zero}. Since the terminal time in Proposition~\ref{proposition:posterior:warm:start} satisfies $T<\bar t$, applying~\eqref{eq:true:score:growth:bound:H} at $t=T$ gives the required containment property
\begin{equation*}
\|\nabla\log q_{T}(\vx \given \vy)\|_{2}^{2}\leq\frac{{4}(d+H_{\vy}^{2}+\|\vx\|_{2}^{2})}{\sigma_{T}^{4}(\mu_{T}^{2}-\alpha\sigma_{T}^{2})^{2}}=a_{T}(\vx,\vy)^{2}\,.
\end{equation*}

\begin{proof}[Proof of Proposition~\ref{proposition:posterior:warm:start}]
Condition on $\calG$ outside the null set of Lemma~\ref{lemma:random:score:field}, so that the clipped field $\tilde{\vs}_{m}^{S}(T,\cdot,\vy)$ is a fixed Lipschitz function with linear growth, and $\hat{q}_{T}^{u}(\cdot\given\vy)$ is the conditional law of $\what{\mX}_{T,u}^{\vy}$ given $\calG$. According to the triangular inequality, we find
\begin{equation}\label{eq:lemma:posterior:warm:start:1}
\|q_{T}(\cdot \given \vy)-\hat{q}_{T}^{U}(\cdot \given \vy)\|_{\tv}^{2}\leq 2\underbrace{\|q_{T}(\cdot \given \vy)-q_{T}^{U}(\cdot \given \vy)\|_{\tv}^{2}}_{\text{convergence of Langevin}}+2\underbrace{\|q_{T}^{U}(\cdot \given \vy)-\hat{q}_{T}^{U}(\cdot \given \vy)\|_{\tv}^{2}}_{\text{score estimation error}}\,,
\end{equation}
where $q_{T}^{u}(\cdot \given \vy)$ is the marginal density of $\mX_{T,u}^{\vy}$ defined in~\eqref{section:method:warm:Langevin} with $q_{T}^{0}(\cdot \given \vy)=\hat{q}_{T}^{0}(\cdot \given \vy)$. The first term on the right-hand side is deterministic.

\noindent\emph{Step 1. The convergence of Langevin dynamics.} Recall that $\zeta_{\vy}^{2}\coloneq \chi^{2}\{\hat{q}_{T}^{0}(\cdot \given \vy)\|q_{T}(\cdot \given \vy)\}$.
For the first term in the right-hand side of~\eqref{eq:lemma:posterior:warm:start:1},
\begin{align}
\|q_{T}(\cdot \given \vy)-q_{T}^{U}(\cdot \given \vy)\|_{\tv}^{2}
&\leq \frac{1}{2}\chi^{2}\big\{q_{T}^{U}(\cdot \given \vy)\|q_{T}(\cdot \given \vy)\big\} \leq\frac{1}{2}\exp\bigg[
-\frac{2U}{C_{\LSI}\{q_{T}(\cdot\given\vy)\}}
\bigg]\zeta_{\vy}^{2}\,.\label{eq:lemma:posterior:warm:start:2}
\end{align}
The first inequality follows from Pinsker's inequality~\citep[Lemma 2.5]{app:Tsybakov2009Introduction} and~\citet[Lemma 2.7]{app:Tsybakov2009Introduction}. The second follows from Lemma~\ref{lemma:Langevin:chisq}.

\noindent\emph{Step 2. The error of the score estimation.}
The drift $\nabla\log q_{T}(\cdot\given\vy)$ of~\eqref{section:method:warm:Langevin} is Lipschitz and of linear growth by Lemma~\ref{lemma:true:score:regularity}, and the clipped drift of~\eqref{section:method:warm:Langevin:score} is Lipschitz with linear growth by Lemma~\ref{lemma:random:score:field}. Both dynamics start from $\hat{q}_{T}^{0}(\cdot\given\vy)$ and have diffusion coefficient $\sqrt{2}$. Writing $\calL$ for the law on path space, we obtain
\begin{align}
&\|q_{T}^{U}(\cdot \given \vy)-\hat{q}_{T}^{U}(\cdot \given \vy)\|_{\tv}^{2}
\leq\frac{1}{2}\kl\big\{q_{T}^{U}(\cdot\given\vy)\,\|\,\hat{q}_{T}^{U}(\cdot\given\vy)\big\} \nonumber \\
&~~~~~~~~~~~~\leq\frac{1}{2}\kl\bigg\{
\calL\big\{(\mX_{T,u}^{\vy})_{0\leq u\leq U}\big\}
\,\big\|\,
\calL\big\{(\what{\mX}_{T,u}^{\vy})_{0\leq u\leq U}\given\calG\big\}
\bigg\} \nonumber \\
&~~~~~~~~~~~~\leq\frac{1}{8}\int_{0}^{U}\int
\|\nabla\log q_{T}(\vx \given \vy)-\tilde{\vs}_{m}^{S}(T,\vx,\vy)\|_{2}^{2}
q_{T}^{u}(\vx \given \vy)\d\vx\du \nonumber \\
&~~~~~~~~~~~~=\frac{1}{8}\int_{0}^{U}\int\|\nabla\log q_{T}(\vx \given \vy)-\tilde{\vs}_{m}^{S}(T,\vx,\vy)\|_{2}^{2}\frac{q_{T}^{u}(\vx \given \vy)}{q_{T}(\vx \given \vy)}q_{T}(\vx \given \vy)\d\vx\du \nonumber \\
&~~~~~~~~~~~~\leq\bigg\{\underbrace{\int\|\nabla\log q_{T}(\vx \given \vy)-\tilde{\vs}_{m}^{S}(T,\vx,\vy)\|_{2}^{2}q_{T}(\vx \given \vy)\d\vx}_{\text{(i)}}\bigg\}^{1/4} \nonumber \\
&~~~~~~~~~~~~\quad\times\bigg\{\underbrace{\int\|\nabla\log q_{T}(\vx \given \vy)-\tilde{\vs}_{m}^{S}(T,\vx,\vy)\|_{2}^{6}q_{T}(\vx \given \vy)\d\vx}_{\text{(ii)}}\bigg\}^{1/4} \nonumber \\
&~~~~~~~~~~~~\quad\times\underbrace{\int_{0}^{U}\bigg[\int\bigg\{\frac{q_{T}^{u}(\vx \given \vy)}{q_{T}(\vx \given \vy)}\bigg\}^{2}q_{T}(\vx \given \vy)\d\vx\bigg]^{1/2}\du}_{\text{(iii)}}\,,\label{eq:lemma:posterior:warm:start:3}
\end{align}
The first inequality is Pinsker's inequality. The second is the data processing inequality for KL divergence under the endpoint map from path space to time $U$. The third follows from Lemma~\ref{lemma:entropy:girsanov}, and the last inequality is H{\"o}lder's inequality.

For the term (i) in~\eqref{eq:lemma:posterior:warm:start:3}, equation~\eqref{eq:true:score:growth:bound:H} at $t=T$ and the definition~\eqref{eq:warm:score:radius} show that the true score $\nabla\log q_{T}(\vx\given\vy)$ lies in the ball defining the projection~\eqref{eq:warm:score:clipped}. Since the projection onto a convex set containing this point is non-expansive,
\begin{equation}\label{eq:lemma:posterior:warm:start:4}
\text{(i)}\leq\int\|\nabla\log q_{T}(\vx \given \vy)-\hat{\vs}_{m}^{S}(T,\vx,\vy)\|_{2}^{2}q_{T}(\vx \given \vy)\d\vx\,.
\end{equation}
For the term (ii) in~\eqref{eq:lemma:posterior:warm:start:3}, both $\nabla\log q_{T}(\vx\given\vy)$ and $\tilde{\vs}_{m}^{S}(T,\vx,\vy)$ lie in the ball of radius $a_{T}(\vx,\vy)$, which gives the pointwise bound $\|\nabla\log q_{T}(\vx\given\vy)-\tilde{\vs}_{m}^{S}(T,\vx,\vy)\|_{2}\leq2a_{T}(\vx,\vy)$. Therefore,
\begin{align}
\text{(ii)}
&\leq 64\int a_{T}(\vx,\vy)^{6}\,q_{T}(\vx \given \vy)\d\vx \notag \\
&=\frac{64\times{4}^{3}}{\sigma_{T}^{12}(\mu_{T}^{2}-\alpha\sigma_{T}^{2})^{6}}\int\big(d+H_{\vy}^{2}+\|\vx\|_{2}^{2}\big)^{3}q_{T}(\vx \given \vy)\d\vx \nonumber \\
&\lesssim \frac{d^{3}+H_{\vy}^{6}+\kappa_{\vy}C_{\SG}V_{\SG}^{6}}{\sigma_{T}^{12}(\mu_{T}^{2}-\alpha\sigma_{T}^{2})^{6}}
\leq C\frac{\kappa_{\vy}^{2}}{(\mu_{T}^{2}-\alpha\sigma_{T}^{2})^{6}}\,, \label{eq:lemma:posterior:warm:start:44}
\end{align}
Here Lemma~\ref{lemma:fourth:moment:t} is used with $m=6$. Moreover, the lower bound on $T$ gives $\sigma_{T}^{2}>2V_{\SG}^{2}/(1+2V_{\SG}^{2})$, and $\kappa_{\vy}\geq1$ follows from~\eqref{eq:posterior:score:condition}, so the last inequality holds with a constant $C$ only depending on $d$, $H_{\vy}$, $V_{\SG}$, and $C_{\SG}$. This bound is deterministic.

For the term (iii) in~\eqref{eq:lemma:posterior:warm:start:3}, the definition of the $\chi^{2}$-divergence, the Cauchy--Schwarz inequality, and Lemma~\ref{lemma:Langevin:chisq} give
\begin{align}
&\int_{0}^{U}\bigg[\int\bigg\{\frac{q_{T}^{u}(\vx \given \vy)}{q_{T}(\vx \given \vy)}\bigg\}^{2}q_{T}(\vx \given \vy)\d\vx\bigg]^{1/2}\du \nonumber \\
&~~~~~~~~~~\quad=\int_{0}^{U}\big[\chi^{2}\{q_{T}^{u}(\cdot\given\vy)\|q_{T}(\cdot\given\vy)\}+1\big]^{1/2}\du \nonumber \\
&~~~~~~~~~~\quad\leq U^{1/2}\bigg[\int_{0}^{U}\chi^{2}\{q_{T}^{u}(\cdot\given\vy)\|q_{T}(\cdot\given\vy)\}\du+U\bigg]^{1/2} \nonumber \\
&~~~~~~~~~~\quad\leq U^{1/2}\bigg[
\frac{1}{2}C_{\LSI}\{q_{T}(\cdot\given\vy)\}\zeta_{\vy}^{2}
+U\bigg]^{1/2}\,.\label{eq:lemma:posterior:warm:start:5}
\end{align}
The last inequality also uses $\int_{0}^{U}\exp\{-2u/C_{\LSI}\{q_{T}(\cdot\given\vy)\}\}\du\leq C_{\LSI}\{q_{T}(\cdot\given\vy)\}/2$. This bound is deterministic. Averaging~\eqref{eq:lemma:posterior:warm:start:3} over $\calG$, bounding $\bbE\{\text{(i)}^{1/4}\}\leq\{\bbE\text{(i)}\}^{1/4}\leq\varepsilon_{\post}^{1/2}$ by Jensen's inequality and~\eqref{eq:lemma:posterior:warm:start:4}, and substituting~\eqref{eq:lemma:posterior:warm:start:44} and~\eqref{eq:lemma:posterior:warm:start:5} yield
\begin{align}\label{eq:lemma:posterior:warm:start:6}
&\bbE\big\{\|q_{T}^{U}(\cdot \given \vy)-\hat{q}_{T}^{U}(\cdot \given \vy)\|_{\tv}^{2}\big\} \notag \\
&~~~~~~~~~~~~\quad\leq CU^{1/2}\bigg[
\frac{1}{2}C_{\LSI}\{q_{T}(\cdot\given\vy)\}\zeta_{\vy}^{2}
+U\bigg]^{1/2} (\mu_{T}^{2}-\alpha\sigma_{T}^{2})^{-3/2}
\kappa_{\vy}^{1/2}\varepsilon_{\post}^{1/2}\,,
\end{align}
where $C$ is a constant only depending on $d$, $H_{\vy}$, $V_{\SG}$, and $C_{\SG}$.

\noindent\emph{Step 3. Conclusion.}
Equations~\eqref{eq:lemma:posterior:warm:start:2} and~\eqref{eq:lemma:posterior:warm:start:6} show that it is sufficient to require
\begin{align*}
U
&\geq\frac{1}{2}C_{\LSI}\{q_{T}(\cdot\given\vy)\}
\log(\zeta_{\vy}^{2}\varepsilon^{-1})\,,\\
\varepsilon_{\post}
&\leq\frac{(\mu_{T}^{2}-\alpha\sigma_{T}^{2})^{3}\varepsilon^{2}}
{U\kappa_{\vy}}
\bigg[
\frac{1}{2}C_{\LSI}\{q_{T}(\cdot\given\vy)\}\zeta_{\vy}^{2}
+U\bigg]^{-1}\,.
\end{align*}
Finally, Proposition~\ref{proposition:log:sobolev} gives
\begin{align*}
\frac{1}{2}C_{\LSI}\{q_{T}(\cdot\given\vy)\}
&\leq 6\sigma_{T}^{2}
\exp\bigg\{2\times\frac{\sigma_{T}^{2}+2\mu_{T}^{2}V_{\SG}^{2}}
{\sigma_{T}^{2}-2\mu_{T}^{2}V_{\SG}^{2}}
\log(\kappa_{\vy}^{2}C_{\SG}^{2})\bigg\}\,.
\end{align*}
Therefore, the two explicit conditions stated in Proposition~\ref{proposition:posterior:warm:start} imply these sufficient conditions. The first term in~\eqref{eq:lemma:posterior:warm:start:1} is then at most $\varepsilon/2$, and~\eqref{eq:lemma:posterior:warm:start:6} is at most $C\varepsilon$. Taking expectations in~\eqref{eq:lemma:posterior:warm:start:1} and combining the two terms complete the proof.
\end{proof}

\section{Proof of Theorem~\ref{theorem:convergence}}
\label{appendix:proof:theorem:convergence}

Corollary~\ref{corollary:posterior:score} gives
\begin{equation*}
\int_{T_{0}}^{T}\bbE\big\{\|\nabla_{\vx}\log q_{t}(\mX_{t}^{\vy}\given\vy)-\hat{\vs}_{m}^{S}(t,\mX_{t}^{\vy},\vy)\|_{2}^{2}\big\}\dt
\leq C\frac{\mu_{T_{0}}^{2}}{\sigma_{T_{0}}^{4}}\varepsilon^{2}\,.
\end{equation*}
By the simultaneous choice in the theorem, the pointwise estimate~\eqref{eq:posterior:score:pointwise} at $t=T$ meets the $\varepsilon_{\post}$ budget in Proposition~\ref{proposition:posterior:warm:start}. The choice of $U$ in that proposition therefore gives
\begin{equation*}
\bbE\big\{\|q_{T}(\cdot\given\vy)-\hat{q}_{T}^{U}(\cdot\given\vy)\|_{\tv}^{2}\big\}\leq C\varepsilon\,.
\end{equation*}

Combining these estimates with Proposition~\ref{proposition:error:decomposition}(ii) yields
\begin{align*}
\bbE\big\{\|q_{T_{0}}(\cdot\given\vy)-\hat{q}_{T-T_{0}}(\cdot\given\vy)\|_{\tv}^{2}\big\}
\leq C\bigg(\frac{\mu_{T_{0}}^{2}}{\sigma_{T_{0}}^{4}}\varepsilon^{2}+\varepsilon\bigg) \leq C\bigg(\frac{\mu_{T_{0}}}{\sigma_{T_{0}}^{2}}\varepsilon+\varepsilon^{1/2}\bigg)^{2}\,.
\end{align*}
Apply Proposition~\ref{proposition:error:decomposition}(i) with $\delta$ chosen as a sufficiently large fixed multiple of $\mu_{T_{0}}\sigma_{T_{0}}^{-2}\varepsilon+\varepsilon^{1/2}$ and with the corresponding radius $R$. For all sufficiently small $\varepsilon$, its conditions $\delta<1$ and $R\geq1$ hold, and
\begin{align*}
&\bbE\big[\bbW_{2}^{2}\big\{q_{0}(\cdot\given\vy),\calM(\mu_{T_{0}}^{-1})\sharp\hat{q}_{T-T_{0}}^{R}(\cdot\given\vy)\big\}\big] \\
&~~~~\leq C\bigg[\frac{\sigma_{T_{0}}^{2}}{\mu_{T_{0}}^{2}}+\bigg(\frac{\mu_{T_{0}}}{\sigma_{T_{0}}^{2}}\varepsilon+\varepsilon^{1/2}\bigg)
\log\bigg\{\kappa_{\vy}\bigg(\frac{\mu_{T_{0}}}{\sigma_{T_{0}}^{2}}\varepsilon+\varepsilon^{1/2}\bigg)^{-1}\bigg\}\bigg] \,.
\end{align*}
Here the additive $1$ in Proposition~\ref{proposition:error:decomposition}(i), as well as the fixed proportionality constant in $\delta$, is absorbed into $C$ because the logarithm in the last display is larger than $1$ for all sufficiently small $\varepsilon$. This proves the first conclusion.

Finally, for $0<t\leq1/2$,
\begin{equation*}
\frac{\sigma_{t}^{2}}{\mu_{t}^{2}}=e^{2t}-1\leq Ct\,, \qquad
\frac{\mu_{t}}{\sigma_{t}^{2}}=\frac{1}{2\sinh(t)}\leq\frac{1}{2t}\,.
\end{equation*}
Thus, at $T_{0}=\sqrt{\varepsilon}$,
\begin{equation*}
\frac{\sigma_{T_{0}}^{2}}{\mu_{T_{0}}^{2}}\leq C\varepsilon^{1/2}\,, \qquad
\varepsilon^{1/2}\leq\frac{\mu_{T_{0}}}{\sigma_{T_{0}}^{2}}\varepsilon+\varepsilon^{1/2}\leq\frac{3}{2}\varepsilon^{1/2}\,.
\end{equation*}
Substitution into the first conclusion gives the stated $C^{\prime}\varepsilon^{1/2}\log(\varepsilon^{-1})$ bound.

\section{Auxiliary Lemmas}
\label{appendix:auxiliary:lemmas}
\subsection{Sub-Gaussian properties of the forward process}
In this subsection, we establish some sub-Gaussian properties of $\mX_{0}^{\vy}$ and $\mX_{t}^{\vy}$. These results are inspired by~\citet[Proposition 2.6.1]{app:Vershynin2018High}.

\begin{lemma}
\label{lemma:tail:proba}
Suppose Assumption~\ref{assumption:prior:subGaussian} holds. Let $\mX_{0}^{\vy}\sim q_{0}(\cdot \given \vy)$. Then for each $\xi>0$,
\begin{equation*}
\bbP\big\{\|\mX_{0}^{\vy}\|_{2}\geq\xi\big\}\leq\kappa_{\vy}C_{\SG}\exp\bigg(-\frac{\xi^{2}}{V_{\SG}^{2}}\bigg)\,.
\end{equation*}  
\end{lemma}

\begin{proof}[Proof of Lemma~\ref{lemma:tail:proba}]
It is straightforward that 
\begin{align*}
\bbP\big\{\|\mX_{0}^{\vy}\|_{2}\geq\xi\big\}
&=\bbP\bigg\{\frac{\|\mX_{0}^{\vy}\|_{2}^{2}}{V_{\SG}^{2}}\geq\frac{\xi^{2}}{V_{\SG}^{2}}\bigg\}=\bbP\bigg\{\exp\bigg(\frac{\|\mX_{0}^{\vy}\|_{2}^{2}}{V_{\SG}^{2}}\bigg)\geq\exp\bigg(\frac{\xi^{2}}{V_{\SG}^{2}}\bigg)\bigg\} \\
&\leq\exp\bigg(-\frac{\xi^{2}}{V_{\SG}^{2}}\bigg)\bbE\bigg\{\exp\bigg(\frac{\|\mX_{0}^{\vy}\|_{2}^{2}}{V_{\SG}^{2}}\bigg)\bigg\}\leq\kappa_{\vy}C_{\SG}\exp\bigg(-\frac{\xi^{2}}{V_{\SG}^{2}}\bigg)\,,
\end{align*} 
where the first inequality holds from Markov's inequality, and the last inequality invokes Assumption~\ref{assumption:prior:subGaussian} and Proposition~\ref{proposition:sub:Gaussian:posterior}. This completes the proof.
\end{proof}

\begin{lemma}
\label{lemma:fourth:moment}
Suppose Assumption~\ref{assumption:prior:subGaussian} holds. Let $\mX_{0}^{\vy}\sim q_{0}(\cdot \given \vy)$. Then for $m\geq1$,
\begin{equation*}
\bbE\big(\|\mX_{0}^{\vy}\|_{2}^{m}\big)\leq 2m^{\frac{m}{2}+1}\kappa_{\vy}C_{\SG}V_{\SG}^{m}\,.
\end{equation*}
\end{lemma}

\begin{proof}[Proof of Lemma~\ref{lemma:fourth:moment}]
It follows from Lemma~\ref{lemma:tail:proba} that 
\begin{align*}
\bbE\big(\|\mX_{0}^{\vy}\|_{2}^{m}\big)
&=\int_{0}^{\infty}\bbP\big\{\|\mX_{0}^{\vy}\|_{2}^{m}\geq\xi\big\}\d\xi =m\int_{0}^{\infty}\bbP\big\{\|\mX_{0}^{\vy}\|_{2}\geq\eta\big\}\eta^{m-1}\d\eta \\
&\leq m\kappa_{\vy}C_{\SG}\int_{0}^{\infty}\exp\bigg(-\frac{\eta^{2}}{V_{\SG}^{2}}\bigg)\eta^{m-1}\d\eta \\
&=\frac{1}{2}m\kappa_{\vy}C_{\SG}V_{\SG}^{m}\int_{0}^{\infty}\exp\bigg(-\frac{\eta^{2}}{V_{\SG}^{2}}\bigg)\bigg(\frac{\eta^{2}}{V_{\SG}^{2}}\bigg)^{\frac{m}{2}-1}\d\frac{\eta^{2}}{V_{\SG}^{2}} \\
&=\frac{1}{2}m\kappa_{\vy}C_{\SG}V_{\SG}^{m}\int_{0}^{\infty}\exp(-\zeta)\zeta^{\frac{m}{2}-1}\d\zeta \\
&=\frac{1}{2}m\varGamma\bigg(\frac{m}{2}\bigg)\kappa_{\vy}C_{\SG}V_{\SG}^{m}\leq 2m^{\frac{m}{2}+1}\kappa_{\vy}C_{\SG}V_{\SG}^{m}\,,
\end{align*}
where the second equality used the change of variables $\xi=\eta^{m}$, the first inequality invokes Lemma~\ref{lemma:tail:proba}, the third equality rewrites the integrand in terms of $\eta^{2}V_{\SG}^{-2}$, the fourth equality follows from a change of variables $\zeta=\eta^{2}V_{\SG}^{-2}$. This completes the proof.
\end{proof}

\begin{lemma}
\label{lemma:fourth:moment:t}
Suppose Assumption~\ref{assumption:prior:subGaussian} holds. Let $\mX_{t}^{\vy}\sim q_{t}(\cdot \given \vy)$. Then for $t\geq0$ and  $m\geq1$,
\begin{equation*}
\bbE\big(\|\mX_{t}^{\vy}\|_{2}^{m}\big)\leq 2^{m-1}\big\{2m^{\frac{m}{2}+1}\kappa_{\vy}C_{\SG}V_{\SG}^{m}+(d+m)^{\frac{m}{2}}\big\}\,.
\end{equation*}
\end{lemma}

\begin{proof}[Proof of Lemma~\ref{lemma:fourth:moment:t}]
Let $\vepsilon\sim N(\bzero,\mI_{d})$. It is straightforward that 
\begin{equation*}
\bbE\big(\|\vepsilon\|_{2}^{m}\big)=2^{\frac{m}{2}}\varGamma\bigg(\frac{d+m}{2}\bigg)/\varGamma\bigg(\frac{d}{2}\bigg)\leq (d+m)^{\frac{m}{2}}\,.
\end{equation*}
Since $\mX_{t}^{\vy}\stackrel{\d}{=}\mu_{t}\mX_{0}^{\vy}+\sigma_{t}\vepsilon$ with $\mX_{0}^{\vy}\sim q_{0}(\cdot \given \vy)$ independent of $\vepsilon$, it follows from the triangular inequality that 
\begin{align*}
\bbE\big(\|\mX_{t}^{\vy}\|_{2}^{m}\big)
&\leq 2^{m-1}\mu_{t}^{m}\bbE\big(\|\mX_{0}^{\vy}\|_{2}^{m}\big)+2^{m-1}\sigma_{t}^{m}\bbE\big(\|\vepsilon\|_{2}^{m}\big) \\
&\leq 2^{m-1}\big\{2m^{\frac{m}{2}+1}\kappa_{\vy}C_{\SG}V_{\SG}^{m}+(d+m)^{\frac{m}{2}}\big\}\,,
\end{align*}
where we used the fact that $\mu_{t},\sigma_{t}\leq 1$. This completes the proof.
\end{proof}

\begin{lemma}
\label{lemma:tail:proba:T}
Suppose Assumption~\ref{assumption:prior:subGaussian} holds. Let $\mX_{t}^{\vy}\sim q_{t}(\cdot \given \vy)$ be a random variable defined by~\eqref{eq:method:forward:y}. Then for each $t>0$ and $\xi>0$,
\begin{equation*}
\bbP\big\{\|\mX_{t}^{\vy}\|_{2}\geq\xi\big\}\leq 2^{d}\kappa_{\vy}C_{\SG}\exp\bigg(-\frac{\xi^{2}}{2\mu_{t}^{2}V_{\SG}^{2}+8\sigma_{t}^{2}}\bigg)\,.
\end{equation*}
\end{lemma}

\begin{proof}[Proof of Lemma~\ref{lemma:tail:proba:T}]
According to Assumption~\ref{assumption:prior:subGaussian} and Proposition~\ref{proposition:sub:Gaussian:posterior}, we have
\begin{equation}\label{eq:lemma:tail:proba:T:1}
\bbE\bigg\{\exp\bigg(\frac{\|\mu_{t}\mX_{0}^{\vy}\|_{2}^{2}}{\mu_{t}^{2}V_{\SG}^{2}}\bigg)\bigg\}=\bbE\bigg\{\exp\bigg(\frac{\|\mX_{0}^{\vy}\|_{2}^{2}}{V_{\SG}^{2}}\bigg)\bigg\}\leq \kappa_{\vy}C_{\SG}\,.
\end{equation}
Let $\vepsilon\sim N(\bzero,\mI_{d})$. Then it follows that 
\begin{align}
&\bbE\bigg\{\exp\bigg(\frac{\|\sigma_{t}\vepsilon\|_{2}^{2}}{4\sigma_{t}^{2}}\bigg)\bigg\}
=\bbE\bigg\{\exp\bigg(\frac{\|\vepsilon\|_{2}^{2}}{4}\bigg)\bigg\} \nonumber \\
&~~~~~~~~~~~~=(2\pi)^{-\frac{d}{2}}\int\exp\bigg(\frac{\|\vepsilon\|_{2}^{2}}{4}\bigg)\exp\bigg(-\frac{\|\vepsilon\|_{2}^{2}}{2}\bigg)\d\vepsilon \nonumber \\
&~~~~~~~~~~~~=(2\pi)^{-\frac{d}{2}}\int\exp\bigg(-\frac{\|\vepsilon\|_{2}^{2}}{4}\bigg)\d\vepsilon\leq 2^{d}\,. \label{eq:lemma:tail:proba:T:2}
\end{align}
Notice that $\mX_{t}^{\vy}\stackrel{\d}{=}\mu_{t}\mX_{0}^{\vy}+\sigma_{t}\vepsilon$, where $\mX_{0}^{\vy}\sim q_{0}(\cdot \given \vy)$ and $\vepsilon\sim N(\bzero,\mI_{d})$ are independent. Therefore, 
\begin{align}
&\bbE\bigg\{\exp\bigg(\frac{\|\mX_{t}^{\vy}\|_{2}^{2}}{2\mu_{t}^{2}V_{\SG}^{2}+8\sigma_{t}^{2}}\bigg)\bigg\}
=\bbE\bigg\{\exp\bigg(\frac{\|\mu_{t}\mX_{0}^{\vy}+\sigma_{t}\vepsilon\|_{2}^{2}}{2\mu_{t}^{2}V_{\SG}^{2}+8\sigma_{t}^{2}}\bigg)\bigg\} \nonumber \\
&~~~~~~~~~~~~\leq\bbE\bigg\{\exp\bigg(\frac{\|\mu_{t}\mX_{0}^{\vy}\|_{2}^{2}}{\mu_{t}^{2}V_{\SG}^{2}+4\sigma_{t}^{2}}+\frac{\|\sigma_{t}\vepsilon\|_{2}^{2}}{\mu_{t}^{2}V_{\SG}^{2}+4\sigma_{t}^{2}}\bigg)\bigg\} \nonumber \\
&~~~~~~~~~~~~\leq\bbE\bigg\{\exp\bigg(\frac{\|\mu_{t}\mX_{0}^{\vy}\|_{2}^{2}}{\mu_{t}^{2}V_{\SG}^{2}+4\sigma_{t}^{2}}\bigg)\bigg\}\bbE\bigg\{\exp\bigg(\frac{\|\sigma_{t}\vepsilon\|_{2}^{2}}{\mu_{t}^{2}V_{\SG}^{2}+4\sigma_{t}^{2}}\bigg)\bigg\} \nonumber \\ 
&~~~~~~~~~~~~\leq\bbE\bigg\{\exp\bigg(\frac{\|\mu_{t}\mX_{0}^{\vy}\|_{2}^{2}}{\mu_{t}^{2}V_{\SG}^{2}}\bigg)\bigg\}\bbE\bigg\{\exp\bigg(\frac{\|\sigma_{t}\vepsilon\|_{2}^{2}}{4\sigma_{t}^{2}}\bigg)\bigg\}\leq 2^{d}\kappa_{\vy}C_{\SG}\,, \label{eq:lemma:tail:proba:T:3}
\end{align}
where the first inequality follows from  $\|\vu+\vv\|_{2}^{2}\leq2\|\vu\|_{2}^{2}+2\|\vv\|_{2}^{2}$, the second inequality holds from the independence of $\mX_{0}^{\vy}$ and $\vepsilon$, and the last inequality is due to~\eqref{eq:lemma:tail:proba:T:1} and~\eqref{eq:lemma:tail:proba:T:2}. Then we aim to bound the tail probability. For each $\xi>0$, we have 
\begin{align*}
\bbP\big\{\|\mX_{t}^{\vy}\|_{2}\geq \xi\big\}
&=\bbP\bigg\{\frac{\|\mX_{t}^{\vy}\|_{2}^{2}}{2\mu_{t}^{2}V_{\SG}^{2}+8\sigma_{t}^{2}}\geq\frac{\xi^{2}}{2\mu_{t}^{2}V_{\SG}^{2}+8\sigma_{t}^{2}}\bigg\} \\
&=\bbP\bigg\{\exp\bigg(\frac{\|\mX_{t}^{\vy}\|_{2}^{2}}{2\mu_{t}^{2}V_{\SG}^{2}+8\sigma_{t}^{2}}\bigg)\geq\exp\bigg(\frac{\xi^{2}}{2\mu_{t}^{2}V_{\SG}^{2}+8\sigma_{t}^{2}}\bigg)\bigg\} \\
&\leq\exp\bigg(-\frac{\xi^{2}}{2\mu_{t}^{2}V_{\SG}^{2}+8\sigma_{t}^{2}}\bigg)\bbE\bigg\{\exp\bigg(\frac{\|\mX_{t}^{\vy}\|_{2}^{2}}{2\mu_{t}^{2}V_{\SG}^{2}+8\sigma_{t}^{2}}\bigg)\bigg\} \\
&\leq 2^{d}\kappa_{\vy}C_{\SG}\exp\bigg(-\frac{\xi^{2}}{2\mu_{t}^{2}V_{\SG}^{2}+8\sigma_{t}^{2}}\bigg)\,,
\end{align*}
where the first inequality invokes Markov's inequality, and the last inequality is due to~\eqref{eq:lemma:tail:proba:T:3}. This completes the proof.
\end{proof}

\subsection{Convergence of Langevin dynamics}
Consider the Langevin dynamics with an invariant density $\nu$:
\begin{equation}\label{eq:Langevin:dynamics}
\d\mZ_{t}=\nabla\log\nu(\mZ_{t})\dt+\sqrt{2}\d\mB_{t}\,, \quad \mZ_{0}\sim\nu_{0}\,,
\end{equation}
where $(\mB_{t})_{t\geq 0}$ is a $d$-dimensional Brownian motion. Denote by $\nu_{t}$ the marginal density of $\mZ_{t}$. In this subsection, we establish the convergence results of Langevin dynamics \eqref{eq:Langevin:dynamics}. At every invocation in this paper, the invariant density $\nu$ is either the strongly log-concave restricted Gaussian oracle target or the terminal posterior $q_{T}(\cdot\given\vy)$, whose $C^{2}$ potentials and Hessian bounds guarantee the well-posedness and semigroup conditions assumed in the two lemmas below. We first characterize its convergence in Wasserstein distance:
\begin{lemma}
\label{lemma:Langevin:W2}
Suppose the invariant distribution $\nu$ is $m$-strongly log-concave, the initial distribution $\nu_{0}$ has a finite second moment, and~\eqref{eq:Langevin:dynamics} admits a unique non-explosive strong solution. Then for each $t>0$,
\begin{equation*}
\bbW_{2}^{2}(\nu_{t},\nu)\leq\exp(-2mt)\bbW_{2}^{2}(\nu_{0},\nu)\,.
\end{equation*}
\end{lemma}

\begin{proof}[Proof of Lemma~\ref{lemma:Langevin:W2}]
We first construct a Wasserstein coupling. Let $\gamma_{0}$ be the optimal coupling of $(\nu_{0},\nu)$, which means,
\begin{equation}\label{eq:lemma:Langevin:W2:1}
\bbE_{(\mZ_{0},\mZ_{0}^{*})\sim\gamma_{0}}\big(\|\mZ_{0}-\mZ_{0}^{*}\|_{2}^{2}\big)=\int\|\vz_{0}-\vz_{0}^{*}\|_{2}^{2}\d\gamma_{0}(\vz_{0},\vz_{0}^{*})=\bbW_{2}^{2}(\nu_{0},\nu)\,. 
\end{equation}
We define two processes $\mZ_{t}$ and $\mZ_{t}^{*}$ by evolving the Langevin dynamics~\eqref{eq:Langevin:dynamics} with the same Brownian motion, that is, 
\begin{equation}\label{eq:lemma:Langevin:W2:2}
\d\mZ_{t}=\nabla\log\nu(\mZ_{t})\dt+\sqrt{2}\d\mB_{t}\,, \quad\text{and}\quad \d\mZ_{t}^{*}=\nabla\log\nu(\mZ_{t}^{*})\dt+\sqrt{2}\d\mB_{t}\,.
\end{equation}
Denote by $\gamma_{t}$ the joint distribution of $(\mZ_{t},\mZ_{t}^{*})$. It is apparent that $\gamma_{t}$ is a coupling of $(\nu_{t},\nu)$, as $\nu$ is the invariant distribution of the Langevin dynamics~\eqref{eq:Langevin:dynamics} and $\mZ_{t}^{*}\sim\nu$ for each $t>0$. Then 
\begin{align}
\d\|\mZ_{t}-\mZ_{t}^{*}\|_{2}^{2}
&=2(\mZ_{t}-\mZ_{t}^{*},\d\mZ_{t}-\d\mZ_{t}^{*}) \nonumber \\
&=2(\mZ_{t}-\mZ_{t}^{*},\nabla\log\nu(\mZ_{t})-\nabla\log\nu(\mZ_{t}^{*}))\dt \nonumber \\
&\leq -2m\|\mZ_{t}-\mZ_{t}^{*}\|_{2}^{2}\dt\,, \label{eq:lemma:Langevin:W2:3}
\end{align}
where the second equality involves~\eqref{eq:lemma:Langevin:W2:2} and the fact that two Langevin dynamics share the same Brownian motion, and the inequality follows from the strong log-concavity of $\nu$. Applying Gronwall's inequality~\citep[Section B.2]{app:evans2010partial} to~\eqref{eq:lemma:Langevin:W2:3} yields 
\begin{equation*}
\|\mZ_{t}-\mZ_{t}^{*}\|_{2}^{2}\leq\exp(-2mt)\|\mZ_{0}-\mZ_{0}^{*}\|_{2}^{2}\,.
\end{equation*}
Taking expectation with respect to both sides of the inequality implies 
\begin{align*}
\bbW_{2}^{2}(\nu_{t},\nu)
&\leq\bbE_{(\mZ_{t},\mZ_{t}^{*})\sim\gamma_{t}}\big(\|\mZ_{t}-\mZ_{t}^{*}\|_{2}^{2}\big) \\
&\leq\exp(-2mt)\bbE_{(\mZ_{0},\mZ_{0}^{*})\sim\gamma_{0}}\big(\|\mZ_{0}-\mZ_{0}^{*}\|_{2}^{2}\big) \\
&=\exp(-2mt)\bbW_{2}^{2}(\nu_{0},\nu)\,,
\end{align*}
where the first inequality is owing to the definition of the Wasserstein distance, and the equality follows from~\eqref{eq:lemma:Langevin:W2:1}. This completes the proof.
\end{proof}

Then, the convergence of Langevin dynamics~\eqref{eq:Langevin:dynamics} in $\chi^{2}$-divergence is stated as the following lemma. The proof can be found in ~\citet{app:Chewi2024Analysis}.

\begin{lemma}
\label{lemma:Langevin:chisq}
Suppose that the invariant density $\nu\in C^{2}(\bbR^{d})$ satisfies the log-Sobolev inequality with constant $C_{\LSI}(\nu)$, and that the Langevin dynamics~\eqref{eq:Langevin:dynamics} generates the conservative reversible Markov semigroup on $L^{2}(\nu)$ associated with $L=\Delta+\langle\nabla\log\nu,\nabla\rangle$, then for each initial density $\nu_{0}$, it holds that 
\begin{equation*}
\chi^{2}(\nu_{t}\|\nu)\leq\exp\bigg\{-\frac{2t}{C_{\LSI}(\nu)}\bigg\}\chi^{2}(\nu_{0}\|\nu)\,.
\end{equation*}
\end{lemma}


\subsection{Log-concavity}
Let $\mu$ be a probability distribution absolutely continuous with respect to the Lebesgue measure. Denote by $\rho$ the density of $\mu$, that is, $\d\mu(\vx)=\rho(\vx)\d\vx$. Suppose that $\rho\in C^{2}(\bbR^{d})$. In this work, we will consider some regimes of the distribution $\mu$:
\begin{enumerate}[label=(\roman*)]
\item The distribution $\mu$ is called strongly log-concave, if there exists a constant $\alpha>0$ such that $-\nabla^{2}\log\rho(\vx)\succeq\alpha\mI_{d}$ for each $\vx\in\bbR^{d}$.
\item The distribution $\mu$ is called semi-log-concave, if there exists a constant $\alpha>0$ such that $-\nabla^{2}\log\rho(\vx)\succeq-\alpha\mI_{d}$ for each $\vx\in\bbR^{d}$.
\item The distribution $\mu$ on $\bbR^{d}$ satisfies the log-Sobolev inequality, if there exists a constant $C_{\LSI}(\mu)>0$ such that
\begin{equation}\tag{LSI}\label{eq:LSI}
\ent_{\mu}(f^{2})\leq 2C_{\LSI}(\mu)\bbE_{\mu}\big(\|\nabla f\|_{2}^{2}\big)\,, \quad \text{for each}~f\in C_{0}^{\infty}(\bbR^{d})\,,
\end{equation}
where $\ent_{\mu}(g)=\bbE_{\mu}(g\log g)-\bbE_{\mu}(g)\log\bbE_{\mu}(g)$ for nonnegative $g$, and $\bbE_{\mu}(\cdot)$ denotes the expectation with respect to $\mu$.
\item The distribution $\mu$ on $\bbR^{d}$ has sub-Gaussian tails, if there exist constants $V_{\SG}>0$ and $C_{\SG}>0$, such that 
\begin{equation*}
\int\exp\bigg(\frac{\|\vx\|_{2}^{2}}{V_{\SG}^{2}}\bigg)\d\mu(\vx)\leq C_{\SG}\,.
\end{equation*}
\end{enumerate}
The conditions (i) to (iv) have close connections:
\begin{equation*}
\text{strongly log-concavity} \subset \text{log-Sobolev inequality} \subset \text{sub-Gaussian tails}\,.
\end{equation*}
Specifically, the Bakry--{\'E}mery theorem~\citep{app:Bakry1985Diffusions} shows that $\alpha$-strong log-concavity implies a log-Sobolev inequality with $C_{\LSI}(\mu)\leq\alpha^{-1}$. A log-Sobolev inequality implies sub-Gaussian tails~\citep{app:Ledoux1999Concentration}. See~\citet[Theorem 9.9]{app:Villani2003Topics} and~\citet[Section 5.4]{app:Bakry2014Analysis} for further details. 

\section{Experimental Details and Additional Results}
\label{app:experimental_details}
In this appendix, we provide the specific parameter settings, mathematical formulations for the measurement operators, and detailed implementation steps for the algorithms and baseline methods used in Section~\ref{section:experiments}, together with an additional cross-dataset experiment on robustness to prior mismatch.

\subsection{Details of measurement operators}
\label{app:measurement_details}

\textbf{Linear Operator.} For both linear tasks, the forward operator is a convolution with a kernel $\psi\in\bbR^{15\times 15}$. For Gaussian deblurring, $\psi$ is an isotropic Gaussian filter with standard deviation $\sigma_{\rm blur}=2.0$. For motion deblurring, $\psi$ is a synthetic motion kernel of intensity $0.5$ generated to simulate random camera trajectories.

\textbf{Nonlinear Operator.} The nonlinear forward operator $\calF_{\tau^*}$ approximates the GOPRO blur formation process~\citep{app:nah2017deep}, in which blur arises from temporal aggregation rather than spatial convolution: a latent signal is realized as a sequence of $M$ closely-spaced sharp frames within the camera exposure window, the frames are linearly averaged, and the result is passed through a nonlinear camera response function $\calR(\vz)=\vz^{1/2.2}$ (gamma correction). In our experiments, we use the pretrained distillation model of~\citet{app:tran2021explore} as a black-box realization of this process, yielding the deterministic nonlinear measurement model $\mY=\calF_{\tau^*}(\mX_{0},\vc^*)+\vn$, where $\vc^*$ is a fixed latent vector that conditions the operator on a single blur configuration.

For completeness, the network $\calF_{\tau}$ was originally trained on a dataset of sharp-blurred image pairs $\{(\mX_j,\mY_j)\}_{j=1}^{N}$ jointly with an auxiliary network $\calC_{\phi}$, which extracts blur information at training time, by minimizing the Charbonnier loss~\citep{app:lai2017deep}: $$\tau^*,\phi^*=\arg\min_{\tau,\phi}\sum_{j=1}^{N}\rho\big[\mY_j,\calF_{\tau}\{\mX_j,\calC_{\phi}(\mX_j,\mY_j)\}\big] \, ,$$ where $\rho$ is the Charbonnier function. To obtain a deterministic operator for our experiments, we fix the auxiliary input by sampling a single latent vector $\vc^*\sim N(\bm{0},\sigma_c^2\mI_c)$ with $\sigma_c=0.3$, after which $\calF_{\tau^*}(\cdot,\vc^*)$ is held fixed across all experiments.

\subsection{Implementation of PDPS}
\label{sec:hyperparameters}
\subsubsection{Prior score approximation}
Our method requires a neural network estimator, $\hat{\vs}_{\rm prior}$, for the prior score, $\nabla_{\vx_0} \log \pi_0(\vx_0)$, and we employ a publicly available, pre-trained model based on the well-known EDM framework~\citep{app:karras2022elucidating}. Specifically, it provides a neural network estimate for the unconditional denoiser $\mD(\vx, \sigma)$, denoted as $\what{\mD}(\vx, \sigma)$, where $$
\mD(\vx, \sigma) \coloneq  \bbE(\mX_0 \given \mX_0 + \sigma \mZ = \vx) \, , \quad \mX_0 \sim \pi_0 \, , \quad \mZ \sim N(\bm{0}, \mI_d) \, , \quad \mX_0 \perp\!\!\!\perp \mZ \, .
$$
Furthermore, the denoiser and score function are related by Tweedie's formula ~\citep{app:Efron2011Tweedie}:
\begin{align*}
\nabla_{\vx} \log \pi_{\sigma}(\vx) = \frac{\mD(\vx, \sigma) - \vx}{\sigma^2} \, ,
\end{align*}
where $\pi_{\sigma}$ is the density of $\mX_0 + \sigma\mZ$. Thus, $\{\what{\mD}(\vx, \sigma) - \vx\} / \sigma^2$ serves to approximate the $\sigma$-smoothed prior score, $\nabla_{\vx} \log \pi_{\sigma}(\vx)$. Given that $\pi_{\sigma} \rightarrow \pi$ as $\sigma \rightarrow 0$, achieving a precise approximation of the true prior score $\nabla_{\vx} \log \pi(\vx)$ theoretically necessitates a sufficiently small $\sigma$ for the denoiser $\what{\mD}(\vx, \sigma)$. However, practical implementation revealed significant numerical instability when using very small $\sigma$ values, likely due to the inherent approximation errors of the neural network $\what{\mD}$. After empirical investigation to balance accuracy and stability, we selected $\sigma_d = 0.09$. Consequently, the prior score estimator employed consistently throughout our numerical experiments is fixed as $\hat{\vs}_{\rm prior}(\vx) \triangleq \{\what{\mD}(\vx, \sigma_d) - \vx\} / \sigma_d^2$.

\subsubsection{Posterior score estimation (Algorithm~\ref{alg:unbiased:posterior:score})} The core of our method is the RGO-based score estimation, which we implement using a Langevin Monte Carlo (LMC) sampler. For the first posterior-score query of each independent reconstruction, the $M$ inner chains are initialized independently from $N(\bm{0},\mI_d)$. For every subsequent query, each chain is initialized from the terminal state of the corresponding chain at the preceding query. This continuation is used throughout the warm-start and reverse-diffusion stages, including across the boundary between the two stages. It is a computational acceleration heuristic and is distinct from the common-random-number construction adopted in the theoretical analysis. Specifically, at each query we discretize the RGO process~\eqref{eq:RGO:Langevin:score} using the Euler--Maruyama method. Let $\vxi_k \sim N(\bm{0}, \mI_d)$. For $k = 0, \ldots, N_{\rm in} - 1$, we iterate
\begin{equation*}
\what{\mX}_{0,k+1}^{\vx,\vy,t} - \what{\mX}_{0,k}^{\vx,\vy,t}=\bigg\{\hat{\vs}_{\rm prior}(\what{\mX}_{0,k}^{\vx,\vy,t})+\frac{\mu_{t}}{\sigma_{t}^{2}}(\vx-\mu_{t}\what{\mX}_{0,k}^{\vx,\vy,t})-\nabla\ell_{\vy}(\what{\mX}_{0,k}^{\vx,\vy,t})\bigg\}\Delta t_{\rm in}+\sqrt{2 \Delta t_{\rm in}} \vxi_k \, .
\end{equation*}
In practice, we use $M=20$ parallel chains to generate samples. The LMC step size $\Delta t_{\rm in}$ is determined adaptively based on a fixed signal-to-noise ratio (SNR) of $r_{\rm in} = 0.075$, following the strategy in~\cite{app:song2021scorebased}. To compute the MC average for the posterior denoiser~\eqref{eq:denoiser:MC}, we use samples from the latter half of the LMC iterations (a burn-in factor of $\rho=0.5$). Thus, $m = \rho M N_{\rm in}$. The LMC iteration count, $N_{\rm in}$, varies depending on the stage, as detailed below.

\subsubsection{Warm-start stage (Algorithm~\ref{alg:warm:start})} The warm-start procedure consists of an outer Langevin loop designed to sample from the initial posterior density $q_T(\cdot \given \vy)$: Let $\vxi_k \sim N(\bm{0}, \mI_d)$. For $k = 0, \ldots, N_{\rm out} - 1$, we start with $\what{\mX}^{\vy}_{T, 0} \sim N(\bm{0}, \mI_d)$ and discretize \eqref{section:method:warm:Langevin:score} as
\begin{equation*}
\what{\mX}_{T,k+1}^{\vy} - \what{\mX}_{T,k}^{\vy}=\hat{\vs}_m(T,\what{\mX}_{T,k}^{\vy},\vy) \Delta t_{\rm out} +\sqrt{2 \Delta t_{\rm out}}\,\vxi_k \, .
\end{equation*} 
The clipping operation in~\eqref{eq:warm:score:clipped} is introduced to control the warm-start drift in the theoretical analysis. In the numerical warm-start, we use the unclipped estimator $\hat{\vs}_m$, since clipping was not found necessary for numerical stability in our experiments. This outer loop runs for $N_{\rm out}=400$ iterations, and its adaptive step size $\Delta t_{\rm out}$ is governed by an SNR of $r_{\rm out} = 0.16$. We invoke Algorithm~\ref{alg:unbiased:posterior:score} to compute the score $\hat{\vs}_m(\vx, \vy, t)$, setting the inner LMC iteration count to $N_{\rm in}=50$. The choice of the warm-start time $T$ is critical and is tuned for each problem.

\subsubsection{Reverse diffusion stage (Algorithm~\ref{alg:posterior:sampling})} The main sampling process simulates the reverse SDE from the warm-start time $T$ down to an early-stopping time $T_0$. We set $T_0=0.05$ in all experiments except the terminal-time ablation in Section~\ref{sec:ablation}, for which we use $T_0=0.001$ so that $T_0<T$ for every tested value of $T$. In practice, an Euler--Maruyama discretization for~\eqref{eq:reversal:score} with step size $\Delta t_{\rm rev} = (T-T_0)/N_{\rm rev}$ is implemented: For $k = 0, \ldots, N_{\rm rev} - 1$, we have
\begin{equation*}
\what{\mX}_{k+1}^{\vy} - \what{\mX}_{k}^{\vy}=\big\{\what{\mX}_{k}^{\vy}+2\,\hat{\vs}_{m}(T-k \Delta t_{\rm rev},\what{\mX}_{k}^{\vy},\vy)\big\}\Delta t_{\rm rev} + \sqrt{2 \Delta t_{\rm rev}}\vxi_{k} \, ,
\end{equation*}
where $\what{\mX}^{\vy}_{0}  = \what{\mX}^{\vy}_{T, N_{\rm out}}$ from the warm-start phase and $\vxi_k\sim N(\bm{0}, \mI_d)$. The number of discretization steps is set to $N_{\rm rev} = 1200 \, T$. At each step of this reverse process, we again call upon Algorithm~\ref{alg:unbiased:posterior:score} for score estimation $\hat{\vs}_m(\vx, \vy, t)$, this time with an inner LMC iteration count of $N_{\rm in}=20$. In place of the theoretical post-processing (truncation $\calT_{R}$ followed by rescaling $\calM(\mu_{T_{0}}^{-1})$) prescribed by Algorithm~\ref{alg:posterior:sampling}, the practical pipeline adopts the following substitution: upon reaching $T_0$, we follow~\cite{app:song2021scorebased} and execute a deterministic step using only the drift term to map $\what{\mX}^{\vy}_{N_{\rm rev}}$ directly from $T_0$ to $0$. We denote the output as $\what{\mX}^{\vy}$.

\subsubsection{Final denoising step} Since our prior score estimator approximates a smoothed prior (with $\sigma_d = 0.09$), the raw output $\what{\mX}^{\vy}$ may contain minor residual noise. To mitigate this, we apply a final, unconditional denoising step using the pre-trained denoiser network $\what{\mD}$ with an empirically chosen noise level of $\sigma_d^{\prime} = 0.03$. The
final reported samples are thus $\what{\mD}(\what{\mX}^{\vy}, \sigma_d^{\prime})$, which enhances the perceptual quality by reducing noise without overly smoothing fine details.

\subsection{Details of comparison methods}
\label{app:comparison_details}
\subsubsection{Diffusion posterior sampling (DPS)}
In essence, the DPS method proposed in \cite{app:chung2023diffusion} relies on a standard DDPM \cite{app:ho2020denoising} forward noise addition process, namely:
\begin{align*}
\mX^{\rm DDPM}_t = \sqrt{\bar{\alpha}_t} \mX_0 + \sqrt{1 - \bar{\alpha}_t} \mZ \, , \quad \mZ \sim N(\bm{0}, \mI_d) \, , \quad \mX_0 \perp\!\!\!\perp \mZ \, .
\end{align*}
For brevity, we abbreviate $\mX^{\rm DDPM}_t$ as $\mX_t$. DPS further uses \eqref{eq:posterior:score:decomposition}, while approximates $p_{t}(\vx_0 \given \vx)$ with $\delta_{\bbE(\mX_{0} \given \mX_{t}=\vx)}(\vx_{0})$. It then controls the generation process by adding a guidance term $-\zeta \nabla_{\vx_t}\|\vy - \calF\{\bbE(\mX_0 \given \mX_t = \vx_t)\}\|_2$ to the unconditional DDPM iterative scheme.  Here, $\zeta$ represents a critical hyperparameter that controls the strength of conditional guidance. To ensure a fair comparison where both PDPS and DPS leverage the same prior knowledge, we fully reused the original codebase of \cite{app:chung2023diffusion}, while employing the score transformation techniques presented in \cite{app:karras2022elucidating} to reconstruct the DDPM-type score from the pre-trained EDM model $D_\theta$, and then use it to replace the original DDPM-type score in \cite{app:chung2023diffusion}. We use a $1000$-step DDPM sampler with a linear beta schedule, epsilon prediction, clipped denoised estimates, and fixed-small reverse variance.

\subsubsection{Total variation (TV)}
This classical approach solves the optimization problem:
\begin{align} \label{eq:tv_objective}
\min_{\mathbf{x} \in \mathbb{R}^d} \frac{1}{2}\|\mathbf{y} - \mathcal{F}(\mathbf{x})\|^2_2 + \lambda \|\nabla \mathbf{x}\|_{2,1}\,,
\end{align}
balancing data fidelity against isotropic Total Variation (TV) regularization, controlled by $\lambda > 0$. We solve~\eqref{eq:tv_objective} using deterministic proximal-gradient descent. Let $f(\mathbf{x})=\frac{1}{2}\|\mathbf{y}-\mathcal F(\mathbf{x})\|_2^2$. At outer iteration $j$, we compute $\nabla f(\mathbf{x}^{(j)})$ using PyTorch automatic differentiation and perform
\begin{equation*}
\mathbf{x}^{(j+1)}
=\operatorname{prox}_{\eta\lambda\mathrm{TV}}
\big\{\mathbf{x}^{(j)}-\eta\nabla f(\mathbf{x}^{(j)})\big\}.
\end{equation*}
The TV proximal map is evaluated using \texttt{TVPrior.prox} from \texttt{deepinv} 0.3.2\footnote{\url{https://github.com/deepinv/deepinv}}. For Gaussian, motion, and nonlinear deblurring, respectively, we use $(\eta,\lambda,K)=(1.0,0.01,50)$, $(1.0,0.01,150)$, and $(0.05,0.05,50)$, where $K$ is the number of outer iterations. For all reported experiments, the inner TV proximal solver uses at most $1000$ iterations with stopping tolerance $10^{-8}$. Each reconstruction is initialized from the measurement, with bilinear resizing only when its spatial dimensions differ from those of the reconstruction.
 
\subsection{Hyperparameter settings}
\label{app:hyperparameters}
The key hyperparameters for PDPS (warm-start time $T$) and DPS (guidance scale $\zeta$) used in our experiments are summarized in Table~\ref{tab:hyperparams}. We present the settings for both the quantitative evaluations on the full test sets (FFHQ and AFHQ) and the specific parameters used for the qualitative case studies, where fine-tuning was applied to exploit the potential of each method.

\begin{table}[H]
\centering
\caption{\footnotesize Hyperparameters $T$ (for PDPS) and $\zeta$ (for DPS) used in different experimental settings.}
\label{tab:hyperparams}
\small
\renewcommand{\arraystretch}{0.9} 
\begin{tabular}{l l c c}
\toprule
\textbf{Task} & \textbf{Subset / Image} & \textbf{PDPS ($T$)} & \textbf{DPS ($\zeta$)} \\
\midrule
\multicolumn{4}{l}{\textit{\textbf{1. Quantitative Evaluation (Batch Average)}}} \\
\multicolumn{4}{l}{{FFHQ Dataset (128 images)}} \\
Gaussian Deblur & full batch & 0.2 & 0.9 \\
Motion Deblur   & full batch & 0.5 & 1.2 \\
Nonlinear Deblur& full batch & 20.0 & 0.2 \\
\multicolumn{4}{l}{{AFHQ Dataset (Cross-Dataset, 128 images)}} \\
Motion Deblur   & full batch & 0.5 & 1.3 \\
Nonlinear Deblur& full batch & 20.0 & 0.2 \\

\midrule
\multicolumn{4}{l}{\textit{\textbf{2. Qualitative Case Studies (Specific Images)}}} \\
\multicolumn{4}{l}{{FFHQ (In-Distribution)}} \\
Gaussian Deblur & face 3 & 0.2 & 0.9 \\
& face 4 & 0.2 & 0.8 \\
Motion Deblur   & face 1-2 & 0.5 & 1.3 \\
Nonlinear Deblur& face 5 & 3.5 & 0.5 \\
& face 6 & 3.5 & 0.9 \\
\multicolumn{4}{l}{{AFHQ (Out-of-Distribution)}} \\
Motion Deblur   & Cat 1 \& Lion & 0.5 & 1.3 \\
Nonlinear Deblur& Dog & 5.0 & 0.4 \\
& Cat 2 & 9.0 & 0.3 \\
\bottomrule
\end{tabular}
\end{table}

\subsection{Robustness to prior mismatch}
\label{app:prior_robustness}
We investigate robustness to prior mismatch by using the FFHQ-trained prior score model to restore animal faces from the AFHQ dataset~\citep{app:choi2020stargan}. This cross-dataset experiment tests whether PDPS's Monte Carlo score estimation, which integrates prior information through the RGO mechanism rather than relying on it directly, can compensate for inaccuracies stemming from the mismatched prior. 

\begin{table}[H]
\small
\centering
\caption{\footnotesize Comparison for various methods under cross-dataset evaluation.}
\begin{tabular}{l ccc ccc}
\toprule
\multirow{2}{*}{} & \multicolumn{3}{c}{Motion Deblur}
& \multicolumn{3}{c}{Nonlinear Deblur} \\
\cmidrule(lr){2-4} \cmidrule(lr){5-7}
& TV & DPS & \textbf{PDPS} & TV & DPS & \textbf{PDPS} \\
\midrule
PSNR & 23.95 & 23.82 & \textbf{25.45} & 18.81 & 17.17 & \textbf{24.45} \\
SSIM & 0.76 & 0.78 & \textbf{0.83} & 0.45 & 0.40 & \textbf{0.80} \\
\bottomrule
\end{tabular}
\label{tab:afhq}
\end{table}

Table~\ref{tab:afhq} summarizes the average quantitative performance. While all methods suffer from the distribution shift, PDPS maintains a clear advantage. DPS is at a similar quantitative level to TV for motion deblurring and falls below TV for nonlinear deblurring, a pattern consistent with increased sensitivity to prior mismatch. PDPS continues to produce coherent reconstructions (Figure~\ref{fig:animal_results}) and recovers recognizable animal-specific structures, including whisker contours in ``cat 2,'' despite using a human-face prior. These results highlight the practical value of PDPS's RGO-based Monte Carlo estimator, which directly targets the posterior score and retains useful likelihood information even under prior mismatch.

\begin{figure}[H]
\centering
\includegraphics[width=0.75\textwidth]{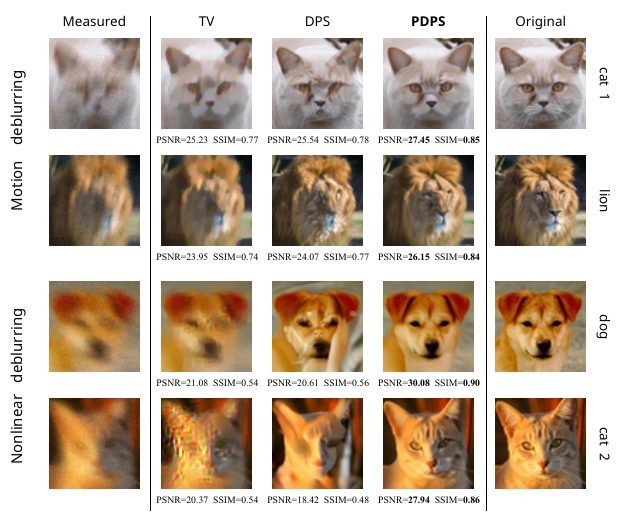}
\vspace{-5pt}
\caption{\footnotesize Cross-dataset deblurring results on AFHQ animal faces (`cat 1', `lion', `dog', `cat 2') using an FFHQ human face prior. Top two rows: Motion Deblurring. Bottom two rows: Nonlinear Deblurring. Comparison of Measured input, TV, DPS, PDPS (Ours), and original ground truth.}
\label{fig:animal_results}
\end{figure}

\end{document}